\documentclass[letterpaper]{article} %
\usepackage{aaai24}  %
\usepackage{times}  %
\usepackage{helvet}  %
\usepackage{courier}  %
\usepackage[hyphens]{url}  %
\usepackage{graphicx} %
\usepackage{natbib}  %
\usepackage{caption} %
\usepackage{algorithm}
\usepackage{algorithmic}

\usepackage{newfloat}
\usepackage{listings}

\usepackage{booktabs}       %
\usepackage{amsfonts}       %
\usepackage{nicefrac}       %
\usepackage[usenames, dvipsnames]{xcolor}         %
\usepackage{amsmath}
\usepackage{amsthm}
\usepackage{amssymb}
\usepackage{mathtools}
\usepackage{bm}
\usepackage{multirow}
\usepackage{makecell}
\usepackage{tikz}
\usepackage{subcaption}
\usepackage{subcaption}
\usepackage[outline]{contour}%
\usepackage{xstring}
\usepackage[capitalise,poorman]{cleveref}
\usepackage{placeins}
\usepackage[para]{footmisc}

\usetikzlibrary{matrix,math,calc,arrows.meta,backgrounds}

\pgfdeclarelayer{overlay}
\pgfsetlayers{overlay,background,main}

\DeclareCaptionStyle{ruled}{labelfont=normalfont,labelsep=colon,strut=off} %
\floatstyle{ruled}
\newfloat{listing}{tb}{lst}{}
\floatname{listing}{Listing}
\crefname{thm}{Thm.}{Thm.}
\crefname{cor}{Cor.}{Cor.}
\crefname{lem}{Lem.}{Lem.}
\crefname{prop}{Prop.}{Prop.}
\crefname{defn}{Def.}{Def.}

\newcommand{\gr}[1]{\textcolor{OliveGreen!85!black}{#1}}
\newcommand{\rr}[1]{\textcolor{RedOrange!85!black}{#1}}

\newcommand{\sigmoid}[1]{\operatorname{\sigma}\left(#1\right)}

\newcommand{\affine}[1]{\operatorname{hom}\left(#1\right)}
\newcommand{\convex}[1]{\operatorname{conv}\left(#1\right)}

\newcommand{\real}[1]{\operatorname{Re}\left(#1\right)}
\newcommand{\R}{\mathbb{R}}
\newcommand\T{{\mathpalette\raiseT\top}}
\newcommand\raiseT[2]{\raisebox{-.8ex}{$#1#2$}}

\newcommand{\vv}[1]{\mathbf{#1}}

\newcommand{\snorm}[1]{\left\lVert#1\right\rVert_2}

\newcommand{\argmax}{\,\arg\!\max\,}

\newcommand{\sign}[1]{\operatorname{sign}{\left(#1\right)}}

\newcommand{\rank}[1]{\operatorname{rank}{\left(#1\right)}}
\newcommand{\signrank}[1]{\operatorname{sign-rank}{\left(#1\right)}}
\newcommand{\act}[1]{\operatorname{act}{\left(#1\right)}}
\newcommand{\alt}[1]{\operatorname{alt}{\left(#1\right)}}
\newcommand{\ys}[1]{{{\scriptstyle#1}}}
\newcommand{\Act}[2]{{A}_{#1, #2}}
\newcommand{\Alt}[2]{{V}_{#1, #2}}
\newcommand{\grass}[2]{{\mathsf{Gr}}_{#1, #2}}
\newcommand{\pgrass}[2]{{\mathsf{Gr}}^{+}_{#1, #2}}

\newcommand{\dft}[2]{\bm{W}^{\mathsf{DFT}}_{#1, #2}}

\newcommand{\cyclic}[2]{\bm{\mathcal{C}}_{#1, #2}}
\newcommand{\argmaxable}[1]{\mathcal{A}\left(#1\right)}

\newcommand{\yy}{label assignment }

\newcommand{\tsv}[1]{
\scalebox{.7}{\begin{tikzpicture}[mycolour, scale=2, baseline, anchor=base]\sv{#1}{0}{0}{#1}{\large};\end{tikzpicture}}
}

\theoremstyle{plain}

\newtheorem{theorem}{Theorem}
\newtheorem{definition}{Definition}
\newtheorem*{lemma*}{Lemma}
\newtheorem{lemma}{Lemma}

\newcommand{\shortparagraph}[1]{\textbf{#1}}
\newcommand{\shortsubsection}[1]{\subsection{#1}}
\newcommand{\shortsection}[1]{\section{#1}}

\renewcommand{\shortparagraph}[1]{\paragraph{#1}}
\renewcommand{\shortsubsection}[1]{\subsection{#1}}
\renewcommand{\shortsection}[1]{\section{#1}}

\title{Taming the Sigmoid Bottleneck:\\ Provably Argmaxable Sparse Multi-Label Classification}
\author {
    Andreas Grivas,
    Antonio Vergari\textsuperscript{$\dagger$},
    Adam Lopez\textsuperscript{$\dagger$}
}
\affiliations {
    School of Informatics, University of Edinburgh, UK\\
    \{agrivas, avergari, alopez\}@ed.ac.uk
}

\begin{document}

\maketitle
\insert\footins{\noindent\footnotesize $\,\dagger$ Co-supervision.}

\begin{abstract}
    Sigmoid output layers are widely used in multi-label classification (MLC) tasks, in which multiple labels can be assigned to any input. In many practical MLC tasks, the number of possible labels is in the thousands, often exceeding the number of input features and resulting in a \emph{low-rank} output layer. In multi-class classification, it is known that such a low-rank output layer is a bottleneck that can result in \emph{unargmaxable} classes: classes which cannot be predicted for any input. 
    In this paper, we show that for MLC tasks, the analogous \emph{sigmoid bottleneck} results in exponentially many unargmaxable label combinations. We explain how to detect these unargmaxable outputs and demonstrate their presence in three widely used MLC datasets. We then show that they can be prevented in practice by introducing a Discrete Fourier Transform (DFT) output layer, which guarantees that all sparse label combinations with up to $k$ active labels are argmaxable. Our DFT layer trains faster and is more parameter efficient, matching the F1@k score of a sigmoid layer while using up to $50\%$ fewer trainable parameters. Our code is publicly available at \url{https://github.com/andreasgrv/sigmoid-bottleneck}.
\end{abstract}

\tikzset{
    define colours/.code n args={1}{
    \def\colorPallete{{"228,26,28", "", "77,175,74", "152,78,163", "255,127,0", "247,129,191"}}
    \definecolor{clr1}{RGB}{228,26,28};
    \definecolor{clr2}{RGB}{55,126,184};
    \definecolor{clr3}{RGB}{77,175,74};
    \definecolor{clr4}{RGB}{152,78,163};
    \definecolor{clr5}{RGB}{255,127,0};
    \definecolor{clr6}{RGB}{247,129,191}
  },
  mycolour/.style={
         define colours={1}
  } 
}

\newcommand{\sv}[5]{
    \StrLen{#1}[\svlen];
    \def \sp {.22};
    \def \xx {#2 - \svlen*.14 + .02};
    \def \yy {#3};
    \foreach \i in {1,...,\svlen}
    {
        \StrChar{#1}{\i}[\svsign];
        \node[align=right](#4) at (\xx+\i*\sp, \yy) {#5\textcolor{clr\i}{\contour{clr\i}{$\bm{\svsign}$}}};
    }
}

\newcommand{\ssv}[5]{
    \StrLen{#1}[\svlen];
    \def \sp {.22};
    \def \xx {#2 - \svlen*.14 + .02};
    \def \yy {#3};
    \foreach \i in {1,...,\svlen}
    {
        \StrChar{#1}{\i}[\svsign];
        \node[align=right](#4) at (\xx+\i*\sp, \yy) {#5\textcolor{black}{\contour{black}{$\bm{\svsign}$}}};
    }
}

\section{Introduction}
\label{sec:intro}

Sigmoid classifiers for Multi-Label Classification (MLC) are simple to implement:
just append a linear layer with sigmoid activations to your neural feature encoder of choice.
They are widely used in neural MLC with thousands of output labels; applications include clinical coding~\citep{mullenbach2018}, image classification~\citep{Baruch2020}, fine-grained entity typing~\citep{choi-etal-2018} and protein function prediction~\citep{Kulmanov2019}.
Moreover, they are the default for MLC in frameworks such as Scikit-learn~\citep{scikit-learn} and Keras~\citep{Paul2020}.
In this paper we highlight an overlooked weakness of this layer. If, as is common for computational efficiency, we make the number of features smaller than the number of labels, the result is a \textbf{Bottlenecked Sigmoid Layer (BSL)}, in which \emph{exponentially many
label combinations cannot be predicted irrespective of input}. We say that such outputs are \textbf{unargmaxable}~\citep{grivas-2022}. \cref{subfig:unargmaxable} illustrates how a BSL with two features and three labels must have unargmaxable label combinations.

\begin{figure}[!t]
\begin{subfigure}{.38\columnwidth}
\centering
    \resizebox{.98\textwidth}{!}{%
        \begin{tikzpicture}[mycolour,scale=2]
        \def \scale {10};
        \begin{scope}
            \clip(0, 0) circle (1.35);
            \node[](w1) at (1.00, 0.00){};
            \draw[clr1, thick] let \p{w1}=(w1) in (\scale * \y{w1}, -\scale * \x{w1}) -- (-\scale * \y{w1}, \scale * \x{w1});
            \node[](w2) at (0.50, 0.70){};
            \draw[clr2, thick] let \p{w2}=(w2) in (\scale * \y{w2}, -\scale * \x{w2}) -- (-\scale * \y{w2}, \scale * \x{w2});
            \node[](w4) at (-0.50, 0.50){};
            \draw[clr3, thick] let \p{w4}=(w4) in (\scale * \y{w4}, -\scale * \x{w4}) -- (-\scale * \y{w4}, \scale * \x{w4});
    
            \sv{---}{-0.41}{-1.00}{---}{\LARGE};
            \sv{--+}{-1.00}{-0.08}{--+}{\LARGE};
            \sv{-++}{-0.51}{1.00}{-++}{\LARGE};
            \sv{+--}{0.51}{-1.00}{+--}{\LARGE};
            \sv{++-}{1.00}{0.08}{++-}{\LARGE};
            \sv{+++}{0.41}{1.00}{+++}{\LARGE};
        \end{scope}
        \node[] at (0, -1.4) {};
        \end{tikzpicture}%
    }
    \caption{Unargmaxability: \tsv{+-+} and \tsv{-+-} cannot be predicted.}
    \label{subfig:unargmaxable}
\end{subfigure}\hfill
\begin{subfigure}{.55\columnwidth}
    \begin{tikzpicture}
        \tikzstyle{every node}=[font=\scriptsize]
        \def \yoff {-.5};
        \def \width {6.8em};
        
        \node[align=left, text width=\width] (l1) at (0, \yoff){Unspecified septicemia};
        \node[align=left, text width=\width, below of=l1, yshift=1.8em] (l2) {Lymphosarcoma and...};
        \node[align=left, text width=\width, below of=l2, yshift=1.8em] (l3) {Anemia, unspecified};
        \node[align=left, text width=\width, below of=l3, yshift=1.8em] (l4) {Influenza...};
        \node[align=left, text width=\width, below of=l4, yshift=1.8em] (l5) {Cellulitis and...};
        \node[align=left, text width=\width, below of=l5, yshift=1.8em] (l6) {Rash...};
        \node[align=left, text width=\width, below of=l6, yshift=1.8em] (l7) {Sepsis};
        \node[align=left, text width=\width, below of=l7, yshift=1.8em] (l8) {Viral pneumonia...};
        \node[align=left, text width=\width, below of=l8, yshift=1.8em] (l9) {\hspace{3.5em}$\vdots$};
        
        \node[align=center, draw, minimum width=3, minimum height=6em, left of=l5, xshift=-4.5em](doc) {Clinical\\ Document};

        \draw (l1.west) -- (doc) node[inner sep=0pt, outer sep=2pt, xshift=-.15em, fill=white, at start]{$\bm{+}$};
        \draw (l2.west) -- (doc) node[inner sep=0pt, outer sep=2pt, xshift=-.15em, fill=white, at start]{$\bm{+}$};
        \draw (l3.west) -- (doc) node[inner sep=0pt, outer sep=2pt, xshift=-.15em, fill=white, at start]{$\bm{+}$};
        \draw (l4.west) -- (doc) node[inner sep=0pt, outer sep=2pt, xshift=-.15em, fill=white, at start]{$\bm{+}$};
        \draw (l5.west) -- (doc) node[inner sep=0pt, outer sep=2pt, xshift=-.15em, fill=white, at start]{$\bm{+}$};
        \draw (l6.west) -- (doc) node[inner sep=0pt, outer sep=2pt, xshift=-.15em, fill=white, at start]{$\bm{+}$};
        \draw (l7.west) -- (doc) node[inner sep=0pt, outer sep=2pt, xshift=-.15em, fill=white, at start]{$\bm{+}$};
        \draw[dotted] (l8.west) -- (doc) node[inner sep=0pt, outer sep=2pt, xshift=-.15em, fill=white, at start]{$\bm{-}$};
        \draw[dotted] (l9.west) -- (doc) node[inner sep=0pt, outer sep=2pt, xshift=-.15em, fill=white, at start]{$\bm{-}$};
    \end{tikzpicture}
    \caption{Unargmaxable test example from MIMIC-III for BSL with $d=50$.}
    \label{subfig:unargmaxable-mimic}
\end{subfigure}
\caption{
When we have more labels ($n$) than features ($d$), some label combinations are unargmaxable, i.e. impossible to predict.
Left: in a $d=2$ feature space with $n=3$ classification hyperplanes through the origin, only 6 out of 8 label combinations can be predicted irrespective of how we orient the hyperplanes.\footnotemark~Right: a BSL trained on the MIMIC-III clinical MLC dataset with $d=50$ and $n=8921$ is unable to predict this label combination\footnotemark~which has the depicted $7$ active labels $(+)$ and the remaining ones are inactive $(-)$.
}    
\label{fig:problem}
\end{figure}
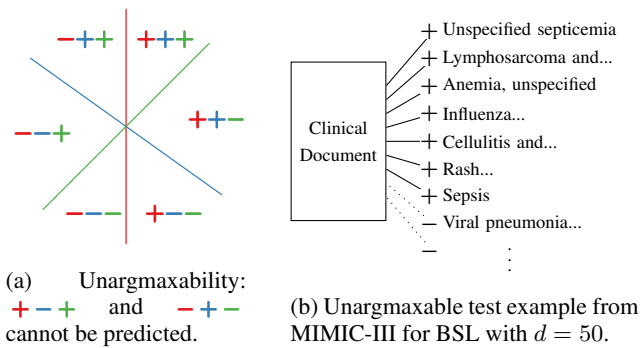
\addtocounter{footnote}{-1}
\footnotetext{Adding a bias term to the BSL allows the hyperplanes to have offsets and they will not necessarily meet at the origin. However, this cannot solve the problem: such a BSL is more restricted than increasing $d$ by one: we still only get 7 out of 8 label combinations.}
\addtocounter{footnote}{1}
\footnotetext{There are also unargmaxable test examples for $d=100$ and $d=200$, we chose this example as it had fewer active labels.}

But unargmaxable label combinations are only a problem if our application requires those combinations. As we show in this paper, they often do. For example, in the safety-critical application of clinical MLC (see~\cref{sec:exps}), unargmaxability can make it impossible to label a report with a specific combination of findings, as illustrated with a real example in \cref{subfig:unargmaxable-mimic}. This would be surprising and unacceptable to users of the system when such label combinations do indeed occur in data. Since BSLs are widely used,
it is critical for developers and users of a model to be aware of this problem and to be able to guarantee that all meaningful outputs for the task at hand are argmaxable.

Previous work has shown that bottlenecked output layers have restrictions on expressivity, but focused only on multi-class classification~\citep{yang2017breaking, Ganea2019} and not MLC. But for multi-class classification, the consequences are minor: classes can be unargmaxable in theory~\citep{demeter2020stolen}
but rarely are in practice~\citep{grivas-2022}. In MLC, we will show that exponentially many label combinations are unargmaxable, and as we have already seen in~\cref{subfig:unargmaxable-mimic}, meaningful outputs can be unargmaxable. While a BSL can in principle learn to represent any particular output, it comes with no guarantees.
And although we can obtain post hoc guarantees by verifying whether specific meaningful outputs are argmaxable, there may be too many outputs to check exhaustively.
To sidestep this limitation, we show how to construct an output layer that guarantees that meaningful outputs are argmaxable by construction. To do so, we provide guarantees for a superset of outputs: those with up to $k$ active labels, where we choose $k$ based on the statistics of the dataset. This is possible since for most MLC tasks $k$ is bounded~\citep{Jain2019}, either empirically (e.g. $k$=80 for MIMIC-III) or by construction (e.g. $k$=50 for BioASQ~\citep{Tsatsaronis2015}).

In summary, our contributions are: \textbf{i)}
We formalise the argmaxability problem for MLC and expose the limitations of BSLs which are widely used in practice (\cref{sec:mlc});
\textbf{ii)} We provide ways of verifying if a label combination is argmaxable for a model (\cref{sec:verify}) and show that for three widely used MLC datasets BSLs can have unargmaxable test set label combinations (\cref{sec:exps}).
\textbf{iii)}
We prove that this need not be the case; we can guarantee that any output with up to $k$ active labels is argmaxable by constraining the output layer parametrisation to a family of matrices. The Discrete Fourier Transform (DFT) matrix is in this family and we use it to parametrise our DFT layer, an efficient replacement output layer with such guarantees (\cref{sec:low-rank-dft}).
\textbf{iv)} Through experiments on three MLC datasets we show that our DFT layer guarantees that meaningful outputs are argmaxable while converging faster and being more parameter efficient than a BSL (\cref{sec:exps}).

\shortsection{Multi-label Classification}
\label{sec:mlc}

We consider a MLC model that predicts a complete label assignment $\vv{y} \in \{+, -\}^{n}$ for a label vocabulary of size $n$, and where each $\vv{y}_i \in \{+, -\}$ denotes if a single label is active (+) or inactive (-). 
Many neural MLC models for problems with large label vocabularies employ an output layer that is linear, e.g., for
fine-grained entity typing~\citep{choi-etal-2018}, protein function prediction~\citep{Kulmanov2019}, clinical coding~\citep{mullenbach2018} and multi-label image classification~\citep{Baruch2020}.

\shortparagraph{Bottlenecked Sigmoid Layers.} 
\label{sec:bsl}
A linear sigmoid layer takes as input a feature vector $\vv{x} \in \R^d$ and predicts $\vv{y}$ by assuming that all labels are independent given $\vv{x}$.
The idea is that a powerful encoder does the ``heavy lifting'' and projects inputs to meaningul embeddings in $\R^d$ such that they can be easily separated by $n$ different hyperplanes.
When $n$ is large, due to computational constraints it is popular to realise such a layer as a \textit{Bottlenecked Sigmoid Layer} (BSL) which is parametrised by a low-rank $\vv{W} \in \R^{n \times d}$ and associates with each label the probability $P(\vv{y}_i \mid \vv{x}) = \sigmoid{{\vv{w}^{(i)}}^\T\vv{x}}$ if $\vv{y}_i = +$ and $1 - \sigmoid{{\vv{w}^{(i)}}^\T\vv{x}}$ otherwise. 
Here,  $\sigma$ is the logistic sigmoid and 
 $\vv{w}^{(i)}$
is the weight vector of the $i$-th classifier (hyperplane), i.e., the $i^{th}$ row of $\vv{W}$. Note that all $\vv{w}^{(i)}$ see the same shared $\vv{x}$.
We focus on such a setup and discuss its limitations because it is the default in mainstream ML libraries such as scikit-learn~\citep{scikit-learn} and is largely used as a simple classifier~\citep{mullenbach2018,Baruch2020,Kulmanov2019}.
In the following, we will denote a whole multi-label classifier by the parametrization of its last layer, e.g., we will say ``a classifier $\vv{W}$'', as our analysis is agnostic to the feature encoder.

\shortparagraph{Argmaxable Label Assignments.}
Making a prediction with a BSL boils down to predicting every label independently by computing $\vv{y}^{*}_{i}=\argmax_{\vv{y}_i} P(\vv{y}_i \mid \vv{x})$.
This is equivalent to taking the sign of the logit of the $i$-th classifier, i.e., computing $\vv{y}_{i}^{*}=\sign{{\vv{w}^{(i)}}^\T\vv{x}}$ where $\sign{\vv{z}_i}=+$ if $\vv{z}_i > 0$ and $-$ if $\vv{z}_i < 0$.~\footnote{In the case $\vv{z}_i = 0$, we perturb $\vv{x}$ by a tiny $\bm\epsilon$ and recompute $\vv{z}$.}
Therefore, $\vv{y}^{*}=\sign{\vv{W}\vv{x}}.$
\begin{definition}
\label{def:argmax}
A label assignment $\vv{y}$ is \textbf{argmaxable} for a classifier $\vv{W}$ if there exists an input $\vv{x}$ for which thresholding the output probabilities using the argmax decision rule $\sign{P(\vv{y}_i = + \mid \vv{x}) - \frac{1}{2}}$ produces $\vv{y}_i$,\footnote{Due to monotonicity of sigmoid $\sigmoid{a} > \frac{1}{2} \iff a > 0$.} i.e. $\vv{y} \, \text{is argmaxable} \iff \exists \vv{x} : \sign{\vv{W}\vv{x}} = \vv{y}$.
\end{definition}

From a geometric perspective, we can interpret the $n$ rows of $\vv{W}$ as the normal vectors, $\vv{w}^{(i)}$, of $n$ hyperplanes that tesselate feature space into regions. 
We can identify each region by assigning it a sign vector which identifies on which side of each hyperplane it is (\cref{fig:problem}, left). 
From this view, a label assignment $\vv{y}$ is argmaxable if the halfspaces intersect in such a way that the corresponding region is formed. For example, the region \tsv{---} is formed as an intersection of the negative halfspaces.
For a classifier parameterised by a full rank $\vv{W}\in\R^{n\times n}$, all $2^n$ label combinations are argmaxable.
However, since $d \ll n$ in many real-world applications, as discussed in \cref{sec:intro}, $\vv{W}$ will be low rank, and therefore \textit{only a (small) subset of label combinations is argmaxable}.

\begin{figure*}[t]
    \centering
    \begin{minipage}[b]{.58\textwidth}
    \centering
    \includegraphics[width=.69\columnwidth]{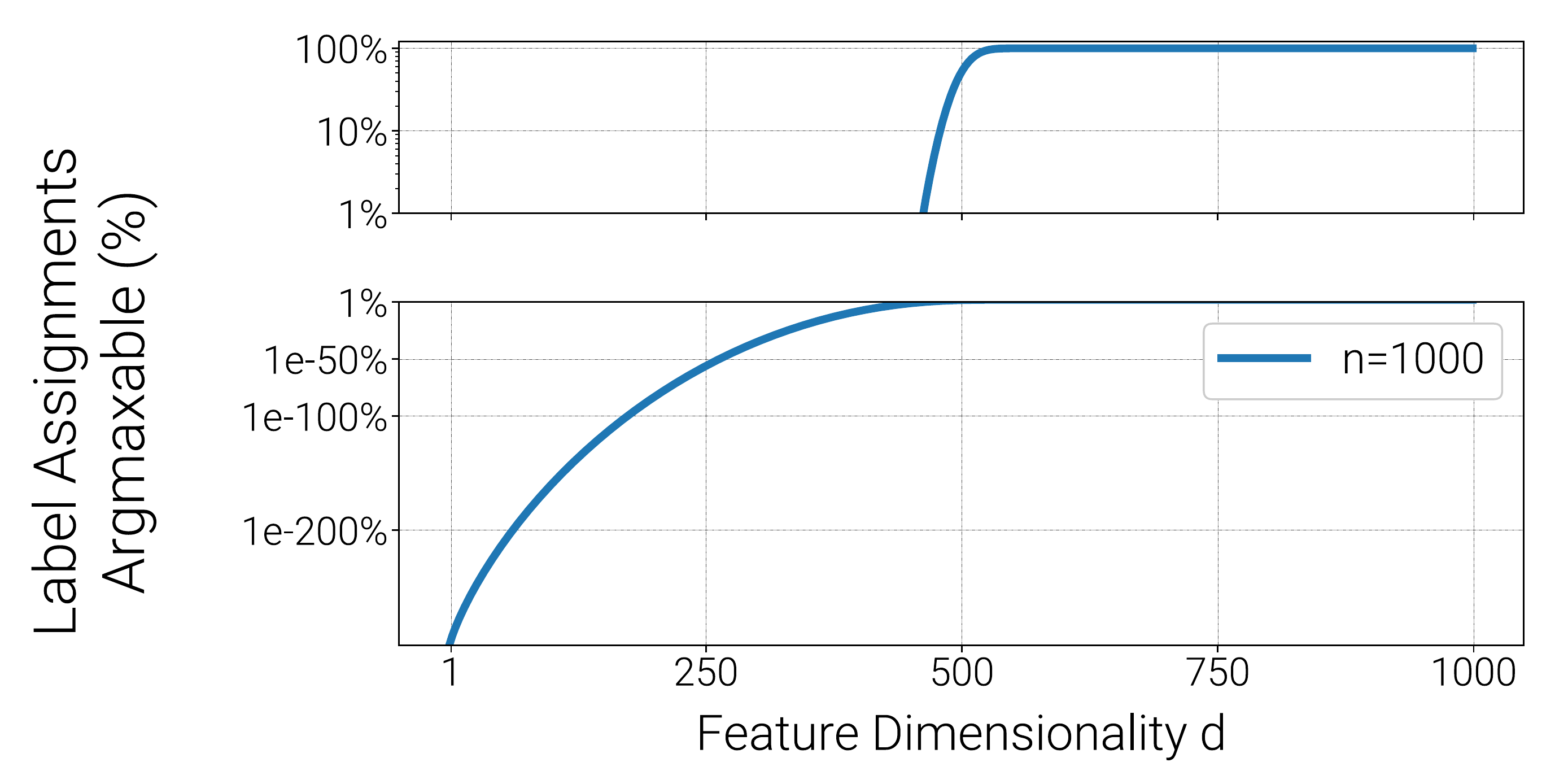}
    
    \caption{
    We log-plot what percentage of the $2^{1000}$ label combinations is argmaxable for a BSL with $n=1000$ labels as we decrease the feature dimensionality $d$ (right to left). When $d \ll n$ we can represent exponentially fewer outputs. We split the y-axis to highlight the fast dip when $d<500$.
    }
    \label{fig:cover}
    \end{minipage}
     \hfill
     \begin{minipage}[b]{0.38\textwidth}
     \centering
     \resizebox{.52\columnwidth}{!}{%
         \begin{tikzpicture}[mycolour,scale=2]
         \def \scale {10};
         \begin{scope}
             \clip(0, 0) circle (1.75);
             \node[](w1) at (1.00, 0.00){};
             \draw[clr1, thick] let \p{w1}=(w1) in (\scale * \y{w1}, -\scale * \x{w1}) -- (-\scale * \y{w1}, \scale * \x{w1});
             \node[](w2) at (0.50, 0.70){};
             \draw[clr2, thick] let \p{w2}=(w2) in (\scale * \y{w2}, -\scale * \x{w2}) -- (-\scale * \y{w2}, \scale * \x{w2});
             \node[](w4) at (-0.50, 0.50){};
             \draw[clr3, thick] let \p{w4}=(w4) in (\scale * \y{w4}, -\scale * \x{w4}) -- (-\scale * \y{w4}, \scale * \x{w4});
     
             \sv{---}{-0.41}{-1.00}{---}{\LARGE};
             \draw[]({-0.41},{-1.00}) circle (0.41421);
             \sv{--+}{-1.00}{-0.08}{--+}{\LARGE};
             \draw[]({-1.00},{-0.08}) circle (0.64858);
             \sv{-++}{-0.51}{1.00}{-++}{\LARGE};
             \draw[]({-0.51},{1.00}) circle (0.51462);
             \sv{+--}{0.51}{-1.00}{+--}{\LARGE};
             \draw[]({0.51},{-1.00}) circle (0.51462);
             \sv{++-}{1.00}{0.08}{++-}{\LARGE};
             \draw[]({1.00},{0.08}) circle (0.64858);
             \sv{+++}{0.41}{1.00}{+++}{\LARGE};
             \draw[]({0.41},{1.00}) circle (0.41421);
         \end{scope}
         \end{tikzpicture}%
     }
     \caption{Our $n=3,d=2$ example from~\cref{fig:problem}. We include the balls found by the Chebyshev LP for each argmaxable label combination. When $d \ll n$, most balls will have a tiny radius.}
     \label{fig:chebyshev}
     \end{minipage}
\end{figure*}

\shortsubsection{Hyperplane Overcrowding}
For a low-rank classifier $\vv{W} \in \R^{n \times d}$ and $d \ll n$, we have a large number of hyperplanes finely slicing a lower dimensional feature space.
The natural question is, \textit{out of the $2^n$ label combinations $\vv{y}$, how many can such a classifier actually represent?}
To elaborate on this, let us define the set of argmaxable label combinations for a classifier $\vv{W}$ as:
\begin{equation}
\argmaxable{\vv{W}} = \{ \sign{\vv{W}\vv{x}}\quad \forall \vv{x} \in \R^d\}
\end{equation}
We can exactly count the number of argmaxable label combinations, i.e. $\left|\argmaxable{\vv{W}} \right|$ if $\vv{W}$ is in \textbf{general position}.
\begin{definition}
    $\vv{W} \in \R^{n \times d}$ is in \textbf{general position} if no subset of $d$ rows is linearly dependent. See~\cref{sec:gp}.
    \label{def:gp}
\end{definition}
\begin{theorem}
\label{thm:cover}
~\citep[Thm 2]{Cover1965} If $\vv{W}$ is in general position, the number of argmaxable label combinations is:
\begin{equation}
\left|\argmaxable{\vv{W}}\right| = 2 \sum\nolimits_{d'=0}^{d-1} \binom{n-1}{d'}.
\label{eq:counts}
\end{equation}
\end{theorem}
It follows that i) the number of argmaxable label combinations depends only on $n$ and $d$, not the specific $\vv{W}$, and ii) most label combinations will be unargmaxable for $d \ll n$ as \cref{eq:counts} indicates an exponential growth (see ~\cref{fig:cover}).

While we can count the number of (un)argmaxable label combinations, it remains an open problem to verify if a specific set of label combinations can ever be predicted by a given classifier.
We provide a solution in the next section.

\shortsection{Verifying Argmaxable Label Assignments}
\label{sec:verify}

Given a low-rank classifier $\vv{W}$, we are interested in verifying if a set of $L$ label combinations of interest $\{\vv{y}^{(l)}\}_{l=1}^{L}$ is (un)argmaxable.
These labels can belong to a held out set, as in our experiments (\cref{sec:exps}), and help quantify the generalisability and trustworthiness of the given classifier, as we would expect it to be able to predict all $L$ outputs.  

A simple strategy is to verify the argmaxability of each $\vv{y}^{(l)}$.
For that, we use a Chebyshev Linear Programme (LP), which also gives us a proxy for the size of a region, as we explain next.
The LP aims to find the Chebyshev center~\citep[p. 417]{boyd2004} of the region encoded by $\vv{y}^{(l)}$: the center of the largest ball of radius $\epsilon$ that can be embedded within it (see~\cref{fig:chebyshev}).
As a constrained optimization problem (see derivation in~\cref{app:derivlp})\footnote{The appendix can be found at \url{https://arxiv.org/abs/2310.10443}.}, we want to solve:
\begin{alignat}{2}
& \text{maximise} & \quad &  \epsilon  \\
& \text{subject to} &  & -\vv{y}_i {\vv{w}^{(i)}}^\T \vv{x} + \epsilon \Vert \vv{w}^{(i)} \Vert_2 \leq 0, \quad  1 \leq i \leq n, \nonumber \\
& & & -10^4 \leq \vv{x}_{j} \leq 10^4, \quad 1 \leq j \leq d, \quad
 \epsilon > \text{eps} \nonumber
\end{alignat}
where we abuse notation and $\vv{y}_i \in \{+1, -1\}$. We constrain each entry of $\vv{x}$ in a bounded region, since the Chebyshev center is not defined otherwise. 
If the LP is feasible it returns the maximum radius $\epsilon$ and we verify that $\vv{y}$ is argmaxable.
Note that we add an additional constraint, $\epsilon> \text{eps}$, where eps is within the numerical accuracy the LP solver can operate.\footnote{We use $\text{eps}=10^{-8}$ since~\citet{gurobi} has a minimum tolerance of $10^{-9}$.}
As such, while we defined argmaxability in absolute terms (\cref{sec:mlc}), in practice it can only be verified up to some numerical precision:
$\vv{y}$ could be argmaxable, but an LP might not be able to detect it if the neighbourhood around all representative $\vv{x}$ is tiny.
As such, we define $\epsilon$-argmaxability to characterise robustly argmaxable label combinations.
\begin{definition}
A label assignment $\vv{y}$ is \textbf{$\bm\epsilon$-argmaxable} for a classifier $\vv{W}$ if it is argmaxable even under the presence of any noise vector $\bm\delta$ having magnitude $\snorm{\bm\delta} \leq \epsilon$:
\resizebox{.98\linewidth}{!}{
\begin{minipage}{1.15\linewidth}
\begin{equation*}
\vv{y} \, \text{is $\epsilon$-argmaxable} \iff \forall \bm\delta, \snorm{\bm\delta} \leq \epsilon,\, \exists \vv{x}  : \sign{\vv{W}\left(\vv{x} + \vv{\bm\delta}\right)} = \vv{y}.
\label{def:eargmax}
\end{equation*}
\end{minipage}
}
\end{definition}
Our Chebyshev LP is able to verify  $\epsilon$-argmaxability, and therefore argmaxability, as the first implies the second.
Note, however, that the reverse is not true. 
Although verifying that a classifier is able to argmax a certain set of labels is of extreme importance, verification can be computationally expensive, as LPs become intractable as we scale $n$, $d$ and $L$.
To avoid this, we devise a classifier that guarantees that all labels of interest are argmaxable \textit{by design}.

\shortsection{DFT Layers for $k$-Active MLC}
\label{sec:low-rank-dft}
Designing a low-rank BSL that guarantees argmaxability for all $2^n$ possible label combinations is impossible, according to \cref{thm:cover}. 
However, for most MLC datasets the label combinations are \textit{sparse}; only a handful of labels are \textit{active} for any given example. 
As such, herein we choose an upper bound $k$ on the number of active labels for each dataset and show how to modify the parametrisation of a BSL so that any $k$-active label assignment is guaranteed to be argmaxable.
We first define sufficient criteria by specifying a \emph{broad family of parametrisations} for which our result holds: the weight matrix should have at least $2k+1$ input features and all its maximal minors should be non-zero and have the same sign. 
Next, we \textit{specify} an implementation satisfying these criteria, the Discrete Fourier Transform (DFT) layer, which 
is computationally appealing  and is accurate in practice.

\begin{figure*}[!t]
    \centering
    \begin{subfigure}[b]{.28\textwidth}
        \begin{tikzpicture}[mycolour,scale=1,font=\small]
        \matrix [matrix of math nodes, left delimiter={[}, right delimiter={]},ampersand replacement=\&](A){
            \textcolor{clr1}{1.0} \&        \textcolor{clr1}{0.0} \\
            \textcolor{clr2}{0.5} \&        \textcolor{clr2}{0.7} \\
            \textcolor{clr3}{0.0} \&        \textcolor{clr3}{1.0} \\
            \textcolor{clr4}{-0.5} \&       \textcolor{clr4}{0.5} \\
        };
        \node[left of=A, xshift=-18pt](L) {$\vv{W}=$};
        \begin{pgfonlayer}{overlay}
            \node[above, xshift=-4em, yshift=-5.5em, align=center] (d12) (A-4-1.south west) {$\Delta_{\left\{\textcolor{clr1}{1},\textcolor{clr2}{2}\right\}}=.7$};
            \node[above, xshift=-4em, yshift=-7em, align=center] (d13) (A-4-1.south west) {$\Delta_{\left\{\textcolor{clr1}{1},\textcolor{clr3}{3}\right\}}=1.$};
            \node[above, xshift=-4em, yshift=-8.5em, align=center] (d14) (A-4-1.south west) {$\Delta_{\left\{\textcolor{clr1}{1},\textcolor{clr4}{4}\right\}}=.5$};
            \node[above, xshift=3em, yshift=-5.5em, align=center] (d23) (A-4-2.south east) {$\Delta_{\left\{\textcolor{clr2}{2},\textcolor{clr3}{3}\right\}}=.5$};
            \node[above, xshift=3em, yshift=-7em, align=center](d24) (A-4-2.south east) {$\Delta_{\left\{\textcolor{clr2}{2},\textcolor{clr4}{4}\right\}}=.6$};
            \node[above, xshift=3em, yshift=-8.5em, align=center](d34) (A-4-2.south east) {$\Delta_{\left\{\textcolor{clr3}{3},\textcolor{clr4}{4}\right\}}=.5$};
        \end{pgfonlayer}
        \end{tikzpicture}
        \caption{All $\binom{4}{2}=6$ maximal minors, $\Delta_I$, are positive, hence $\vv{W} \in \pgrass{4}{2}$.}
        \label{subfig:maxminors}
    \end{subfigure}%
    \hfill
    \begin{subfigure}[b]{.3\textwidth}
        \begin{tikzpicture}[mycolour,scale=1.3,font=\small]
        \def \scale {10};
        \begin{scope}
            \clip(0, 0) circle (1.6);
            \node[](w1) at (1.00, 0.00){};
            \draw[clr1, dashed] let \p{w1}=(w1) in (\scale * \y{w1}, -\scale * \x{w1}) -- (-\scale * \y{w1}, \scale * \x{w1});
            \node[](w2) at (0.50, 0.70){};
            \draw[clr2, dashed] let \p{w2}=(w2) in (\scale * \y{w2}, -\scale * \x{w2}) -- (-\scale * \y{w2}, \scale * \x{w2});
            \node[](w3) at (0.00, 1.00){};
            \draw[clr3, dashed] let \p{w3}=(w3) in (\scale * \y{w3}, -\scale * \x{w3}) -- (-\scale * \y{w3}, \scale * \x{w3});
            \node[](w4) at (-0.50, 0.50){};
            \draw[clr4, dashed] let \p{w4}=(w4) in (\scale * \y{w4}, -\scale * \x{w4}) -- (-\scale * \y{w4}, \scale * \x{w4});
            \draw[-{Latex[length=3mm]}, thick, clr1] (0, 0) -- (w1);
            \node[] at (1.00, 0.00){\textcolor{clr1}{$\mathbf{w}_1$}};
            \draw[-{Latex[length=3mm]}, thick, clr2] (0, 0) -- (w2);
            \node[] at (0.50, 0.70){\textcolor{clr2}{$\mathbf{w}_2$}};
            \draw[-{Latex[length=3mm]}, thick, clr3] (0, 0) -- (w3);
            \node[] at (0.00, 1.00){\textcolor{clr3}{$\mathbf{w}_3$}};
            \draw[-{Latex[length=3mm]}, thick, clr4] (0, 0) -- (w4);
            \node[] at (-0.50, 0.50){\textcolor{clr4}{$\mathbf{w}_4$}};
        \end{scope}
        \end{tikzpicture}
        \caption{Correspondence of $\vv{W}$ and geometric picture: each normal vector is a row of $\vv{W}$.}
    \end{subfigure}%
    \hfill
    \begin{subfigure}[b]{.3\textwidth}
        \begin{tikzpicture}[mycolour,scale=1.4]
        \def \scale {10};
        \begin{scope}
            \clip(0, 0) circle (1.5);
            \node[](w1) at (1.00, 0.00){};
            \draw[clr1, thick] let \p{w1}=(w1) in (\scale * \y{w1}, -\scale * \x{w1}) -- (-\scale * \y{w1}, \scale * \x{w1});
            \node[](w2) at (0.50, 0.70){};
            \draw[clr2, thick] let \p{w2}=(w2) in (\scale * \y{w2}, -\scale * \x{w2}) -- (-\scale * \y{w2}, \scale * \x{w2});
            \node[](w3) at (0.00, 1.00){};
            \draw[clr3, thick] let \p{w3}=(w3) in (\scale * \y{w3}, -\scale * \x{w3}) -- (-\scale * \y{w3}, \scale * \x{w3});
            \node[](w4) at (-0.50, 0.50){};
            \draw[clr4, thick] let \p{w4}=(w4) in (\scale * \y{w4}, -\scale * \x{w4}) -- (-\scale * \y{w4}, \scale * \x{w4});
    
            \sv{----}{-0.48}{-1.00}{----}{\normalsize};
            \sv{---+}{-1.00}{-0.35}{---+}{\normalsize};
            \sv{--++}{-1.00}{0.28}{--++}{\normalsize};
            \sv{-+++}{-0.51}{1.00}{-+++}{\normalsize};
            \sv{+---}{0.51}{-1.00}{+---}{\normalsize};
            \sv{++--}{1.00}{-0.28}{++--}{\normalsize};
            \sv{+++-}{1.00}{0.31}{+++-}{\normalsize};
            \sv{++++}{0.46}{1.00}{++++}{\normalsize};
        \end{scope}
        \end{tikzpicture}
        \caption{The Argmaxable label assignments are the 1-alternating vectors, $\argmaxable{\vv{W}}=\Alt{4}{1}$.}
    \end{subfigure}%
    \label{fig:altexample}
    \caption{Visual evidence of~\cref{thm:altfeas}. a) We construct a BSL having $n=4$ labels and $d=2$ features parametrised by $\vv{W} \in \R^{4 \times 2}$ such that all maximal minors are positive, i.e. $\vv{W} \in \pgrass{n=4}{d=2}$. (b) The rows of the matrix are binary classifiers, we demarcate the decision boundaries for each classifier using a dashed line. (c) We assign each region a sign vector corresponding to which labels the BSL would flag as active for an input falling in that region. As per~\cref{thm:altfeas}, exactly the $(d-1) = 1$-alternating outputs are argmaxable. More generally, for $d=2k+1$, all $k$-active outputs are argmaxable (see~\cref{app:3d}).
    }
\end{figure*}
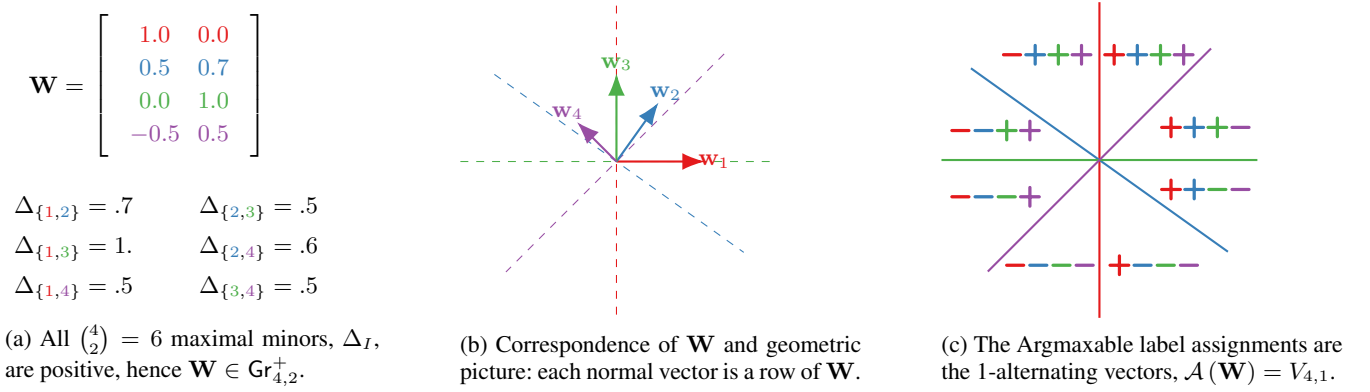

\subsection{$k$-Active Label Assignments}

Label combinations in real-world MLC datasets are often sparse~\citep{Babbar2019}; it is unlikely an image will contain more than $k$ objects or that a clinical document will be assigned more than $k$ clinical codes, where $k \approx \mathcal{O}(\log n)$ is a dataset dependent upper bound on the number of active labels~\citep{Jain2019}.
We now show how to guarantee that all $k$-active outputs are argmaxable by controlling the parametrisation of a BSL. 
We first formalise what a $k$-active label combination is below.
\begin{definition}
For a label assignment $\vv{y}$ we define $\act{\vv{y}}$ to be the number of active labels in $\vv{y}$, i.e:
\begin{equation}
    \act{\vv{y}} =
    \sum\nolimits_{i=1}^{n}
    \mathbf{1}_{\vv{y}} \left\{  \vv{y}_i = +\right\}
\end{equation}
e.g. $\act{-----}=0$ and $\act{+--+-} = 2$.
\end{definition}
\begin{definition}
The \textbf{$k$-active assignments} on $n$ labels are:
\begin{equation}
\Act{n}{k} = \left\{ \vv{y} \in \{+, -\}^{n}  : \act{\vv{y}} \leq k \right\}
\end{equation}
\end{definition}
For example, the MIMIC-III dataset~\citep{Johnson2016, mullenbach2018} has $n=8921$ labels, but no example has more than $80$ active labels. We now show how to guarantee that all label assignments in $\Act{n}{k}$ are argmaxable.

\subsection{$k$-Active Argmaxability Guarantees}
\label{sec:criteria}
Our goal in this section is to prove~\cref{thm:card}, which states that a \textit{general criterion}
for guaranteeing 
all $k$-active labels are argmaxable is that the weight matrix $\vv{W} \in \R^{n \times d},\, d \geq 2k +1$ that parametrises the BSL has \emph{maximal minors that agree in sign and are non-zero}.
As such, we next introduce maximal minors and the family of matrices with the above property. We prove our result by showing that the argmaxable label assignments for this family of matrices are ``$2k$-alternating'' and that these subsume the $k$-active ones.

A \textbf{maximal minor} $\Delta_I$ of a $n \times d$ matrix, $n > d$, is the determinant of any $d \times d$ submatrix formed by deleting the $n-d$ rows not indexed by $I$. For example, in~\cref{subfig:maxminors}, all maximal minors are positive.
We denote with $\pgrass{n}{d}$ the set of all matrices $\vv{W} \in \R^{n \times d}$ whose maximal minors are non-zero and agree in sign (see~\cref{app:grass}).
To prove~\cref{thm:card}, we use the following facts known about $k$-alternating outputs.
\begin{definition}
For a label assignment $\vv{y}$ we define $\alt{\vv{y}}$ to be the number of sign changes encountered when reading the sequence of labels from left to right, i.e:
\begin{equation}
    \alt{\vv{y}} = \sum_{i=1}^{n-1}
    \mathbf{1}_{\vv{y}} \{ \vv{y}_i \neq \vv{y}_{i+1} \}
\end{equation}
\end{definition}
e.g: $\alt{++---}=1$ and $\alt{++-+-}=3$.
\begin{definition}
The \textbf{$k$-alternating assignments} on $n$ labels are:
\begin{equation}
\Alt{n}{k} = \left\{ \vv{y} \in \{+, -\}^{n} : \alt{\vv{y}} \leq k \right\}
\end{equation}
\end{definition}
\begin{lemma}
$\vv{y}$ is $k$-active $\implies \vv{y}$ is $2k$-alternating.
\label{lem:altsp}
\end{lemma}
See~\cref{proof:altsp} for reasoning.
\begin{theorem}
\citep{Gantmakher1961} see \citep[Theorem 1.1]{Karp2017}.
If all maximal minors of $\vv{W} \in \R^{n \times d}$ are non-zero and have the same sign, all label assignments $\vv{y}$ computed as $\vv{y} = \sign{\vv{Wx}},\, \vv{x} \in \R^d$ are $d-1$ alternating.
$\vv{W} \in \pgrass{n}{d} \implies \alt{\vv{y}} \leq d - 1$  .
\label{thm:upperbound}
\end{theorem}
\begin{theorem}
For $\vv{W} \in \pgrass{n}{d}$ the argmaxable label assignments are the $(d-1)$-alternating vectors.
$\vv{W} \in \pgrass{n}{d} \implies \argmaxable{\vv{W}} = \Alt{n}{d-1}$.
\label{thm:altfeas}
\end{theorem}
See~\cref{fig:altexample} for intuition and~\cref{proof:altfeas} for a proof.
\begin{theorem}
Consider a BSL parametrised by $\vv{W} \in \pgrass{n}{2k+1}$ which predicts label assignments using argmax prediction, $\vv{y} = \sign{\vv{Wx}}$. All $k$-active label assignments are argmaxable:
$\Act{n}{k} \subset \argmaxable{\vv{W}}$.
\label{thm:card}
\end{theorem}
\begin{proof}
From~\cref{thm:altfeas}, for $\vv{W}$ in $\pgrass{n}{2k+1}$ the set of $2k$-alternating label assignments $\Alt{n}{2k}$ is argmaxable. Then, from \cref{lem:altsp}, we have $\Act{n}{k} \subseteq \Alt{n}{2k} = \argmaxable{\vv{W}}$, and therefore all $k$-active labels are argmaxable.
\end{proof}

In summary, we have showed that if $\vv{W} \in \pgrass{n}{2k+1}$, all $k$-active outputs are argmaxable, but we have not given a concrete implementation. We next introduce the DFT layer, a practical and efficient member of the $\pgrass{n}{2k+1}$ family.

\subsection{DFT Layers}

Herein we engineer a specific BSL parametrisation that satisfies \cref{thm:card} and relies on the Discrete Fourier Transform~(DFT). 
In addition to this property,
the DFT is also appealing because:
a) it allows us to reduce the number of learnable parameters as the DFT coefficients can be fixed~\footnote{In early experiments we parametrised and learned the $t_i$ of the DFT matrix but found little impact on the results. 
} 
and b) we can compute the activation of the DFT Layer in $\mathcal{O}(n\log{n})$ time via the Fast Fourier Transform (see~\cref{app:dft}) instead of a more expensive $\mathcal{O}(n^2)$ generic matrix-vector product.
We next describe a truncated DFT matrix and show that it provides the guarantees we seek. 
For the DFT matrix, we truncate frequencies larger than $k$:

\resizebox{.96\linewidth}{!}{
\begin{minipage}{1.23\linewidth}
\begin{align}
\dft{n}{2k+1}=&
\setlength\arraycolsep{2pt}
\begin{bmatrix}
\frac{1}{\sqrt{n}} & \sqrt{\frac{2}{n}}\cos{t_1} & \sqrt{\frac{2}{n}}\sin{t_1} & \cdots & \sqrt{\frac{2}{n}}\cos{k t_1}  & \sqrt{\frac{2}{n}}\sin{k t_1} \\
\frac{1}{\sqrt{n}} & \sqrt{\frac{2}{n}}\cos{t_2} & \sqrt{\frac{2}{n}}\sin{t_2} & \cdots & \sqrt{\frac{2}{n}}\cos{k t_2} & \sqrt{\frac{2}{n}}\sin{k t_2} \\
\vdots & \vdots & \vdots & \ddots & \vdots & \vdots \\
\frac{1}{\sqrt{n}} &\sqrt{\frac{2}{n}} \cos{t_n} & \sqrt{\frac{2}{n}}\sin{t_n} & \cdots & \sqrt{\frac{2}{n}}\cos{k t_n} &\sqrt{\frac{2}{n}} \sin{k t_n} \\
\end{bmatrix}, \nonumber
\\ & t_i = \frac{2 \pi (i-1)}{n},\, i \in [n]
\end{align}
\end{minipage}
}
\begin{lemma}
\label{lem:dft}
A truncated DFT matrix $\dft{n}{2k+1} \in \pgrass{n}{2k+1}$.
\end{lemma}

We need to show that the maximal minors of $\dft{n}{2k+1}$ are non-zero and agree in sign. See~\cref{proof:dft} for a proof.

\shortparagraph{Problem: Regions can become too small.}
While in practice we could use the fixed DFT Layer as defined above and rely on an expressive feature encoder to do the heavy lifting, if the number of labels, $n$, is much greater than the number of features, $2k+1$, it becomes hard to classify some label assignments with large confidence.
This is because segmenting a low dimensional space with very many hyperplanes induces regions that become arbitrarily small shards. %
In argmaxability terms, if we fix an $\epsilon$ and increase the number of labels, all $k$-active labels are argmaxable but increasingly more are not $\epsilon$-argmaxable, see left in \cref{fig:eargmax}. 
This is a problem for any classifier $\vv{W}$, because for training and for generalisation we need to project points into large enough regions.
However, we found that DFT layers are more susceptible to it (\cref{app:unargmax-test}) than general BSL (see~\cref{app:small} for a more detailed explanation).

\shortparagraph{Solution: Slack variables.}
A practical way to deal with small regions is to increase the dimensionality of the features by adding learnable \textit{slack variables}, see~\cref{fig:eargmax}.
Crucially, when doing so we retain our guarantees, as we show below.
\begin{figure}[!t]
    \centering
    \includegraphics[width=.99\columnwidth]{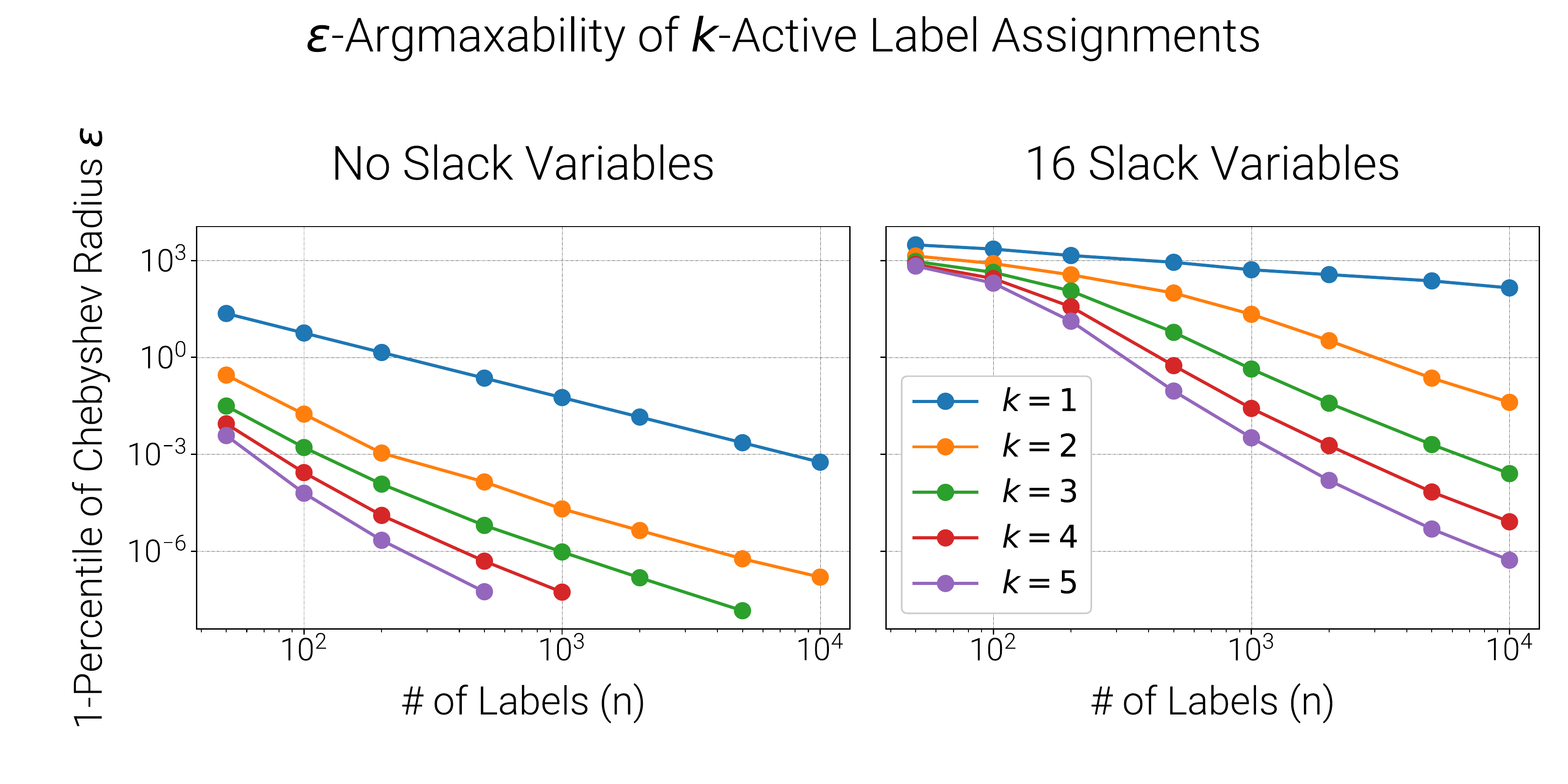}
    \caption{Left: As we increase the number of labels $n$ for the DFT Layer, the radii of the regions shrink, making them harder to predict in practice. Right: Adding slack variables ameliorates this problem. We plot $\epsilon$-argmaxability (\cref{def:eargmax}), measured here for the %
    1\% of labels that have radius less than that plotted. For the DFT Layer, i.e. $\vv{W}=\dft{n}{2k+1}$, all $k$-active label assignments are argmaxable, but as we increase $n$, some (see $k \geq 3$) cannot be detected at the precision of the LP ($10^{-8}$). Adding 16 randomly initialised slack columns, i.e. $\vv{W} = \begin{bmatrix} \dft{n}{2k+1}\,\,\vv{S} \end{bmatrix}$, makes the regions $\epsilon$-argmaxable with larger $\epsilon$.}
    \label{fig:eargmax}
\end{figure}

\begin{lemma}
\label{lem:slack}
Assume a label assignment $\vv{y}$ is argmaxable for a classifier $\vv{W} \in \R^{n \times d}$. Consider increasing the dimensionality of the features of the classifier $\vv{W}$ by adding $s$ more randomly initialised slack columns $\vv{S} \in \R^{n \times s}$. Then $\vv{y}$ is also argmaxable in $\vv{W}' = \begin{bmatrix} \vv{W}\,\,\vv{S} \end{bmatrix},\, \vv{W}' \in \R^{n \times (d + s)}$.
\begin{proof}
Consider the input feature vector for $\vv{W}'$, $\vv{x}'= \begin{bmatrix} \vv{x} \\ \vv{x}_s \end{bmatrix},\, \vv{x} \in \R^{d},\, \vv{x}_s \in \R^{s}$. Set $\vv{x}_s = \vv{0}$. Then notice that
$\vv{y} = \sign{\begin{bmatrix} \vv{W}\,\,\vv{S} \end{bmatrix} \begin{bmatrix} \vv{x} \\ \vv{0} \end{bmatrix}} = \sign{\vv{W} \vv{x}}
$
is equivalent to the original classifier, so if $\vv{y}$ is argmaxable in $\vv{W}$ it is also argmaxable in $\vv{W}'$ by setting $\vv{x}_s$ to zero.
\end{proof}
\end{lemma}
As such, we propose the DFT layer, which has $(2k+1) \times n$ fixed parameters which enforce argmaxability and $s \times n$ learnable parameters which give it flexibility.

\section{Experiments}
\label{sec:exps}
We now empirically evaluate the BSL and DFT layers on three MLC datasets and answer the following research questions: \textbf{RQ1)} Do BSLs have unargmaxable labels in practice?
\textbf{RQ2)} Can DFT layers guarantee that meaningful labels are argmaxable in practice?
\textbf{RQ3)} What is the trade-off between performance and the number of trainable parameters?
We answer the above after introducing the models and datasets.

\subsection{Model Setup}
We define the two MLC output layers we will compare, the \emph{BSL Layer} that is unconstrained and does not guarantee argmaxability of $k$-active outputs, and our \emph{DFT layer} which does. In our experiments we want to study the effect of varying the bottleneck width irrespective of the feature dimensionality of the encoder which varies across datasets and models. As such, we introduce an \emph{affine projection layer} (parametrised by $\vv{P}$ and $\vv{b}$) between the feature encoder and the \emph{linear classifier} (parametrised by $\vv{W}$). For simplicity, we do not include bias terms for the classifiers. For both models, we compute the logits $\vv{z} \in \R^n$ from the encoder activation $\vv{e} \in \R^e$ as:
\begin{equation}
\quad \vv{z} = \vv{W}\vv{x},\quad \vv{x} = \vv{P}\vv{e} + \vv{b}
\end{equation}
While the classifiers $\vv{W}$ have the same number of learnable parameters for both output layers,
the parametrisations differ, as we discuss next.
\subsubsection{BSL Output Layer}
For the BSL, the projection layer maps from $e$, the dimensionality of the encoder embeddings, to $d$, the feature dimensionality using $\vv{P} \in \R^{e \times d}$ and bias $\vv{b} \in \R^d$. This is followed by the linear classifier $\vv{W} \in \R^{n \times d}$.
\subsubsection{DFT Output Layer}
For the DFT, we first pick the maximum number of active labels, $k$, depending on the statistics of the dataset. We then set the number of slack dimensions to be $d$ so we can directly compare to the BSL.
As such, we have $\vv{P} \in \R^{e \times (2k + 1 + d)}$ and $\vv{b} \in \R^{(2k + 1 + d)}$, since we include $2k+1$ more features that map to the DFT columns of the classifier.
The learnable parameters of the classifier comprise $d$ slack columns. Conceptually, and for the purposes of checking the classifier with our LP, we construct the classifier by concatenating the fixed DFT matrix to the slack columns, i.e. $\vv{W} = \begin{bmatrix}\dft{n}{2k+1}\,\,\vv{S} \end{bmatrix}$, $\vv{W} \in \R^{n \times (2k + 1 + d)}$. In practice we compute the logits $\vv{z}$ efficiently as $\vv{z}=\operatorname{FFT}(\vv{x}_{:2k+1}) + \vv{S}\vv{x}_{2k+1:}$ (see \cref{app:dft}).
\paragraph{Computational Cost of DFT} Compared to the BSL, the cost of the DFT layer with $n$ labels is a) an additional cheap $\mathcal{O}(n\log{n})$ matrix vector multiplication and b) an additional $e \times (2k+1)$ trainable parameters in the projection layer.
However, for models with large $n$, we can easily offset the increase in parameters if we want, by modestly shrinking the slack $d$ of the DFT output layer. For example, the MIMIC-III CNN models~\citep{mullenbach2018} have $n=8921$ and $e=500$. For $k=80$, a DFT adds $500 \times 161$ parameters to the projection layer. We could offset this by decreasing $d$ in the output layer by only $\lceil \frac{500 \times 161}{8921} \rceil  = 10$. In~\cref{sec:results} we show that DFT layers obtain better performance with lower $d$ than BSLs, and as such can be more parameter efficient.
\paragraph{Faster Training of DFT}
\label{sec:fast}
For the DFT, we introduce an initialisation trick to speed up training. We exploit that a) $\vv{W}^{\mathsf{DFT}}$ is known and fixed and b) the outputs are $k$-active. Since the outputs are $k$-active, we would prefer to assign a probability $\frac{k}{n}$ to all labels when we start training. To achieve this, we can exploit the fact that the first column of $\vv{W}^{\mathsf{DFT}}$ is $\frac{1}{\sqrt{n}}$ and initialise the bias vector of the projection layer to be $[\sqrt{n}\operatorname{logit}(\frac{k}{n}),0, \ldots,0 ]$, where the logit function is the inverse of sigmoid. This way, assuming logits are close to zero when we begin training, the model will assign probability $\approx \frac{k}{n}$ to each label instead of $\approx \frac{1}{2}$. A similar bias initialisation idea for MLC was discussed in~\citep{Schultheis2021}, but it was not used in a neural network.
\begin{figure*}
    \begin{subfigure}[t]{0.432\linewidth}
    \centering
    \begin{tikzpicture}[inner sep=0pt, remember picture]
    \node (m) at (0,0) {\includegraphics[width=.2765\linewidth]{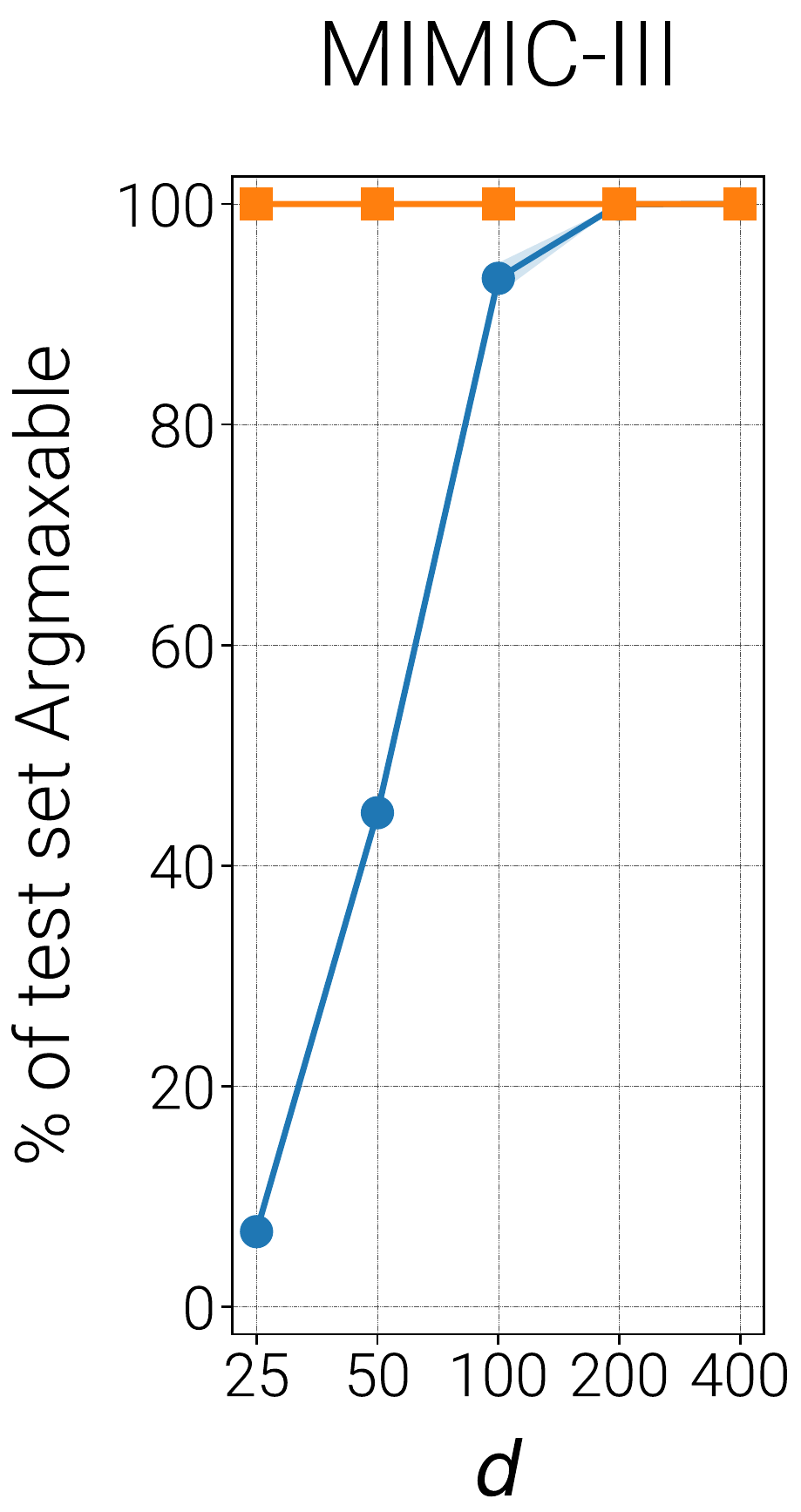}};
    \node[rotate=90, above of=m, yshift=.6cm] {\large Argmaxability};
    \end{tikzpicture}
    \begin{tikzpicture}[inner sep=0pt, remember picture]
    \node at (0,0) {\includegraphics[width=.25\linewidth]{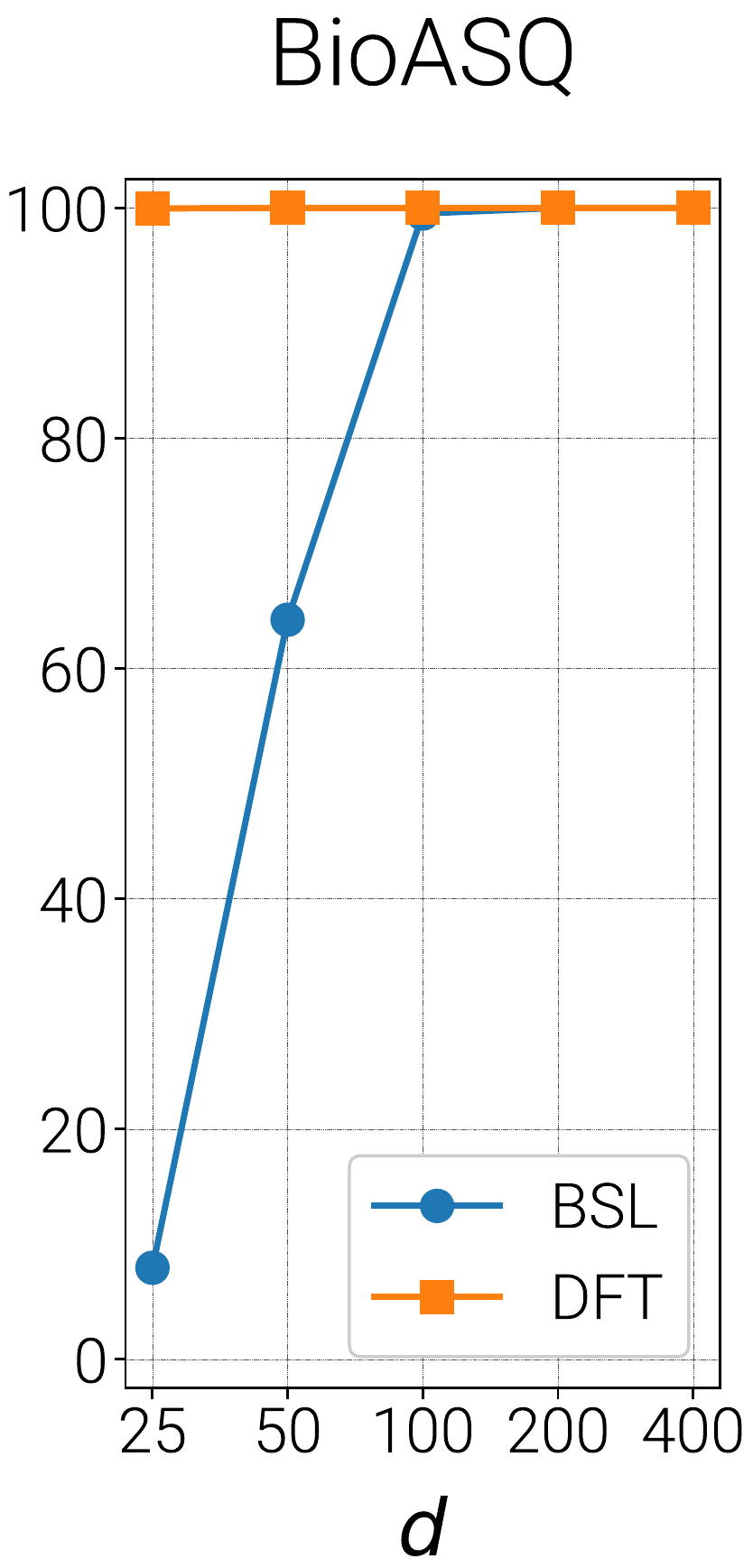}};
    \end{tikzpicture}
    \begin{tikzpicture}[inner sep=0pt, remember picture]
    \node at (0,0) {\includegraphics[width=.25\linewidth]{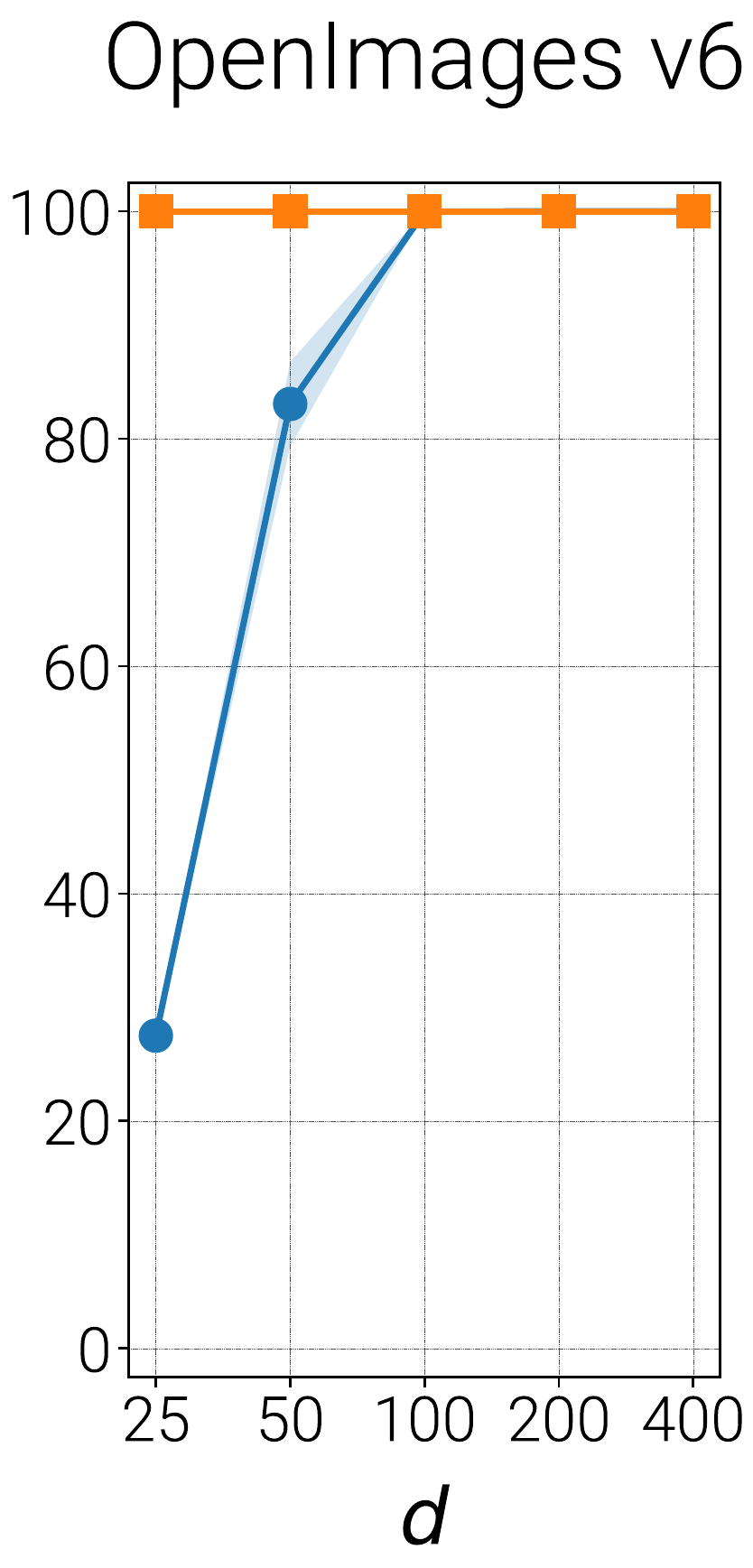}};
    \end{tikzpicture}
    \caption{Our DFT layer does not suffer from unargmaxable test examples. BSLs cannot provide any such guarantees.}
    \label{fig:mimic-argmax}
    \end{subfigure}
    \hfill
    \begin{subfigure}[t]{0.545\linewidth}
    \centering
    \begin{tikzpicture}[inner sep=0pt, remember picture]
    \node (m2) at (0,0) {\includegraphics[width=.29\linewidth]{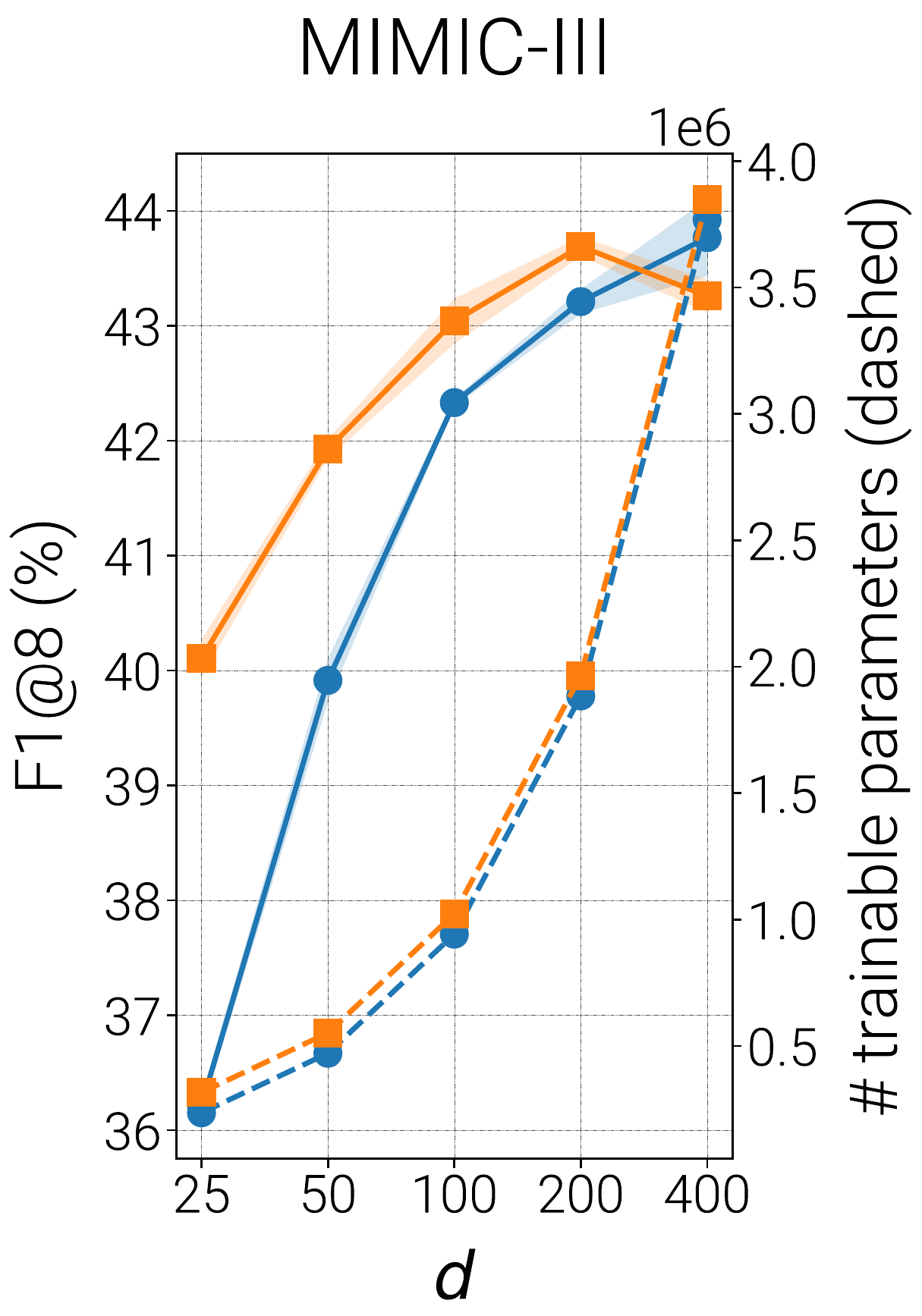}};
    \node[rotate=90, above of=m2, yshift=.9cm] {\large Parameter Efficiency};
    \end{tikzpicture}
    \begin{tikzpicture}[inner sep=0pt, remember picture]
    \node at (0,0) {\includegraphics[width=.29\linewidth]{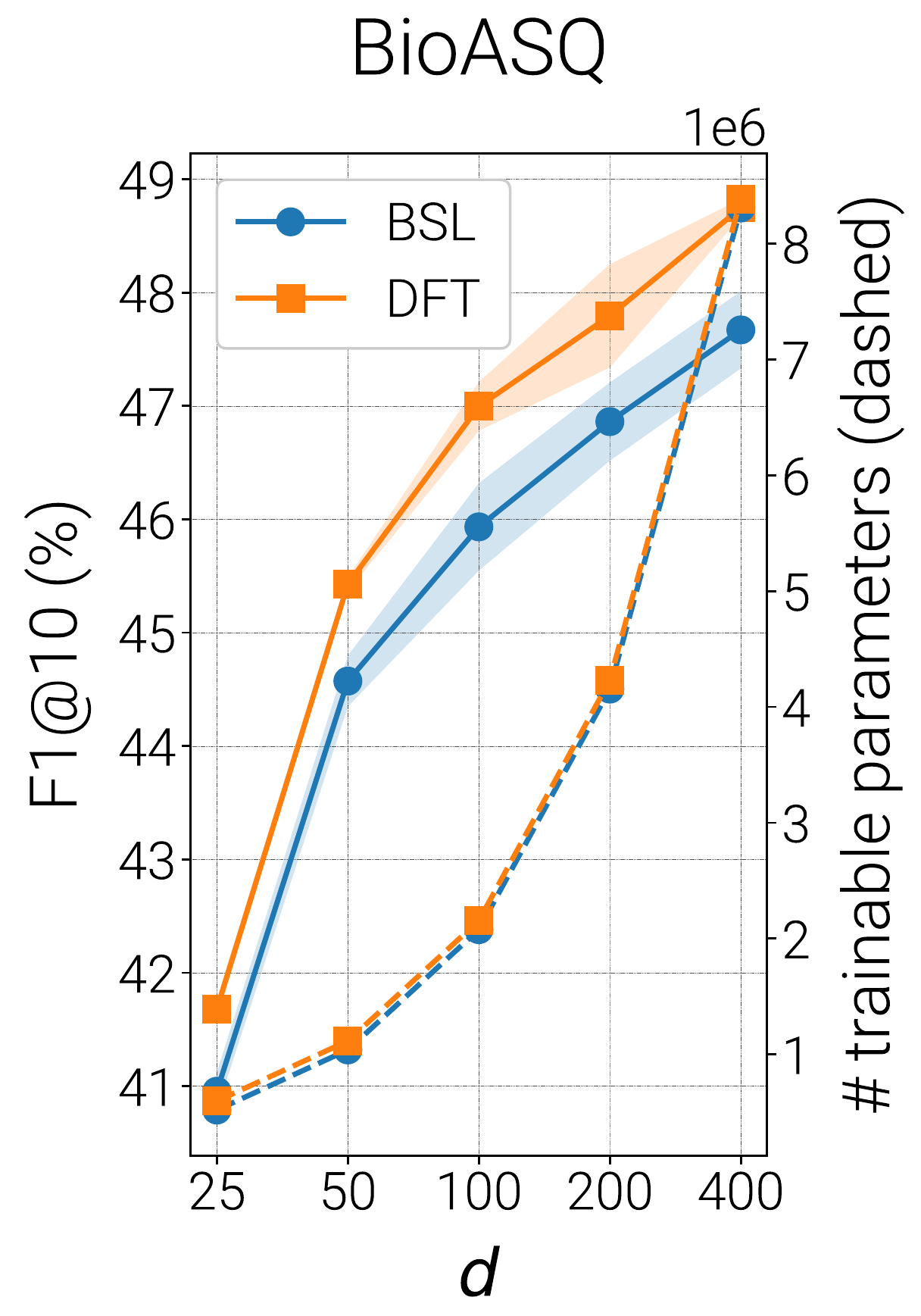}};
    \end{tikzpicture}
    \begin{tikzpicture}[inner sep=0pt, remember picture]
    \node at (0,0) {\includegraphics[width=.29\linewidth]{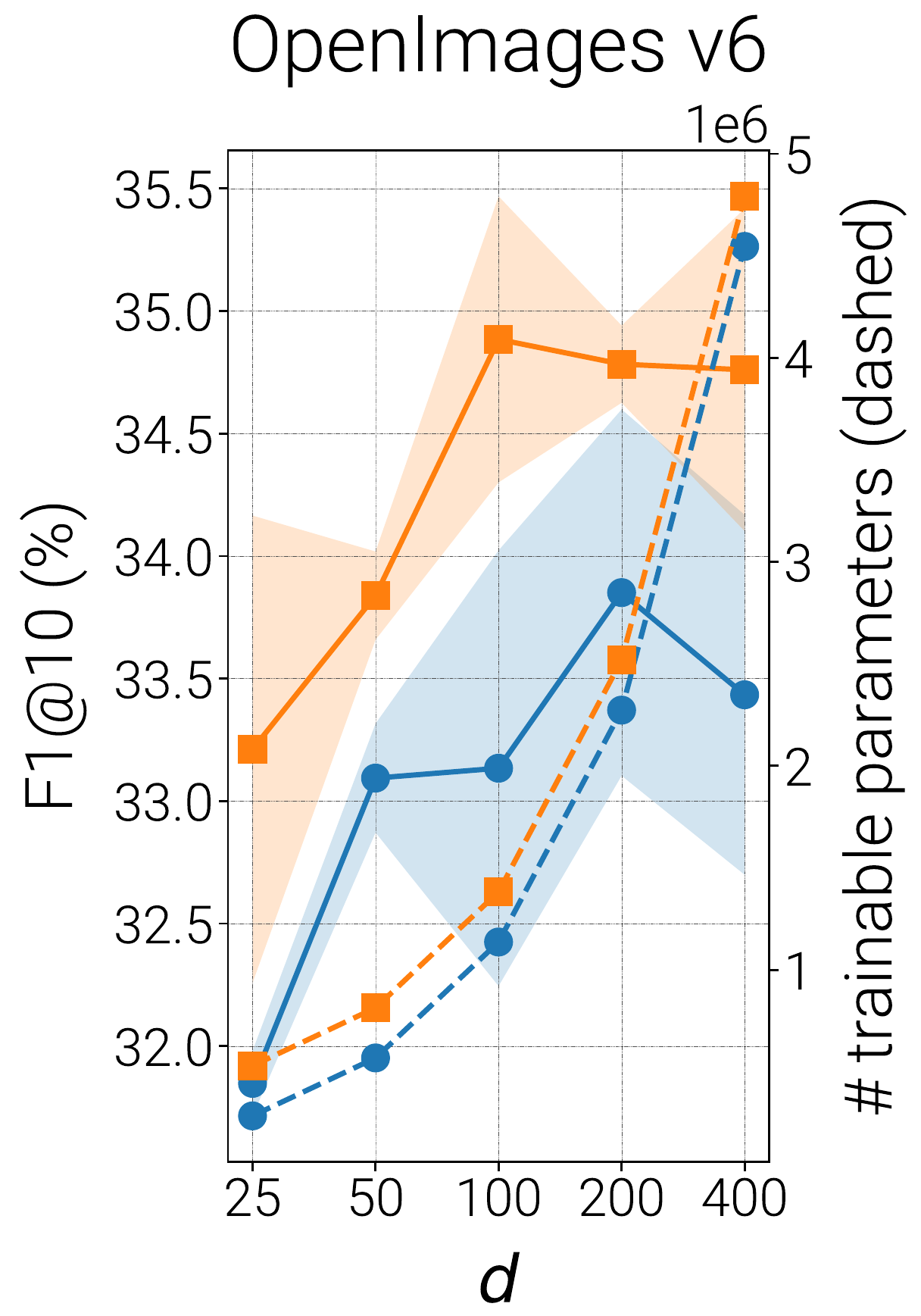}};
    \end{tikzpicture}
    \caption{Our DFT layer is more parameter efficient. F1 is overall better than the BSL, so we can match BSL's F1 with fewer trainable parameters (smaller $d$).
    }
    \label{fig:mimic-f1}
    \end{subfigure}
    \caption{
    Comparison of BSL and our DFT output layer on three MLC datasets. We vary, $d$, the number of trainable dimensions and plot the mean and std (shaded) over 3 runs with different random seeds. Left: We use the LP to verify the output layers on the test sets. When $d < 200$ a large percentage of label configurations becomes unargmaxable for the BSL, in contrast to our DFT. Right: Performance in F1@8 or F1@10 (left axis) in terms of the number of trainable parameters (right axis, dashed lines). Our DFT is better ($d \leq 200$) or comparable to the BSL (MIMIC-III, $d > 200$). We can therefore retain a high F1@k score for DFT even if we reduce the number of trainable parameters by making the bottleneck narrower. For example, on BioASQ we can surpass the F1@10 of the BSL that has $d=200$ with a DFT of $d=100$, which has about 50\% fewer trainable parameters.
    }
\end{figure*}
\begin{table}[b]
    \centering
    \small
    \begin{tabular}{lccccc}
    \toprule
     MLC Dataset   &     n    & k   & N      & encoder \\ \toprule
 MIMIC-III    & $8921$   & 80      &  44k   & CNN (e=500)\\
 BioASQ task A & $20000$  & 50      & 100k   & BERT (e=768)\\
 OpenImagesV6  & $8933$   & 50      & 108k   & TResnet  (e=2432)\\
 \bottomrule
    \end{tabular}
    \caption{Setup: Number of labels, $n$, max number of active labels, $k$, and number of training examples, $N$.}
    \label{tab:data-stats}
\end{table}
\subsection{MLC Tasks (Datasets)}
We now introduce the three datasets. We summarise their important attributes in~\cref{tab:data-stats}.
See~\cref{app:repr} for details on dataset construction and more dataset statistics.
\subsubsection{Clinical Coding (MIMIC-III)}
\label{sec:exp-mimic}
We first test the DFT layer on MIMIC-III~\citep{Johnson2016}. For this safety critical application of clinical coding, the goal is to tag each clinical note with a set of relevant ICD-9 codes which describe findings (see \cref{fig:problem}). We retrain the CNN encoder model defined in~\citet{mullenbach2018} which has $n=8921$ and $e=500$.
We use the same word embeddings, preprocessed data, data splits, metrics (Prec@8) and hyperparameters reported in the paper~\citep{mullenbach2018}. We only change the learning rate of the Adam optimiser to $0.001$, as this improves results (as also found by~\citet{Edin2023}).

\subsubsection{Semantic Indexing (BioASQ Task A)}
Next, we focus on the $2021$ BioASQ semantic indexing challenge~\citep{Tsatsaronis2015,Nentidis2021OverviewOB,krithara2023}. For this task, we are given PubMed abstracts and asked to predict a set of relevant MeSH headings\footnote{\url{https://www.nlm.nih.gov/mesh/meshhome.html}} for each article.
We create dataset splits (see~\cref{app:dataset-proc} for details) with $n=20000$, making sure that all individual labels occur in both the train and test sets. We do so to avoid claiming a label assignment is unargmaxable when in fact it would be hard to predict the labels that constitute it.
We use $k=50$ as this is the maximum number of active labels per example for this dataset by construction.
We finetune PubMedBERT~\citep{pubmedbert}, a domain specific uncased BERT~\citep{devlin2018} encoder that has been pretrained on PubMed abstracts and has $e=768$.
We use early stopping with a patience of $10$ on the validation crossentropy loss.
\subsubsection{Image MLC (OpenImages v6)}
We use the OpenImages v6 dataset~\citep{Kuznetsova2018}
where the goal is to tag each image with objects that appear in it. 
Similar to the BioASQ case, we choose the label vocabulary $n=8933$ such that all examples in the train, validation and test set are covered. We pick $k=50$ since the training data has at most $45$ active labels per example. 
We finetune the TResnet~\citep{Ridnik2020} with $e=2432$ that~\citet{Baruch2020} pretrained for MLC on this dataset.
We use early stopping with a patience of $10$ on the validation crossentropy loss.

\subsection{Results}
\label{sec:results}
\textbf{RQ1) BSL: Unargmaxable outputs.}  We use the LP to verify the BSL on the test sets. 
As can be seen in~\cref{fig:mimic-argmax}, for $d > 200$ all the examples in the test set are argmaxable for the BSL. However, as we reduce $d$, the number of unargmaxable label assignments increases for all datasets. More specifically, we first get unargmaxable outputs at $d=200$ for MIMIC-III (see~\cref{app:inf-mimic}).
Analogous considerations can be drawn for BioASQ and OpenImages (\cref{fig:mimic-argmax}).
As such, we conclude that unargmaxability can indeed be an issue when 
$d$ is not large enough.
Note that one can never determine a $d$ that can always guarantee all label configurations of interest to be argmaxable: even an exhaustive verification on all test samples does not imply that future unseen configurations will be argmaxable.
\textbf{RQ2) DFT: Argmaxability.} While we know that the DFT layer guarantees argmaxability in theory, we also see that it works in practice (apart from a handful of outputs when $d=25$, which are $\epsilon$-unargmaxable, see~\cref{app:unargmax-test}). Crucially, these guarantees also apply to unseen $k$-active label assignments; this would be impossible to enforce with BSLs, as discussed.
\textbf{RQ3) Parameter Efficiency.}
 In addition, we turn to~\cref{fig:mimic-f1}, and see that the DFT layer outperforms the BSL layer by a wide margin for small $d$. This allows us to match the performance of the BSL using a DFT with smaller $d$ and hence fewer trainable parameters.
As can be seen, in some cases DFT layers obtain better or comparable performance with up to 50\% less trainable parameters.
\textbf{Faster Training.}
Lastly, a benefit of DFT layers is that they converge faster due to the initialisation trick (~\cref{sec:fast}), see~\cref{app:dft-conv} for a comparison of training times.

\section{Related Work}
\label{sec:relatedwork}
\shortparagraph{Limitations of Low-rank Parametrisations in MLC.}
Previous work has shown that the low-rank assumption can be problematic for MLC: long-tail labels lack strong linear dependencies, making the output label matrix high-rank~\citep{Bhatia2015, Xu2016, Tagami2017}.
As such, low-rank approximations suffer from high reconstruction error.
Herein, we highlight a more tangible limitation:
meaningful label assignments can be unargmaxable.
While we prove that we can retain argmaxability with a very low-rank matrix if the outputs are sparse, we concur that low-rank parametrisations can still be problematic, especially for predicting long-tail labels.
For effective prediction of such labels,~\citet{Babbar2019} showed the importance of making a model robust to input perturbations. We similarly find that we need $\epsilon$-argmaxability with a large enough $\epsilon$, both for effective training and prediction.
\shortparagraph{Sign Rank of a Matrix.}
Consider approximating a matrix of sign vectors with a matrix $\vv{M}$ of rank $r$. The smallest $r$ for which all sign vectors can be recovered by taking the sign of $\vv{M}$ is known as its sign rank~\citep{Alon2015,Neumann2016}, see~\cref{app:signrank} for more details.
Our rank $2k+1$ approximation of a $k$-active label matrix is a non-trivial upper bound on its sign rank, see also~\citet[Theorem 1.1, part i]{Alon1985}.
~\citet[Theorem 6]{Chanpuriya2020} use a similar construction
to show that adjacency matrices of graphs of bounded degree $k$ can be embedded in $2k+1$ dimensions.

\section{Discussion and Conclusion}
\label{sec:conclusion}

Through extensive experiments on three multi-label classification datasets we have shown that BSLs -- which are still ubiquitous in large MLC scenarios (\cref{sec:bsl}) -- can have unargmaxable label assignments when the label vocabulary is much larger than the number of input features (\cref{sec:exps}).
We believe practitioners should be aware that selecting the feature dimension of a BSL is not an innocuous hyperparameter search: by lowering the rank of the output layer they are potentially making some label assignments impossible to predict.
Unargmaxability impacts the generalisation and trustworthiness of these classifiers, since BSLs cannot guarantee that meaningful label assignments will not be missed at test time or that the classifiers will not be targetted by adversarial attacks~\citep{zhu2019transferable,aghakhani2021bullseye}.

However, we showed that this does \emph{not} have to be the case. We provided a family of parametrisations that guarantee that all label assignments having up to $k$ active labels are argmaxable (\cref{sec:criteria}) and implemented our DFT output layer which converges faster than a BSL, is up to 50\% more parameter efficient, and outperforms or matches the BSL on three widely used MLC datasets.

Our findings also prompt several avenues for future work.
As we make the sigmoid bottleneck narrower, the argmaxable regions get smaller (\cref{fig:eargmax}), making label assignments harder to predict robustly.
Can we parametrise output layers in a way that guarantees $\epsilon$-argmaxability for large $\epsilon$?
Moreover, while we focussed on BSLs, there are many families of output layers to analyse. E.g. those that
a) partition the label vocabulary so that classifiers do not all compete in a shared feature space~\citep{Yu2020},
b) parametrise the classifier based on the input~\citep{mullenbach2018} and
c) predict labels autoregressively~\citep{simig-etal-2022,Kementchedjhieva2023}.
It is an open question for future work to verify that these models do not admit unargmaxable outputs and to robustify them otherwise.

\section*{Acknowledgements}
 
We would like to thank Lorenzo Loconte, Rickey Liang, and Marina Potsi for feedback and Emile van Krieken for feedback and catching omissions. Furthermore, Asif Khan for many fruitful discussions and help connecting the Cyclic Polytope to truncated DFT, Nikolay Bogoychev for discussing multi-label classification datasets, Jesse Sigal for discussions on cyclic permutations, Matúš Falis for discussing multi-label classification and datasets, Tom Sherbourne for discussing DFT and low-pass filters and Ivan Titov and Henry Gouk for helpful suggestions and advice.

AG was supported by the Engineering and Physical Sciences Research Council (EP/R513209/1). AV was supported by the ``UNREAL: Unified Reasoning Layer for Trustworthy ML'' project (EP/Y023838/1) selected by the ERC and funded by UKRI EPSRC.

\bibliography{aaai24}

\newpage
\clearpage
\appendix

\section{Vectors in General Position}
\begin{definition}
We say $n$ vectors are in \textbf{general position} in $\R^d$ if all subsets of $d$ or fewer vectors is linearly independent~\citep{Cover1965}.

\end{definition}

Intuitively, this means that the vectors are no more dependent than they need to be in $\R^{d}$.
No 2 vectors lie on a line through the origin, no 3 vectors lie on a plane through the origin, no $d$ vectors lie in a $d-1$ subspace.
Algebraically, it means that a $d \times d$ matrix formed by stacking any $d$ out of the $n$ vectors together has non-zero determinant.

\label{sec:gp}

\section{The Cyclic Polytope}
\label{app:cyc}

To prove \cref{lem:dft} in the next section, we leverage the theorem in \citet{Cordovil2000}, who present their result in terms of the homogenisation of the Cyclic Polytope. Herein we highlight the equivalence of the homogenisation of the Cyclic Polytope to a DFT matrix with total order constraints on the $t_i$. We will need this to make claims about the maximal minors of the DFT matrix.

\paragraph{The Cyclic Polytope} We start with the standard definition. The $\cyclic{n}{d}$ with $n$ vertices in $\R^d$ is defined as the convex hull of $n > d$ distinct points on the moment curve in $\R^d$~\citep[Example 0.6, p.11]{Ziegler1994LecturesOP}. The moment curve is a map $m(t): \R \mapsto \R^d$ defined as:
\begin{align}
\cyclic{n}{d} =& \convex{m(t_1), m(t_2), \ldots, m(t_n)} \\
\text{where}\quad m(t) =& \begin{bmatrix} t \\ t^2 \\ \vdots \\ t^d \end{bmatrix},\quad
t_1 < t_2 < \ldots < t_n
\end{align}

\paragraph{Homogenisation} In order to study the affine dependencies of a point configuration in $\R^d$ (e.g. the face structure of a polytope), it is convenient to map them to linear dependencies of a vector configuration in $\R^{d+1}$ and study those instead. This can be done via homogenisation: we map points in $\R^d$ to vectors in $\R^{d+1}$ by appending an extra dimension and fixing it to $1$~\citep[Section 6.2]{Ziegler1994LecturesOP}, i.e. $\operatorname{hom}: \R^d \rightarrow \R^{d+1},\, \operatorname{hom}(\vv{x}) = \begin{pmatrix}1\\ \vv{x}\end{pmatrix}$. To abide by earlier notation, we stack the vertices of the Cyclic Polytope in the rows of the matrix. For the standard Cyclic Polytope on the moment curve we get a Vandermonde matrix:

\resizebox{.85\columnwidth}{!}{
\begin{minipage}{\columnwidth}
\begin{align}
\cyclic{n}{d} &= \begin{bmatrix}
t_1 & t_1^2 & \cdots & t_1^d \\
t_2 & t_2^2 & \cdots & t_2^d \\
\vdots & \vdots & \ddots & \vdots \\
t_n & t_n^2 & \cdots & t_n^d \\
\end{bmatrix} \\
\affine{\cyclic{n}{d}} &= \begin{bmatrix}
1 & t_1 & t_1^2 & \cdots & t_1^d \\
1 & t_2 & t_2^2 & \cdots & t_2^d \\
\vdots & \vdots & \vdots & \ddots & \vdots \\
1 & t_n & t_n^2 & \cdots & t_n^d \\
\end{bmatrix}
\label{eq:cyclic}
\end{align}
\end{minipage}
}

\paragraph{Trigonometric Cyclic Polytope} Instead of the moment curve, \citet{Gale1963} used the trigonometric moment curve to construct the Cyclic Polytope, see also~\citet[Section 3]{Donoho2005}. We note that its homogenisation is the truncated DFT matrix:

\resizebox{.95\columnwidth}{!}{
\begin{minipage}{1.05\columnwidth}
\begin{align}
\cyclic{n}{2k} &= \begin{bmatrix}
\cos{t_1} & \sin{t_1} & \cdots & \cos{k t_1} & \sin{k t_1}\\
\cos{t_2} & \sin{t_2} & \cdots & \cos{k t_2} & \sin{k t_2}\\
\vdots & \vdots & \ddots & \vdots \\
\cos{t_n} & \sin{t_n} & \cdots & \cos{k t_n} & \sin{k t_n}\\
\end{bmatrix}\\
\affine{\cyclic{n}{2k}} &= \begin{bmatrix}
1 & \cos{t_1} & \sin{t_1} & \cdots & \cos{k t_1} & \sin{k t_1}\\
1 & \cos{t_2} & \sin{t_2} & \cdots & \cos{k t_2} & \sin{k t_2}\\
\vdots & \vdots & \vdots & \ddots & \vdots \\
1 & \cos{t_n} & \sin{t_n} & \cdots & \cos{k t_n} & \sin{k t_n}\\
\end{bmatrix} \nonumber
\\ & t_i = \frac{2 \pi (i-1)}{n},\, i \in [n]
\label{eq:trig-cyclic}
\end{align}
\end{minipage}
}

In~\cref{proof:dft}, we will use the connection between the DFT matrix and the Cyclic Polytope to prove Lemma 2.

\section{Grassmanians}

The \textbf{Grassmanian} $\grass{n}{d}$
is the set of all $d$-dimensional subspaces of $\R^n$. We will think of $\grass{n}{d}$ as the space of rank $d$ matrices $\vv{W} \in \R^{n \times d}$ with $1 \leq d \leq n$. More precisely, a single member of $\grass{n}{d}$ defines a subspace and corresponds to all $\vv{W}$ that form a basis for that subspace, i.e. all matrices obtained by multiplying the basis on the right by any invertible $d \times d$ matrix.~\footnote{Column operations change the basis of the columnspace but not the subspace itself.} 
We can subdivide the whole Grassmanian into more granular matrix families by considering the sign of the maximal minors.
In this paper, we are especially interested in the \textbf{Totally Positive Grassmanian} $\pgrass{n}{d}$, the space of $n \times d$ matrices for which all maximal minors are non-zero and have the same sign~(see also~\citet[Definition 3.1]{Postnikov2006}).
We note that if all maximal minors are non-zero, then the rows of $\vv{W}$ are in general position.

\label{app:grass}

\section{Proofs}

\subsection{\cref{lem:altsp}}
\label{proof:altsp}
$\vv{y}$ is $k$-active $\implies \vv{y}$ is at most $2k$-alternating.
\begin{proof}
Construct a $k$-active $\vv{y}$ of length $n$ from the all inactive $\vv{y}$ by flipping all signs after any of the $n-1$ positions between labels. We need at most $2k$ distinct flips, i.e. when the active labels are not adjacent and do not include the $1^{st}$ or $n^{th}$ label, and less otherwise. For example, $\ys{-+--}$ can be produced with 2 flips starting from the all inactive vector: $\ys{----} \rightarrow \ys{-+++} \rightarrow \ys{-+--}$.
\end{proof}

\subsection{\cref{thm:altfeas}}
Here we show that for a classifier $\vv{W} \in \pgrass{n}{d}$ the argmaxable label assignments are the $(d-1)$-alternating vectors. We do so by invoking Theorem 2 to guarantee that any sign vector has at most $d-1$ sign changes. We then use a counting argument to show that the number of argmaxable sign vectors is the same as the number of alternating sign vectors, and as such the argmaxable sign vectors are exactly the $(d-1)$-alternating vectors.
\label{proof:altfeas}
\begin{proof}
Since $\vv{W} \in \pgrass{n}{d}$, all maximal minors are non-zero and hence the rows of $\vv{W}$ are in general position (see~\cref{def:gp}).
By invoking~\cref{thm:cover}, the number of argmaxable label assignments is $\left|\argmaxable{\vv{W}}\right| = 2 \sum\nolimits_{d'=0}^{d-1} \binom{n-1}{d'}$. Note that this is exactly the number of $(d-1)$-alternating label assignments $\left|\Alt{n}{d-1}\right|$ as we elaborate on next. The binomial coefficient comes from choosing $d'$ out of the $n-1$ positions between labels to flip the sign and we sum over all possible number of sign changes up to $d-1$. For each alternating $\vv{y}$, we can produce another by flipping all signs, hence the leading multiplier $2$. Now, from~\cref{thm:upperbound}, none of the label assignments can have more than $d-1$ sign changes; hence they must be exactly the $(d-1)$-alternating label assignments: $\argmaxable{\vv{W}}=\Alt{n}{d-1}$.
\end{proof}

\subsection{\cref{lem:dft}}
We now show that the maximal minors of the DFT matrix are non-zero and have the same sign.
In fact, we show this is true more generally for any matrix with rows that are homogenised vertices of an even dimensional Cyclic Polytope.
\label{proof:dft}
\begin{proof}
The DFT matrix corresponds to the homogenisation of the vertices of a Cyclic Polytope~\citep{Gale1963, Cordovil2000}, see~\cref{app:cyc}, where the vertices of the Cyclic Polytope are in $2k$ dimensions before being homogenised. Cyclic Polytopes in $2k$ dimensions (even dimension) are rigid\footnote{See \citep[Section 6.6]{Ziegler1994LecturesOP} for more details on rigidity, and Example 6.3 for a polytope that is not rigid.}: their face structure determines their geometric structure~\citep[Theorem 5.1]{Cordovil2000}. Their geometric structure is that of a Uniform Alternating Oriented Matroid~\footnote{By this we mean that the matrix obtained by the homogenisation of any Cyclic Polytope has the structure of the Uniform Alternating Oriented Matroid.}. Any matrix realisation of a Uniform Alternating Oriented Matroid has maximal minors that agree in sign and are non-zero~\citep[Proposition 3.1]{Cordovil2000, Sturmfels1987}, see chirotope representation of Oriented Matroids~\citep[Section 9.4]{OrientedMatroids1999}. Any normalisation of the DFT matrix obtained by scaling columns using non-zero scalars does not alter the columnspace of the matrix, and hence the oriented matroid structure is unchanged: the orthants intersected by the columnspace are the same.
\end{proof}

\section{DFT Layer as FFT}

We can use the FFT to speed up computations, as we will show that computing the logits $\vv{z}$ is equivalent to computing the truncated inverse DFT of the input $\vv{x}$, if we reinterpret the vector $\vv{x}$ that has $2k+1$ entries as the coefficients of $k+1$ complex numbers. Let us start from the inverse DFT, that computes the complex signal in the time domain from the frequency domain. We use $n', d'$ and $k'$ as variables to avoid confusion with $n, d$ and $k$, which we have already defined as constants throughout the paper. Denote the complex frequency component for frequency $k'$ by $X_{k'}$ and the signal at time $n'$ by $x_{n'}$, we have:
\begin{equation}
x_{n'} = \sum_{k'=0}^{n-1} X_{k'} \left[ \cos \left(\frac{2 \pi n'}{n}k' \right) + i \sin \left( \frac{2 \pi n'}{n}k'\right)\right]
\end{equation}
We take the real part of the iDFT, to obtain:
\begin{align*}
&\real{x_{n'}} = \real{\sum_{k'=0}^{n-1} X_{k'} \left[ \cos \left(k' t_{n'} \right)
+ i \sin \left( k' t_{n'} \right)\right]} \\
  &= \real{\sum_{k'=0}^{n-1} (a_{k'} + ib_{k'}) \left[ \cos \left(k' t_{n'} \right)
+ i \sin \left( k' t_{n'}\right)\right]}  \\
  &= \sum_{k'=0}^{n-1} \left[a_{k'} \cos \left(k' t_{n'} \right)
- b_{k'} \sin \left( k' t_{n'}\right)\right]  \\
\end{align*}
If we truncate the iDFT to the first $k$ frequencies, we get:
\begin{align*}
\real{x_{n'}} &= \sum_{k'=0}^{k} \left[a_{k'} \cos \left(k' t_{n'} \right)
- b_{k'} \sin \left( k' t_{n'}\right)\right]  \\
\end{align*}
We will now match the coefficients $a_{k'}$ and $b_{k'}$ to corresponding elements in $\vv{x}$ (ignoring scaling factors).
From the earlier computation of $\vv{W}\vv{x}$, we rewrite the logits $\vv{z}$ as below:
\begin{align}
\vv{z}_{n'} &= {\vv{w}^{(n')}}^\T\vv{x} \\
         &= \vv{x}_1 + \sum_{k'=1}^k \left[ \vv{x}_{2k'} \cos\left(k' t_n\right) + \vv{x}_{2k' + 1} \sin\left(k' t_n\right)\right]
\end{align}
From which we see that we can write the DFT layer as a truncated Inverse DFT by matching the coefficients
of the sines and cosines: $\vv{x}_1 = a_0,\, \vv{x}_{2k'} = a_{k'}$ and $\vv{x}_{2k'+1} = - b_{k'}$. See also our code \textbf{test\_dft\_equivalence.py}.
From this perspective, this paramatrisation is a low-pass filter.
\label{app:dft}

\section{Unargmaxable Test Examples}
\label{app:inf-mimic}
We now introduce two fine-grained measurements of argmaxability, eps-argmaxability and $1$-argmaxability. We discuss a few additional insights in the results section below.

\subsection{Argmaxability Measurements}

\paragraph{eps-argmaxable}
This is the estimate of argmaxability we can get at the precision of our LP, which is eps=$10^{-8}$. This roughly means that we can only detect regions which can contain a ball with a radius that is larger than $10^{-8}$.

\paragraph{1-argmaxable}
As we discussed in the paper, some label assignments may be $\epsilon$-argmaxable with a very small $\epsilon$, which makes it hard to predict such label assignments in practice. We therefore also report $1$-argmaxability, assuming -- by a fairly large margin -- that regions with radius $1$ are large enough to be easily predicted in practice.

\subsection{Results}
In ~\cref{tab:inf-bioasq,tab:inf-mimic,tab:inf-open} we tabulate the number of argmaxable label combinations on the MIMIC-III, BioASQ and OpenImages v6 test sets. We obtain the following insights.

a) For the DFT, when we make $d$ very small, e.g. $d=25$, a handful of label combinations in BioASQ are not eps-argmaxable according to the LP. Out of these, one example was found to be infeasible, while for the remaining ones numerical difficulties were encountered by the LP, i.e. Gurobi returned status 12: ``\emph{Operation terminated due to unrecoverable numerical difficulties}''. The above issues are likely caused because of dimensionality pressures, the label combinations are $\epsilon$-argmaxable but with a very small $\epsilon$ that cannot be detected with the precision of current LPs (we use eps=$10^{-8}$, as we discussed in~\cref{sec:verify}). This highlights the importance of our proofs, since our results would be tricky to verify using empirical methods alone. Moreover, as we discussed in~\cref{sec:conclusion}, being able to guarantee $\epsilon$-argmaxability with a large $\epsilon$ is important future work, since while we showed that our current solution of adding slack variables works in practice, if we increase the pressure on the bottleneck by making $d$ small enough, we can still run into $\epsilon$-argmaxability issues.

b) We note that the behaviour of BSLs in terms of eps-argmaxability and $\epsilon$-argmaxability is quite different when compared to DFT. For BSLs, if a label combination is argmaxable, it is very often also $\epsilon$-argmaxable. On the other hand, for the DFT some label combinations are argmaxable but are not $\epsilon$-argmaxable, highlighting that the regions do indeed exist, but they can shrink quite a bit in size due to the reduced dimensionality. See also~\cref{app:small}.
\label{app:unargmax-test}
\begin{table}[h!]
\centering\caption{Median number of $eps$-argmaxable and $1$-argmaxable label assignments over 3 random seeds on the dev and test sets of MIMIC-III. Takeaway: BSL layers have \rr{unargmaxable labels} starting from $d=200$ but it does not have to be this way. DFT layers resolve this problem and make all examples \gr{argmaxable}, but when slack dimensionality is very small, the regions are too small to detect with the precision of the LP.
Note that even if a BSL is able to argmax all test label configurations, this does not imply it will be able to guarantee so for meaningful but unseen future configurations.}
\label{tab:inf-mimic}
\begin{tabular}{ll@{\hspace{2em}}cc@{\hspace{2em}}cc}
\toprule
\multirow{2}{*}{split}     &  \multirow{2}{*}{$d$}   & \multicolumn{2}{l}{\# $eps$-Argmaxable} & \multicolumn{2}{l}{$1$-Argmaxable} \\
&    &            BSL &        DFT &                BSL &        DFT \\
\midrule
dev & 25  &      \rr{  128} &  \gr{1631} &           128 &  1572 \\
     & 50  &     \rr{  781} &  \gr{1631} &           781 &  1625 \\
     & 100 &     \rr{ 1533} &  \gr{1631} &          1533 &  1631 \\
     & 200 &     \gr{ 1631} &  \gr{1631} &          1631 &  1631 \\
     & 500 &     \gr{ 1631} &  \gr{1631} &          1631 &  1631 \\
     \midrule
test & 25  &     \rr{  229} &  \gr{3371} &           229 &  3216 \\
     & 50  &     \rr{ 1515} &  \gr{3371} &          1515 &  3356 \\
     & 100 &     \rr{ 3137} &  \gr{3371} &          3137 &  3370 \\
     & 200 &     \rr{ 3370} &  \gr{3371} &          3370 &  3371 \\
     & 500 &     \gr{ 3371} &  \gr{3371} &          3371 &  3371 \\
\bottomrule
\end{tabular}
\end{table}
\begin{table}[h!]
\begin{tabular}{llcccc}
\toprule
\multirow{3}{*}{split}     & \multirow{3}{*}{d}    & \multicolumn{2}{c}{\# $eps$-Argmaxable} & \multicolumn{2}{c}{\# $1$-Argmaxable} \\
      &     &       \multicolumn{2}{c}{out of 10000}    &          \multicolumn{2}{c}{out of 10000}   \\
      &   &           BSL  &  DFT   &           BSL      &   DFT  \\
\midrule
test & 25  &         \rr{  880} &  \rr{ 9995} &                879 &   8174 \\
     & 50  &         \rr{ 6498} &  \gr{10000} &               6483 &   9925 \\
     & 100 &         \rr{ 9951} &  \gr{10000} &               9950 &  10000 \\
     & 200 &         \gr{10000} &  \gr{10000} &              10000 &  10000 \\
     & 400 &         \gr{10000} &  \gr{10000} &              10000 &  10000 \\
\bottomrule
\end{tabular}
\centering\caption{Median number of $eps$-argmaxable and $\epsilon$-argmaxable label assignments over 3 random seeds on the test set of BioASQ. Note that even if a BSL is able to argmax all test label configurations, this does not imply it will be able to guarantee so for meaningful but unseen future configurations.}
\label{tab:inf-bioasq}
\end{table}
\begin{table}[h!]
\begin{tabular}{llcccc}
\toprule
\multirow{3}{*}{split}     & \multirow{3}{*}{d}    & \multicolumn{2}{c}{\# $eps$-Argmaxable} & \multicolumn{2}{c}{\# $1$-Argmaxable} \\
      &     &       \multicolumn{2}{c}{out of 10000}    &          \multicolumn{2}{c}{out of 10000}   \\
      &   &           BSL  &  DFT   &           BSL      &   DFT  \\
\midrule
test & 25  &         \rr{  2758} &  \gr{10000} &               2757 &   9981 \\
     & 50  &         \rr{  8439} &  \gr{10000} &               8435 &  10000 \\
     & 100 &         \rr{  9997} &  \gr{10000} &               9997 &  10000 \\
     & 200 &         \gr{ 10000} &  \gr{10000} &              10000 &  10000 \\
     & 400 &         \gr{ 10000} &  \gr{10000} &              10000 &  10000 \\
\bottomrule
\end{tabular}
\centering\caption{Median number of $eps$-argmaxable and $\epsilon$-argmaxable label assignments over 3 random seeds on the test set of OpenImages. Note that even if a BSL is able to argmax all test label configurations, this does not imply it will be able to guarantee so for meaningful but unseen future configurations.}
\label{tab:inf-open}
\end{table}

\section{DFT Train Efficiency}
Herein we provide more information showing how the DFT layer speeds up convergence and requires less training time to reach equivalent performance. We focus on the most demanding datasets, BioASQ and OpenImages v6, and highlight two perspectives. In~\cref{fig:dft-conv}, we show how the training loss evolves over time. Meanwhile, in~\cref{fig:dft-time} we compare the number of hours it took for the BSL and DFT models to converge on BioASQ and OpenImages v6. Both figures show that the DFT layer leads to faster convergence, as it starts training with a lower loss due to the initialisation trick and maintains its lead over the BSL throughout training.
\label{app:dft-conv}
\begin{figure}[h!]
\begin{subfigure}{.4\columnwidth}
\centering
\includegraphics[width=\linewidth]{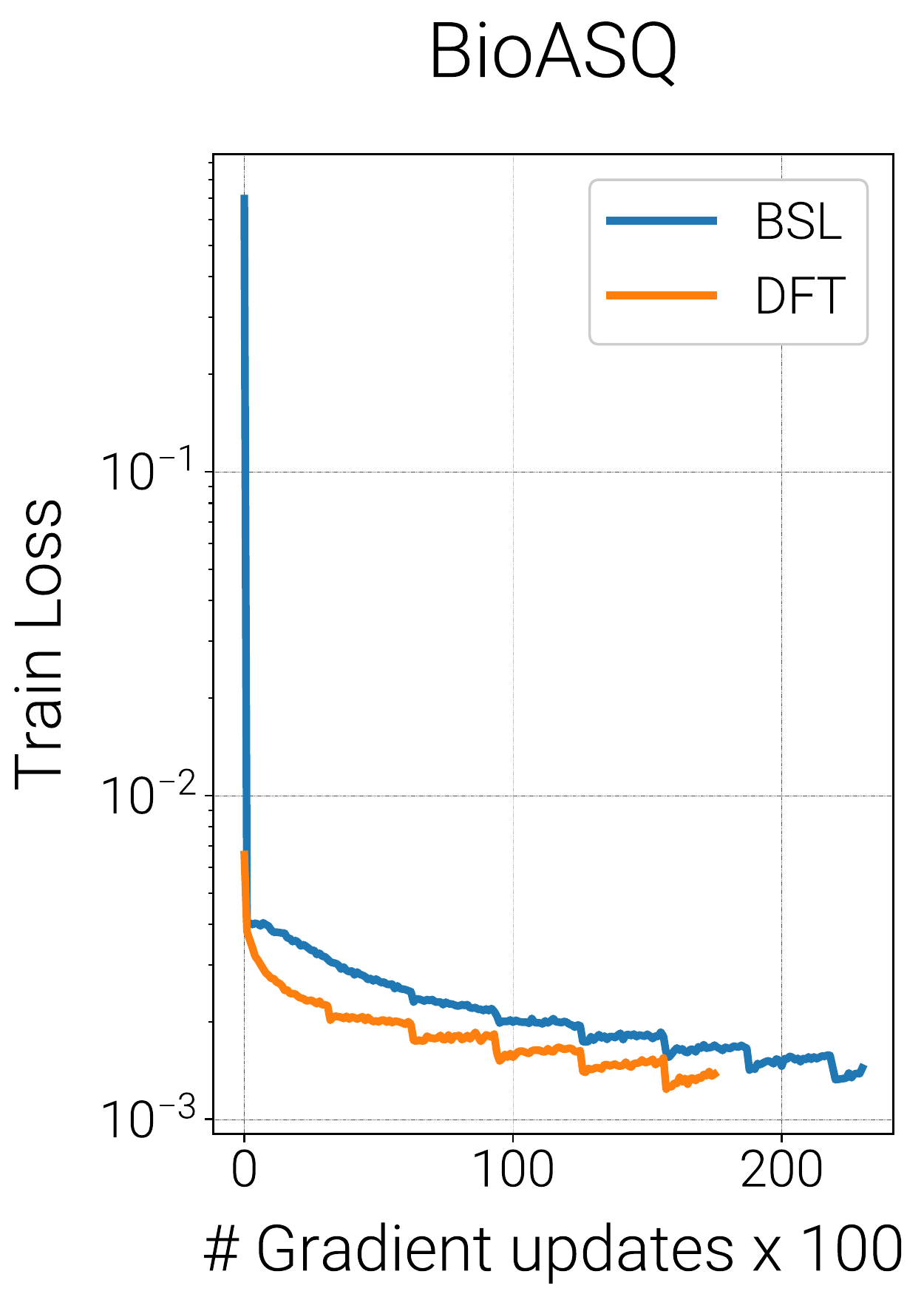}
\end{subfigure}
\hfill
\begin{subfigure}{.4\columnwidth}
\centering
\includegraphics[width=\linewidth]{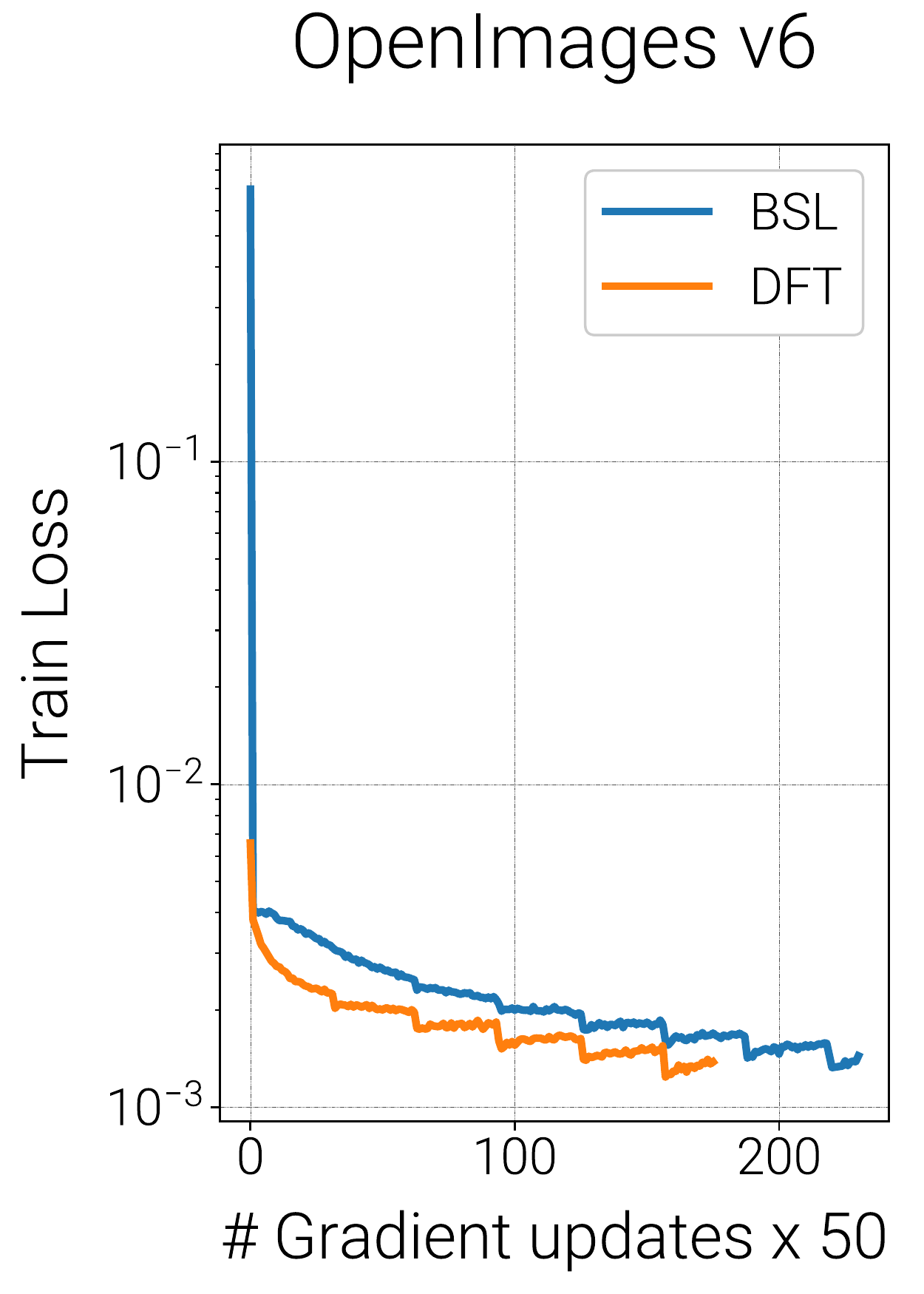}
\end{subfigure}
\caption{Comparison of BSL and DFT for $d=100$, in terms of the training cross entropy loss (y-axis, log scale) as training evolves (x-axis). Due to the initialisation trick, the DFT starts training at a lower loss and converges faster.}
\label{fig:dft-conv}
\end{figure}
\begin{figure}[h!]
\begin{subfigure}{.4\columnwidth}
\centering
\includegraphics[width=\linewidth]{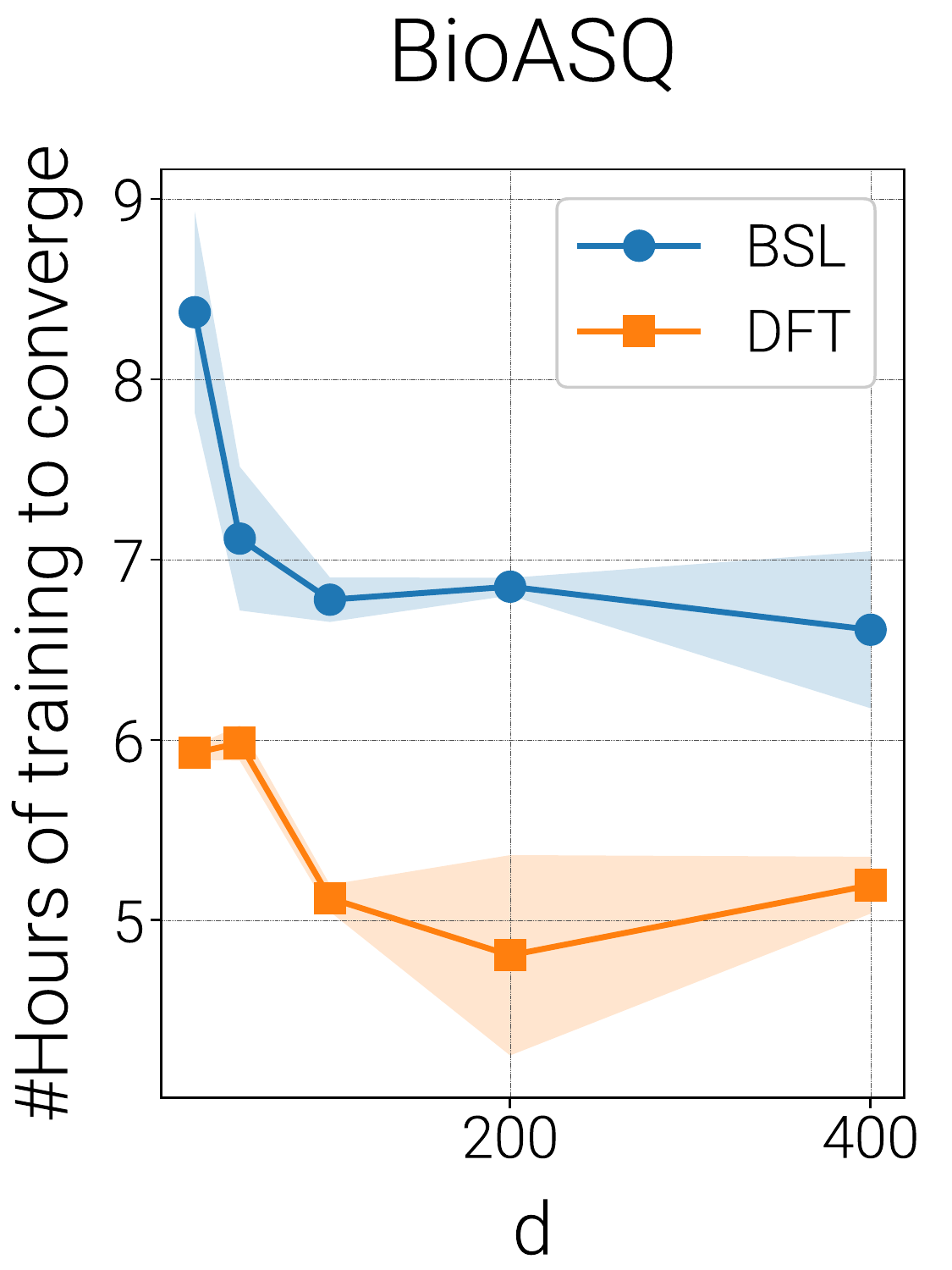}
\end{subfigure}
\hfill
\begin{subfigure}{.4\columnwidth}
\centering
\includegraphics[width=\linewidth]{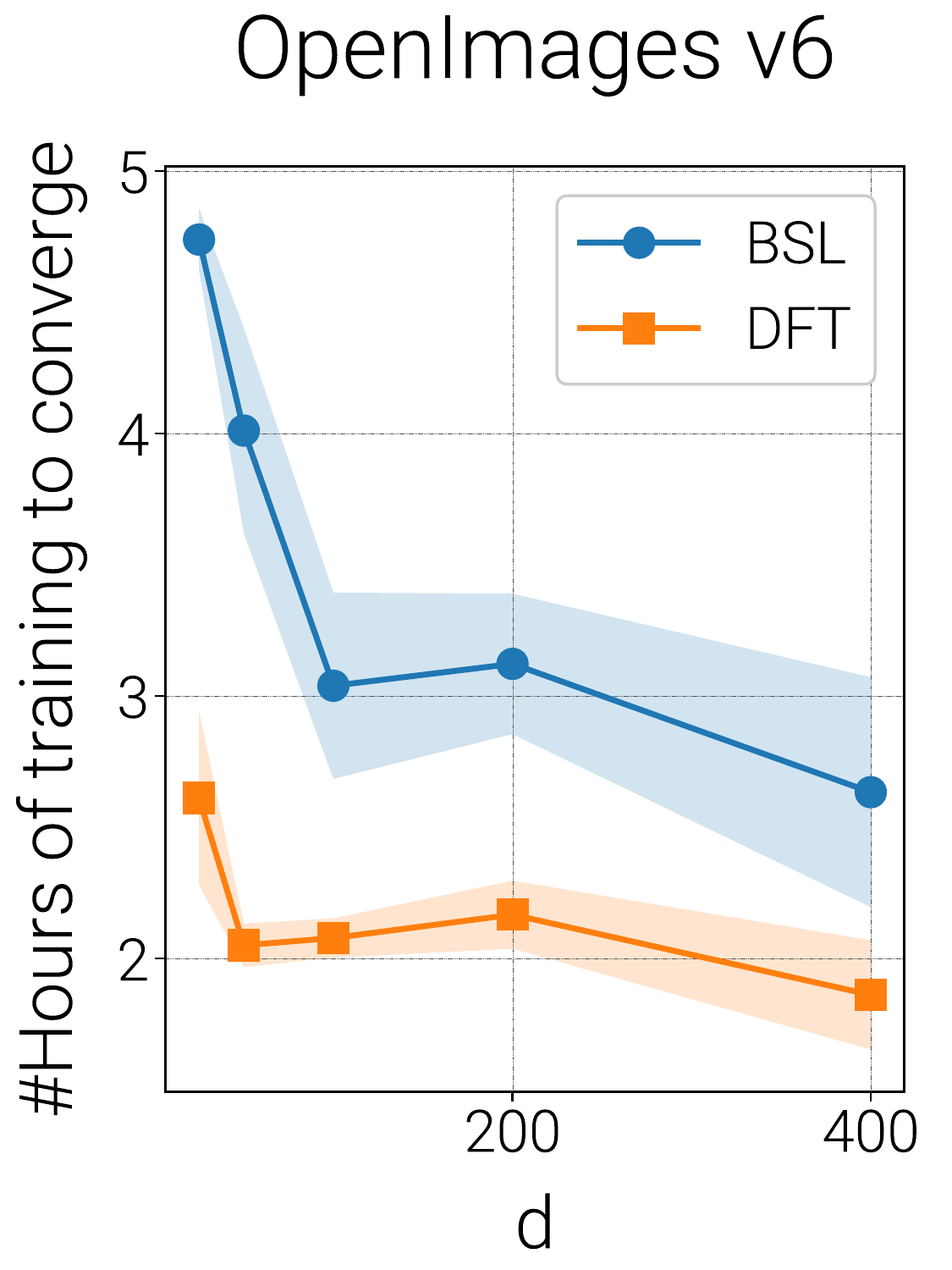}
\end{subfigure}
\caption{Comparison of BSL and DFT in terms of training time (in hours) to convergence. As can be seen, the DFT converges about 25\% faster.}
\label{fig:dft-time}
\end{figure}

\section{Reproducibility}
\label{app:repr}
\subsection{Dataset Access and Preprocessing}
\subsubsection{MIMIC-III} While de-identified, the MIMIC-III dataset~\citep{Johnson2016} contains sensitive and detailed information on the clinical care of patients. As such, permission to access this dataset needs to be requested, as explained here.\footnote{\url{https://mimic.mit.edu/docs/gettingstarted/}} We used the same preprocessing, setup and train, validation and test splits as~\citep{mullenbach2018}. See their github repository for more details.

\subsubsection{BioASQ Task A 2021}
The BioASQ Task A dataset~\citep{Tsatsaronis2015,Nentidis2021OverviewOB} is available after registering for the task on the BioASQ website.\footnote{\url{http://participants-area.bioasq.org/}}
We created dataset splits which cover $n=$20k labels using a 1m subset of the 2021 BioASQ task A dataset.
We construct train, validation and test splits by sampling examples, making sure that all individual labels (not label combinations) occur in both the train and test sets.
We encode the concatenation of the journal, title and abstract as text input.
Due to the context size limitation of BERT, we truncate the input to the first $512$ subwords. See our code for more details.

\subsubsection{OpenImages v6}
The OpenImages v6 dataset~\citep{Kuznetsova2018} can be accessed from the project website.\footnote{\url{https://storage.googleapis.com/openimages/web/index.html}}
We downloaded the images from CVDF, which was linked from the website.
Since the dataset is very large, we only used $N=108228$ images, these had hashes that started with $1$ and were available as a single zip download from CVDF.
Since the validation and test sets are also large, we validate and test on the first 5k examples of the validation set and the first 10k examples of the test set, correspondingly. For preprocessing, we simply reshape all images to 448x448, as done in~\citet{Baruch2020}.

\subsection{Dataset Statistics}
We tabulate the sizes of the dataset splits in~\cref{tab:hyper}.
A histogram of the number of active labels can be seen in~\cref{fig:cardinalities}.
\label{app:datastats}

\begin{table}[t]
\resizebox{\columnwidth}{!}{%
\begin{tabular}{llll}
\toprule
 & \multicolumn{3}{c}{\Large Datasets} \\
 & MIMIC-III & BioASQ task A & OpenImagesV6 \\
\midrule
n           &   8921      &    20000      &    8933      \\
d           &   25-400  &    25-400     &     25-400   \\
encoder     &       CNN &    PubmedBERT &   T-Resnet-L \\
pretrained      &     no    &    yes        &   yes        \\
encoder dim $e$ &       500 &           768 &         2432 \\
lr (encoder)    &     0.001 &       0.00005 &       0.0001 \\
lr (classifier) &     0.001 &         0.001 &        0.001 \\
batch size  &        16 &            32 &            64 \\
patience    &        10 &            10 &           10 \\
eval every  &  1 epoch   &  500 steps       &       250 steps \\
criterion  &       P@8 &         Valid loss &        Valid loss \\
\midrule
\# train   &       44k  &         100k      &        108k       \\
\# valid   &       1.6k &           5k      &           5k      \\
\# test    &       3.3k &           10k     &          10k      \\
\bottomrule
\end{tabular}
}
\caption{Dataset attributes and model hyperparameters.}
\label{tab:hyper}
\end{table}
\label{app:dataset-proc}

\begin{figure*}[t]
    \centering
    \includegraphics[width=.32\textwidth]{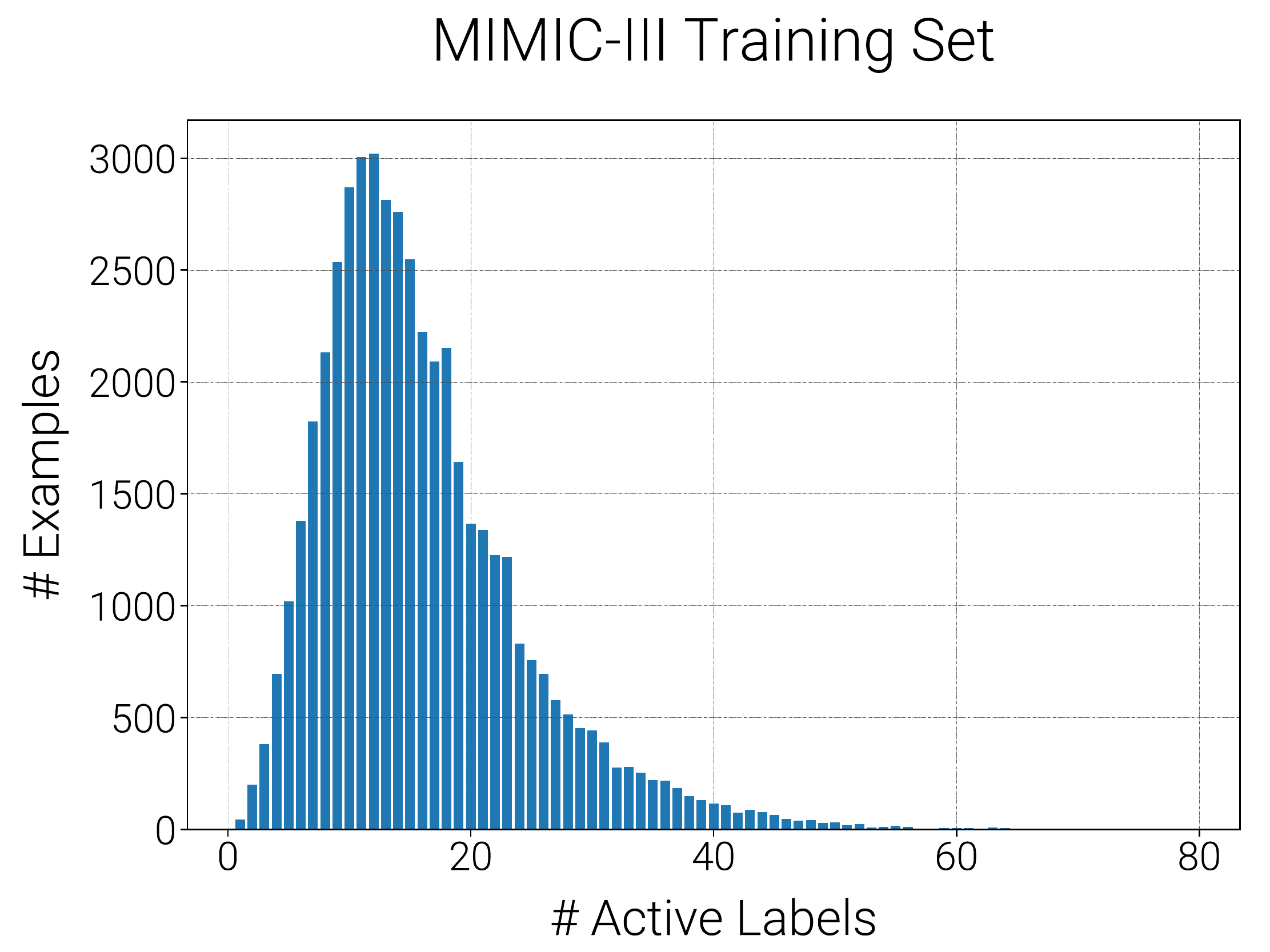}
    \hfill
    \includegraphics[width=.32\textwidth]{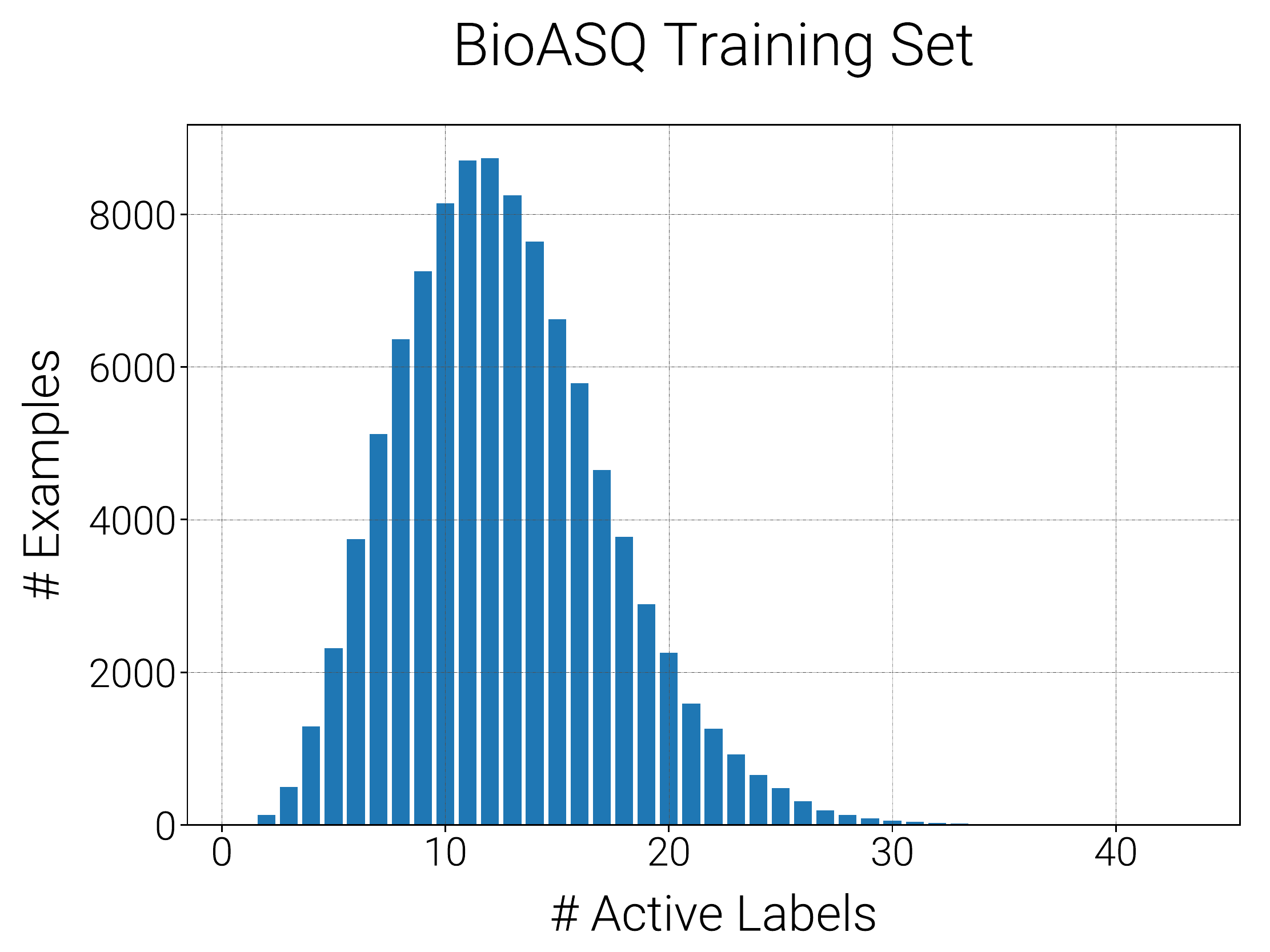}
    \hfill
    \includegraphics[width=.32\textwidth]{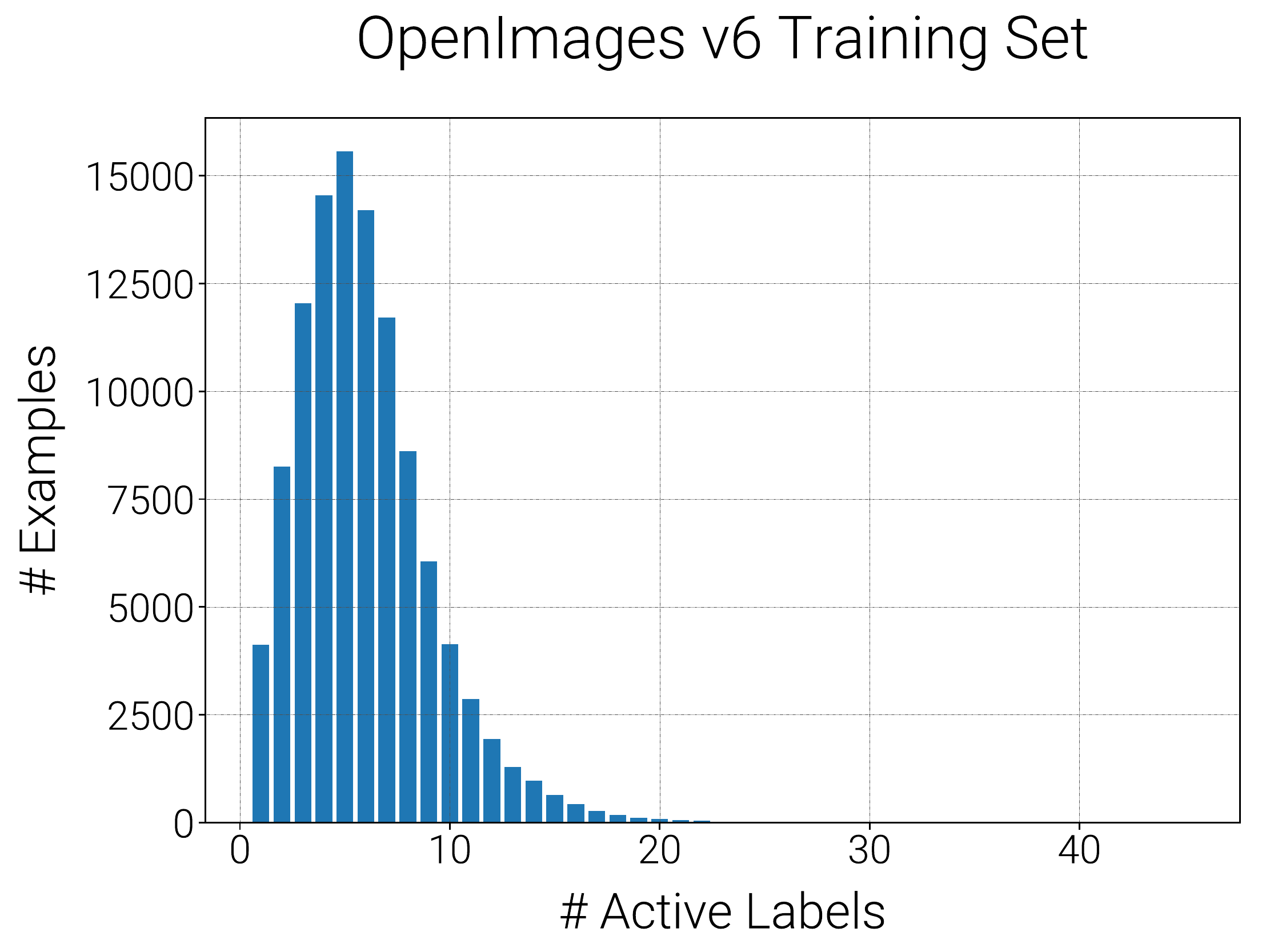}
    \caption{Comparison of the number of active labels on the training sets of MIMIC-III, BioASQ and OpenImages datasets. All three datasets have a long tail of less sparse active labels (large $k$). MIMIC-III has the longest ($k \leq 80$).}
    \label{fig:cardinalities}
\end{figure*}

\subsection{Hyperparameters}
In order to study the sensitivity of our methods to random initialisation we ran all experiments three times, once per random seed in (0, 1, 2). We train all models using binary crossentropy loss.
We summarise all hyperparameters in~\cref{tab:hyper}.
We use early stopping for all models with a patience of $10$. The stopping criterion is Prec@8 for MIMIC-III and Validation Loss for BioASQ and OpenImages.
\label{app:traindeets}

\subsection{Training Resources and Train time}
We train and evaluate our models on GPUs. For the MIMIC-III dataset we used a NVIDIA 3090 GPU that has 24Gb of RAM and for the remaining two datasets we used an NVIDIA RTX A6000 which has 48Gb of RAM.
The experiments took about two weeks of compute. More specifically, the MIMIC-III runs took 9 hours, the OpenImages runs took 85 hours and the BioASQ runs took 188 hours. 
We verified our models using the LP on CPU on a cluster with an AMD EPYC 7452 32-Core Processor and 500GB of RAM. We parallelised verification by running the LP for each label in parallel to others of the same model. With this setup and running on 50 threads the verification of 10k examples takes from between 20 minutes to 3 hours, depending on the dimensionality (slower for large $d$).

\section{Why DFT Regions Become Very Small}
\label{app:small}
As we saw in~\cref{fig:eargmax}, while we can provide guarantees that all k-active label assignments are argmaxable, they may not be epsilon argmaxable (i.e. argmaxable with a large margin).
This is because, due to numerical issues, the regions may be too small to detect via our Linear Programme of~\cref{sec:verify}.

Herein we provide some thoughts on why the DFT regions become small when $d \ll n$.
For a matrix $\vv{W}$ to be in $\pgrass{n}{d}$, all its maximal minors have to be non-zero and have the same sign. Let us assume they are all positive. For the matrix to ``robustly'' have this structure, the maximal minors should be large; the larger they are the more the row vectors will have to change before some become collinear and make one of the maximal minors zero. We note that the maximal minors are affected by both the angles between the vectors and the norm of the vectors. As such, one consideration is whether some vectors are more sensitive to perturbations due to their norm, e.g. row vectors having a small norm would generally be more sensitive to perturbations than row vectors that have a large norm. However, for the DFT matrix, we can ignore the effect of norms as we shall see below.

\paragraph{a) Row vectors of the DFT matrix have equal norm.}
We can show (see~\cref{sec:normder}) that all row vectors $\vv{w}^{(i)}$ of $\vv{W} \in \R^{n \times (2k+1)},\, k \in \mathbb{N}, k \geq 1$ have norm given by:
\begin{equation}
\snorm{\vv{w}^{(i)}}=\sqrt{\frac{2k+1}{n}}
\end{equation}

\paragraph{b) Maximal minor constraints for orthonormal matrices.}
In addition, the truncated DFT matrix is also an orthonormal matrix (its columns are pairwise orthogonal and have norm 1). An orthonormal matrix $\vv{M} \in \R^{n \times d},\, d < n$ with maximal minors $\Delta_{I}$ indexed by $d$-subsets of rows, has the property that the sum of squares of its maximal minors is 1:
\begin{equation}
\sum_{I \in \binom{[n]}{d}} \left(\Delta_I(\vv{M})\right)^2 = 1
\end{equation}
The above follows from the Cauchy-Binet formula (see~\cref{sec:detder}).
Therefore, we have a bound of $1$ on the sum of squares of maximal minors. Since there are $\binom{n}{d}$ maximal minors, a lot of them will have to become very small in magnitude as we increase $n$ while keeping $d$ fixed, i.e. $d << n$.  As we saw in a), the magnitude of all row vectors is equal. Therefore, the only way to have small maximal minors is to have small angles between the vectors, which in turn forces many regions to be small (i.e. narrow wedges).

\subsection{Derivation of a)}
\label{sec:normder}

Recall a row $\vv{w} \in \R^{2k+1},\, k \in \mathbb{N},\, k \geq 1$ of the DFT matrix:

\resizebox{.98\linewidth}{!}{
\begin{minipage}[b]{1.1\linewidth}
\centering
\begin{align}
\vv{w} =&
\setlength\arraycolsep{2pt}
\begin{bmatrix}
\frac{1}{\sqrt{n}} & \sqrt{\frac{2}{n}}\cos{t} & \sqrt{\frac{2}{n}}\sin{t} & \cdots & \sqrt{\frac{2}{n}}\cos{k t}  & \sqrt{\frac{2}{n}}\sin{k t}
\label{eq:dftrow}
\end{bmatrix}
\end{align}
\end{minipage}
}
Using the form \cref{eq:dftrow}, we compute the Euclidean norm:
\resizebox{.98\linewidth}{!}{
\begin{minipage}[b]{1.1\linewidth}
\centering
\begin{align}
\snorm{\vv{w}^{(i)}}&
= \sqrt{ \sum_{j=1}^{2k+1} \left(\vv{w}^{(i)}_j\right)^2 }\\
&= \sqrt{  \left(\vv{w}^{(i)}_1\right)^2 + \sum_{j=1}^{2k} \left(\vv{w}^{(i)}_{j+1}\right)^2}\\
&= \sqrt{ \left( \frac{1}{\sqrt{n}} \right)^2 + \sum_{k'=1}^{k} \left( \sqrt{\frac{2}{n}}\right)^2\left(\left(\sin k't\right)^2 + \left(\cos k't\right) ^2\right)} \\
&= \sqrt{ \frac{1}{n} + \frac{2}{n}\sum_{k'=1}^{k} 1} \\
&= \sqrt{\frac{2k+1}{n}}
\end{align}
\end{minipage}
}

\subsection{Derivation of b)}
\label{sec:detder}
The Cauchy-Binet formula~\citep[page 2]{Pinkus2009} expresses the determinant of a product of two rectangular matrices $\det\left({\vv{AB}}\right)$ with $\vv{A} \in \R^{d \times n},\, \vv{B} \in \R^{n \times d}$ in terms of a sum of maximal minors $\Delta$ of $\vv{A}$ and $\vv{B}$:
\begin{equation}
\sum_{I \in \binom{[n]}{d}}  \Delta_I(\vv{A}) \Delta_I(\vv{B}) = \det\left({\vv{AB}}\right)
\end{equation}

Note that for an orthonormal matrix $\vv{M} \in \R^{n \times d}$, if we set $\vv{A}=\vv{M}^\top$ and $\vv{B}=\vv{M}$, we get:

\begin{equation}
\begin{split}
\sum_{I \in \binom{[n]}{d}}  \Delta_I(\vv{M}^\top) \Delta_I(\vv{M}) &= \det\left({\vv{M}^\top\vv{M}}\right)
\implies \\
\sum_{I \in \binom{[n]}{d}}  \left(\Delta_I(\vv{M})\right)^2 &=\det\left(\vv{I}\right) = 1
\end{split}
\end{equation}

where the bolded $\vv{I}$ is the identity matrix.

\clearpage
\section{MLC Evaluation Metrics}
\label{app:evalmetrics}

\paragraph{F1}
We compute \textbf{Micro F1} by computing Precision and Recall across all labels and then computing F1.
We compute \textbf{Macro F1} by computing Precision, Recall and F1 score for each label individually and then averaging them. Macro F1 does not allow label imbalance to skew the results.

\paragraph{Precision@$k$ (Prec@$k$)}
Prec@$k$ computes the percentage of the $k$ retrieved labels that are indeed correct. We compute the metric by ranking the labels by their assigned probabilities and take the top $k$ as active. We then compute the percentage that is active in the gold data, i.e. we divide the number of correct active labels by $k$. Prec@k is not sensitive to the relative ordering of the labels within the top-k.

\paragraph{Recall@$k$ (Rec@$k$)}
Rec@$k$ computes how many of the actual active labels are actually retrieved in the top $k$. As in Prec@k, we rank the labels by their assigned probabilities and take the top $k$ as active, but this time we divide by the number of labels that are actually active.

\paragraph{F1@$k$ (F1@$k$)}
As is common with metrics, it is useful to distil as much information as possible into a single number. To achieve this, it is common to use the harmonic mean of Prec@k and Rec@k which captures the intuition that we want both Prec@k and Rec@k to be high:
\begin{equation}
\text{F1@k} = 2\frac{\text{Prec@k} \,\text{Rec@k}}{\text{Prec@k} + \text{Rec@k}}
\end{equation}

\paragraph{Normalised Discounted Cumulative Gain @$k$ (nDCG@$k$)}
nDCG~\citep{Jrvelin2002} is sensitive to the relative ranking of examples within the $k$-top subset, due to the use of discounting. To compute it, we rank the labels according to their assigned probabilities, take the top $k$ and sum their truth values, which for MLC is $1$ if the label is active and $0$ otherwise. However, as opposed to Prec@k, DCG@k adds a logarithmic discount factor such that ranking an irrelevant label above a relevant one is penalised. We use nDCG, which includes normalisation such that the score is in $[0, 1]$, with $1$ being optimal.

\subsection{More Results for Above Metrics}
\begin{figure}[h!]
\begin{subfigure}{.32\columnwidth}
\centering
\includegraphics[width=\linewidth]{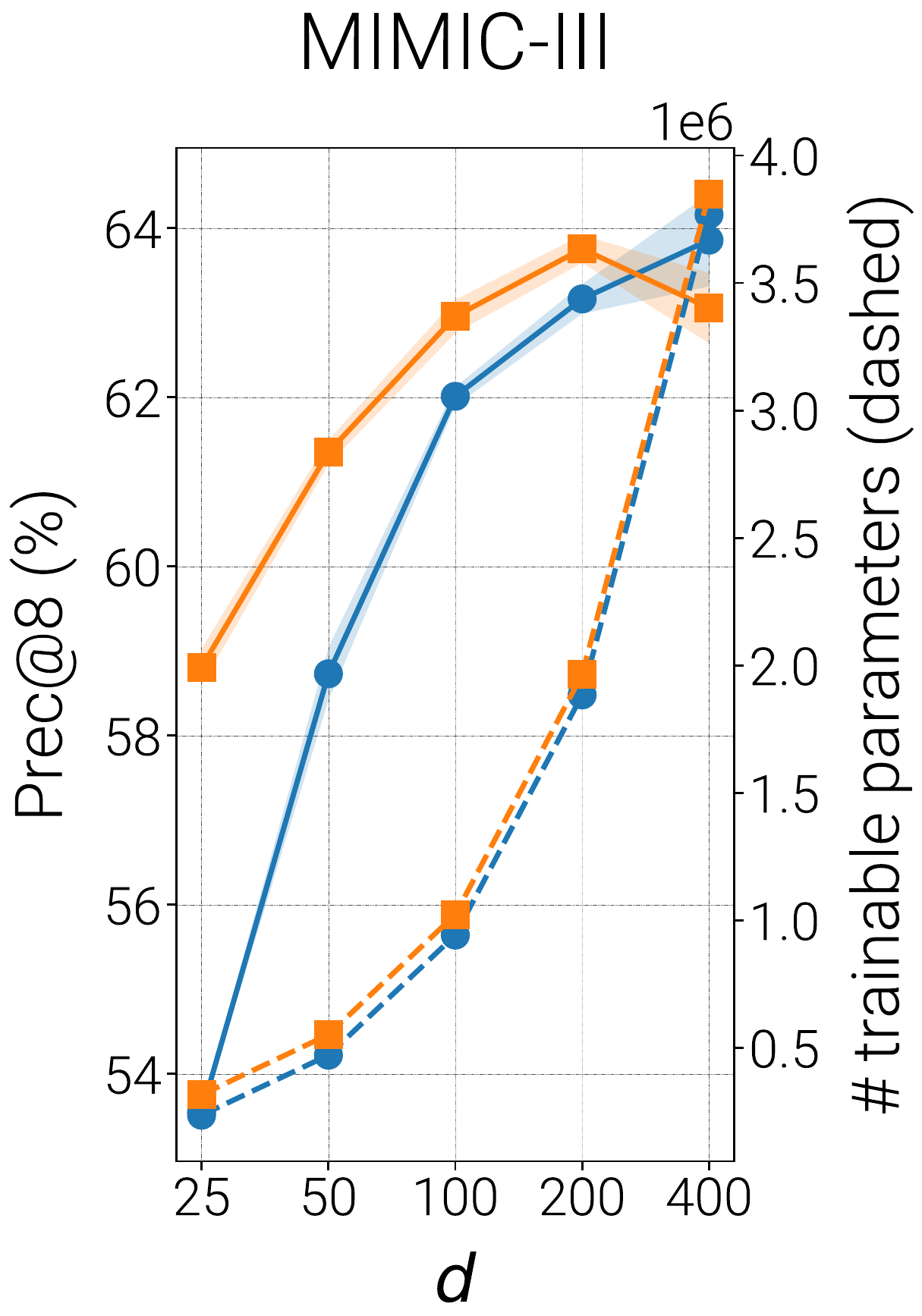}
\end{subfigure}
\hfill
\begin{subfigure}{.32\columnwidth}
\centering
\includegraphics[width=\linewidth]{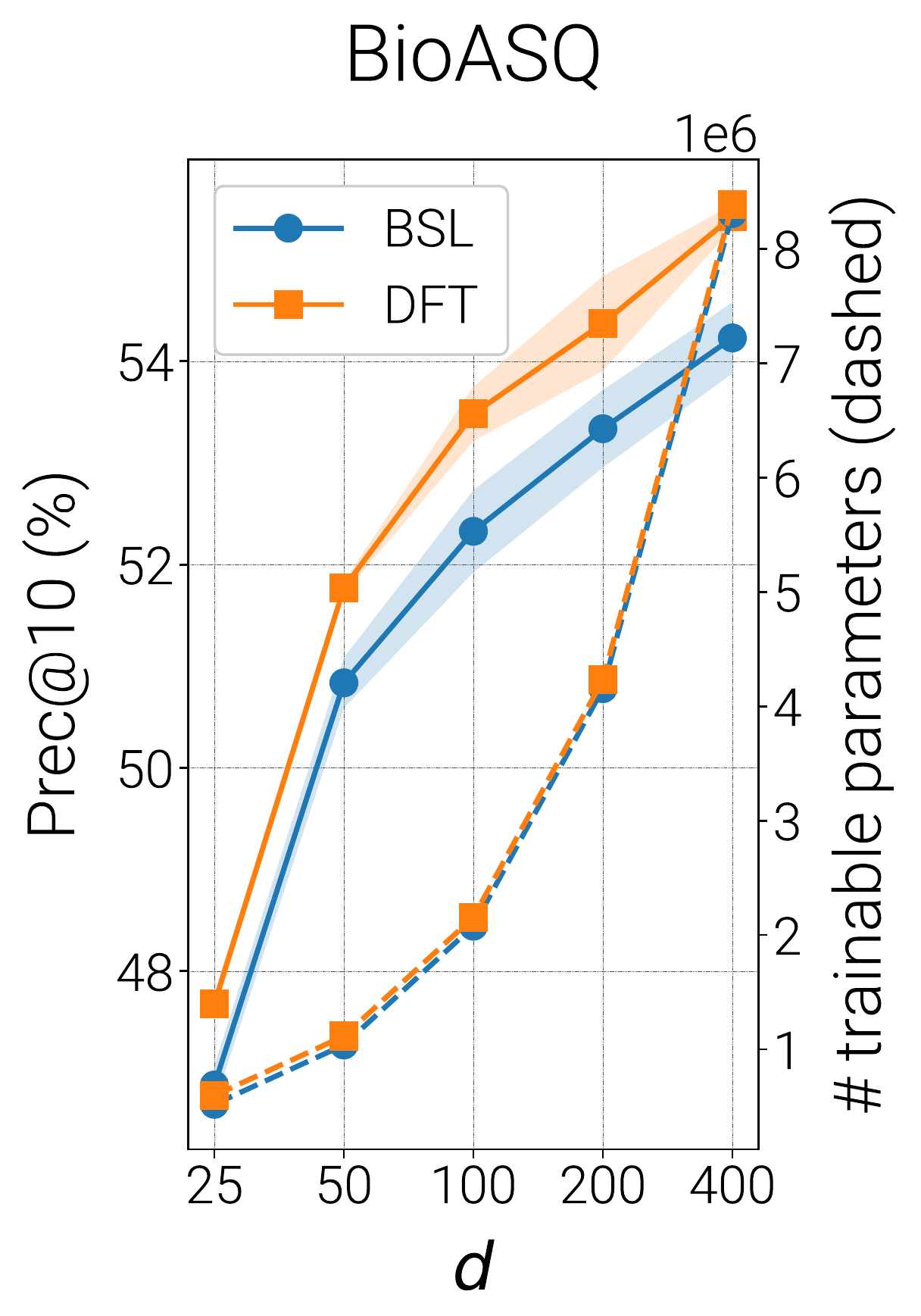}
\end{subfigure}
\hfill
\begin{subfigure}{.32\columnwidth}
\centering
\includegraphics[width=\linewidth]{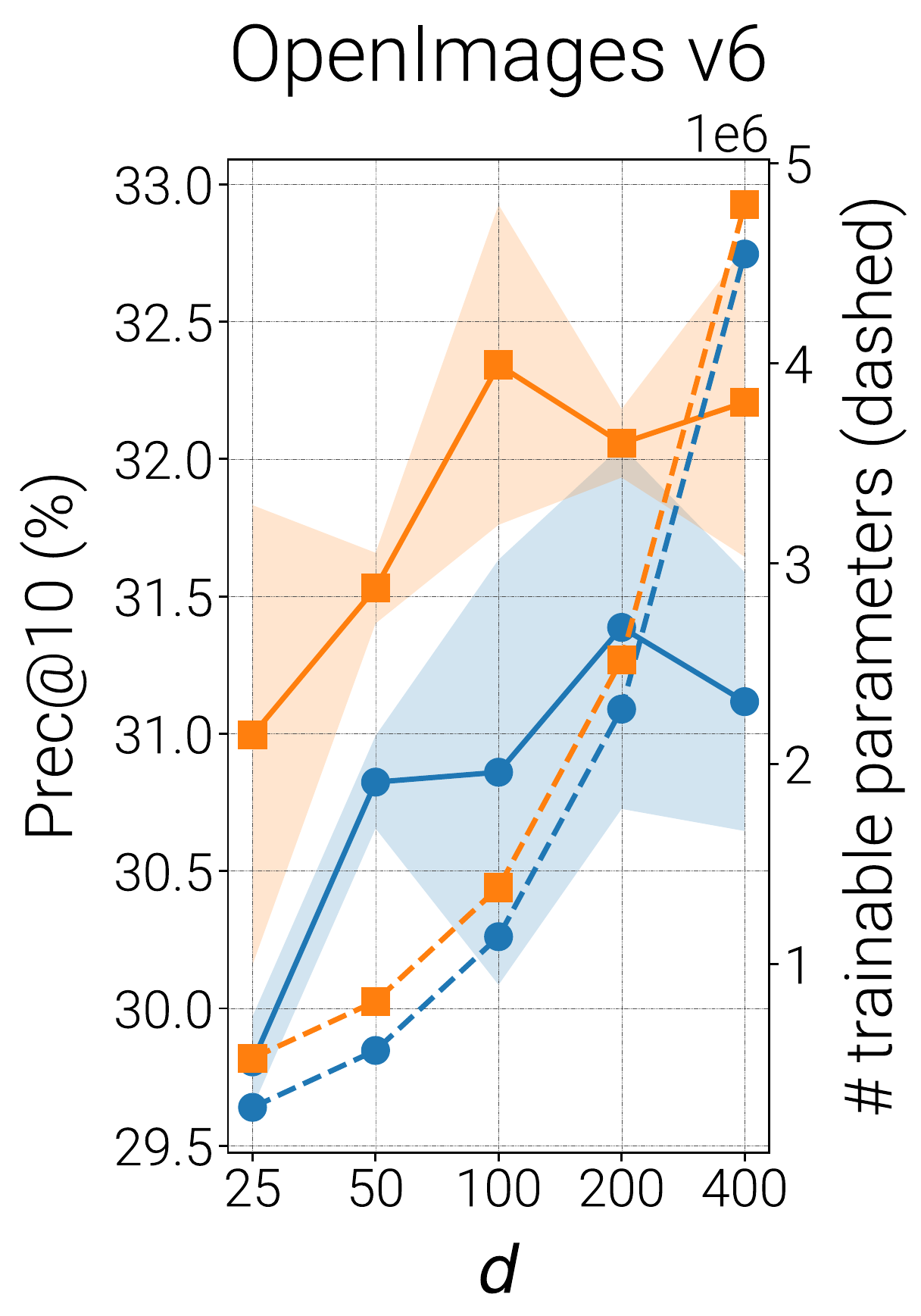}
\end{subfigure}
\caption{Test set Precision@k across datasets. For MIMIC-III, \citet{mullenbach2018} report 58.1 Prec@8 for their CNN baseline. We showed that their result can be improved by: a) Making the learning rate 0.001 b) adding a projection layer after the CNN and c) using the DFT layer.}
\end{figure}
\begin{figure}[h]
\begin{subfigure}{.32\columnwidth}
\centering
\includegraphics[width=\linewidth]{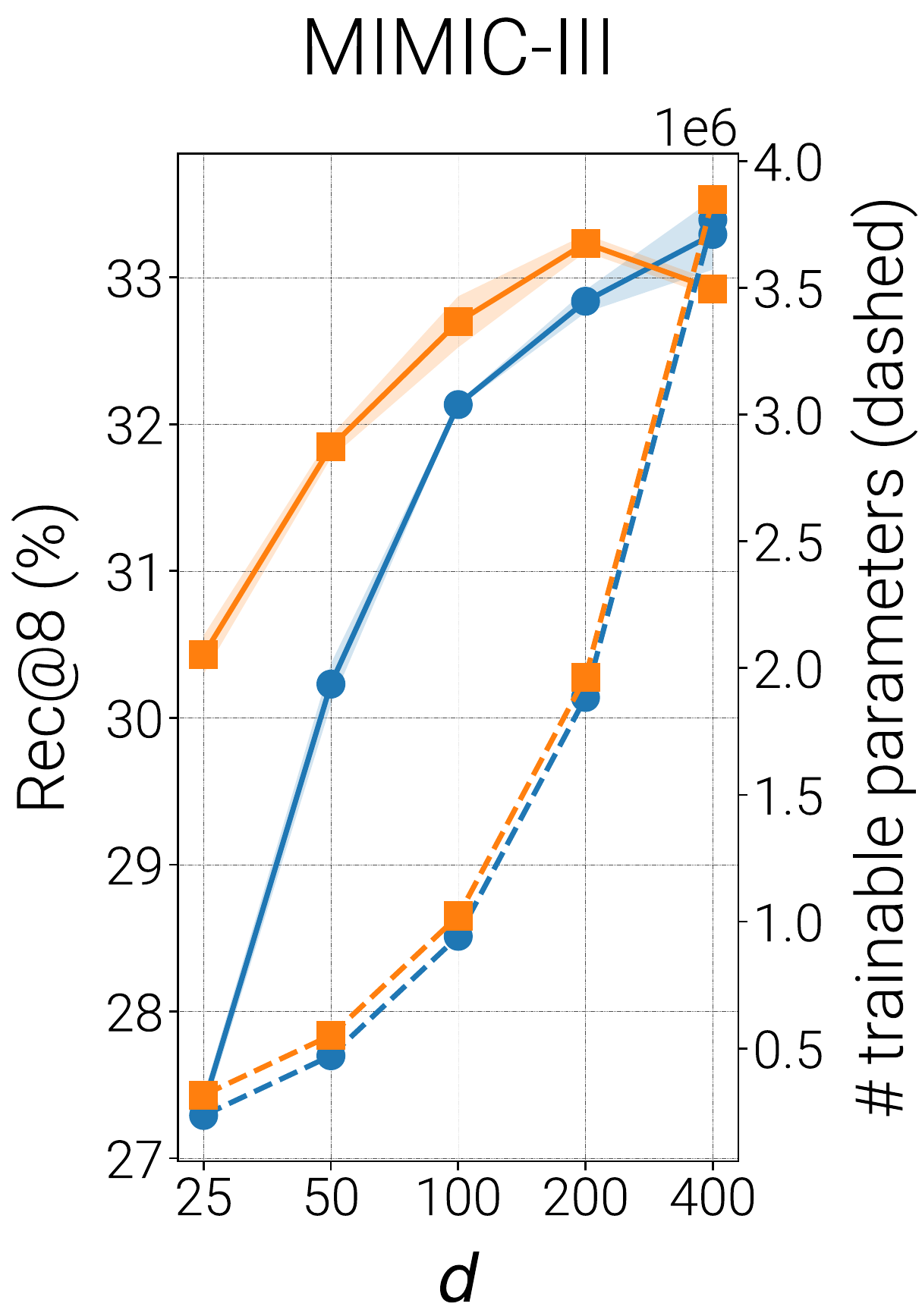}
\end{subfigure}
\hfill
\begin{subfigure}{.32\columnwidth}
\centering
\includegraphics[width=\linewidth]{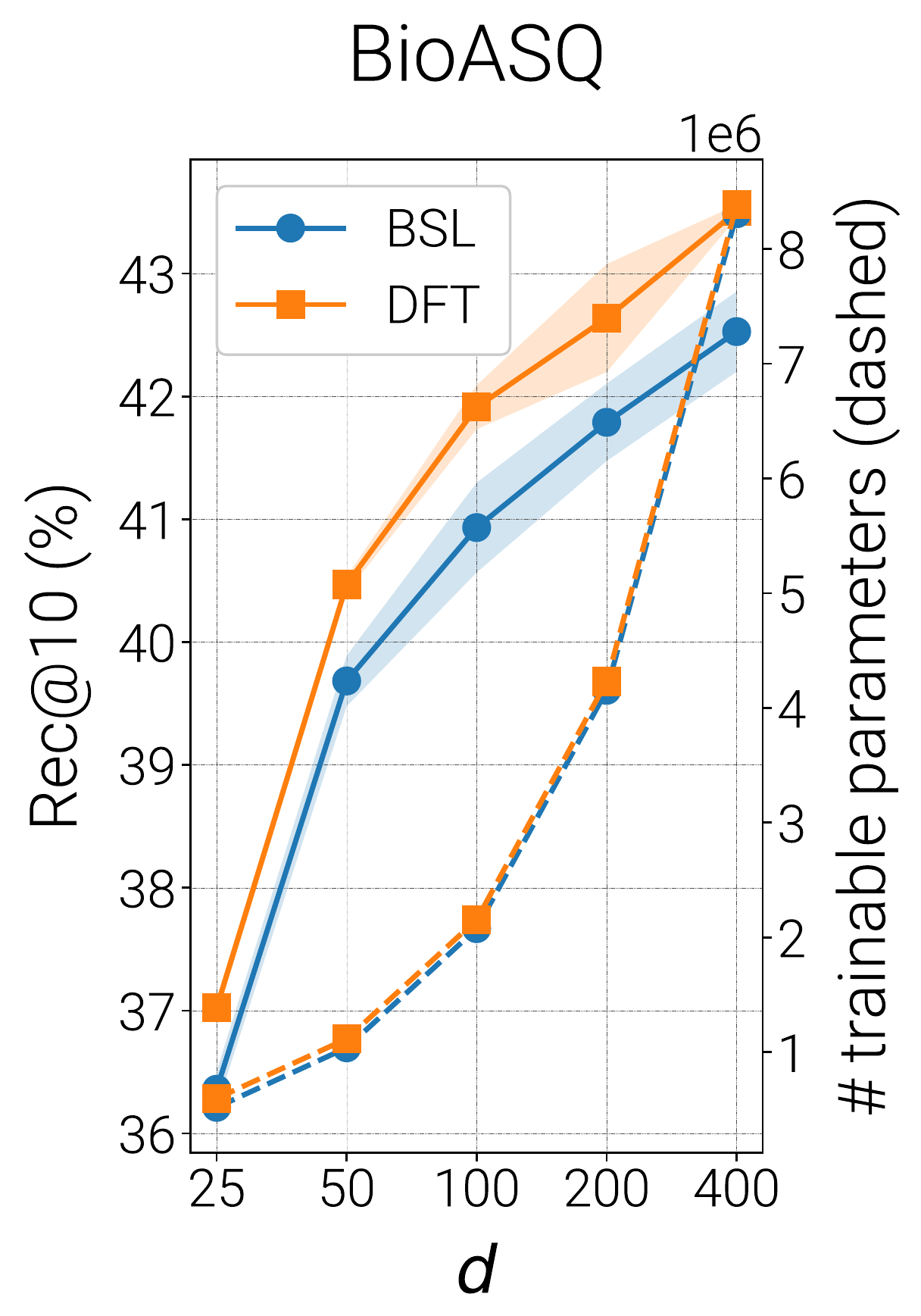}
\end{subfigure}
\hfill
\begin{subfigure}{.32\columnwidth}
\centering
\includegraphics[width=\linewidth]{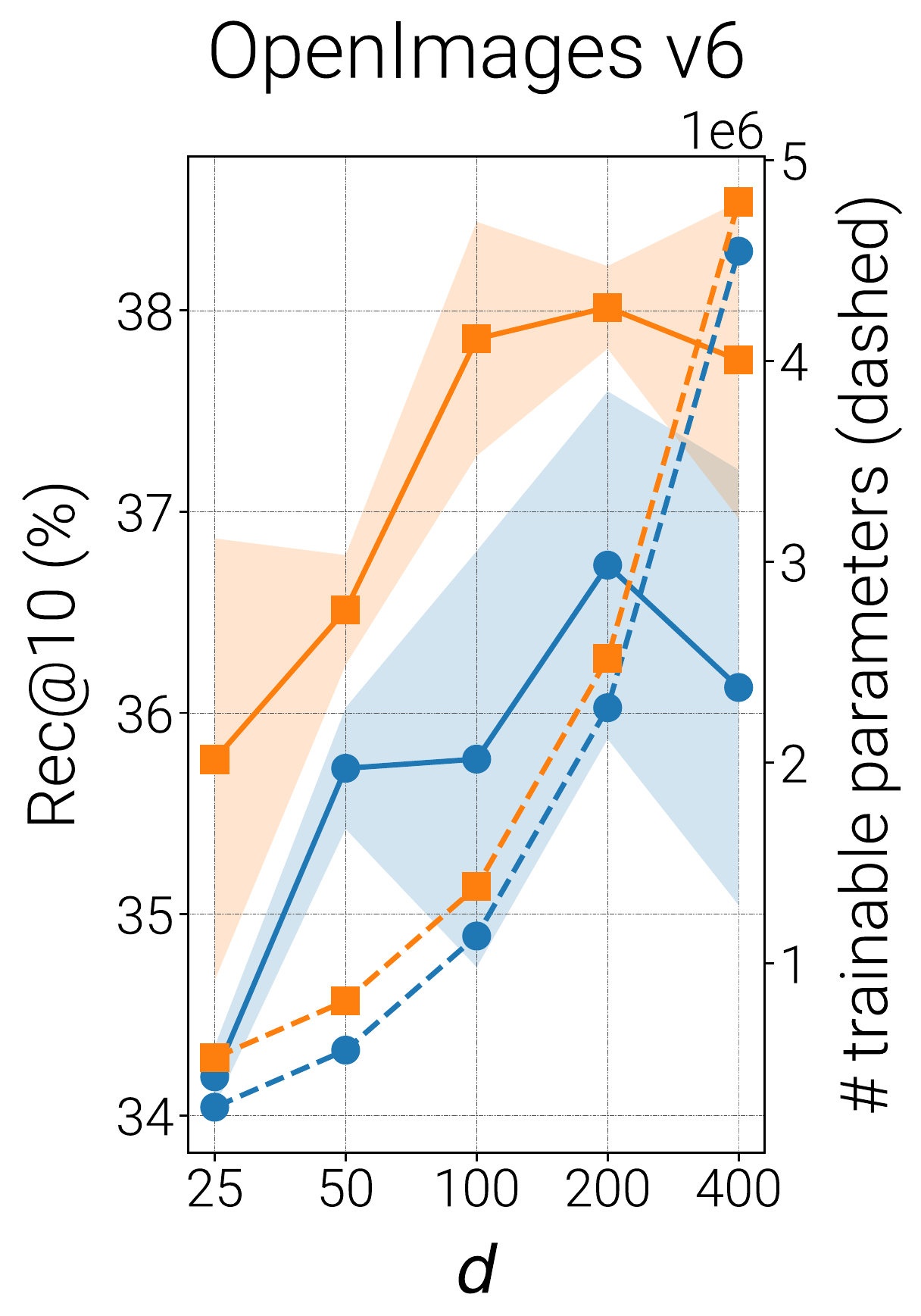}
\end{subfigure}
\caption{Test set Recall@k across datasets.}
\end{figure}

\begin{figure}[h!]
\begin{subfigure}{.32\columnwidth}
\centering
\includegraphics[width=\linewidth]{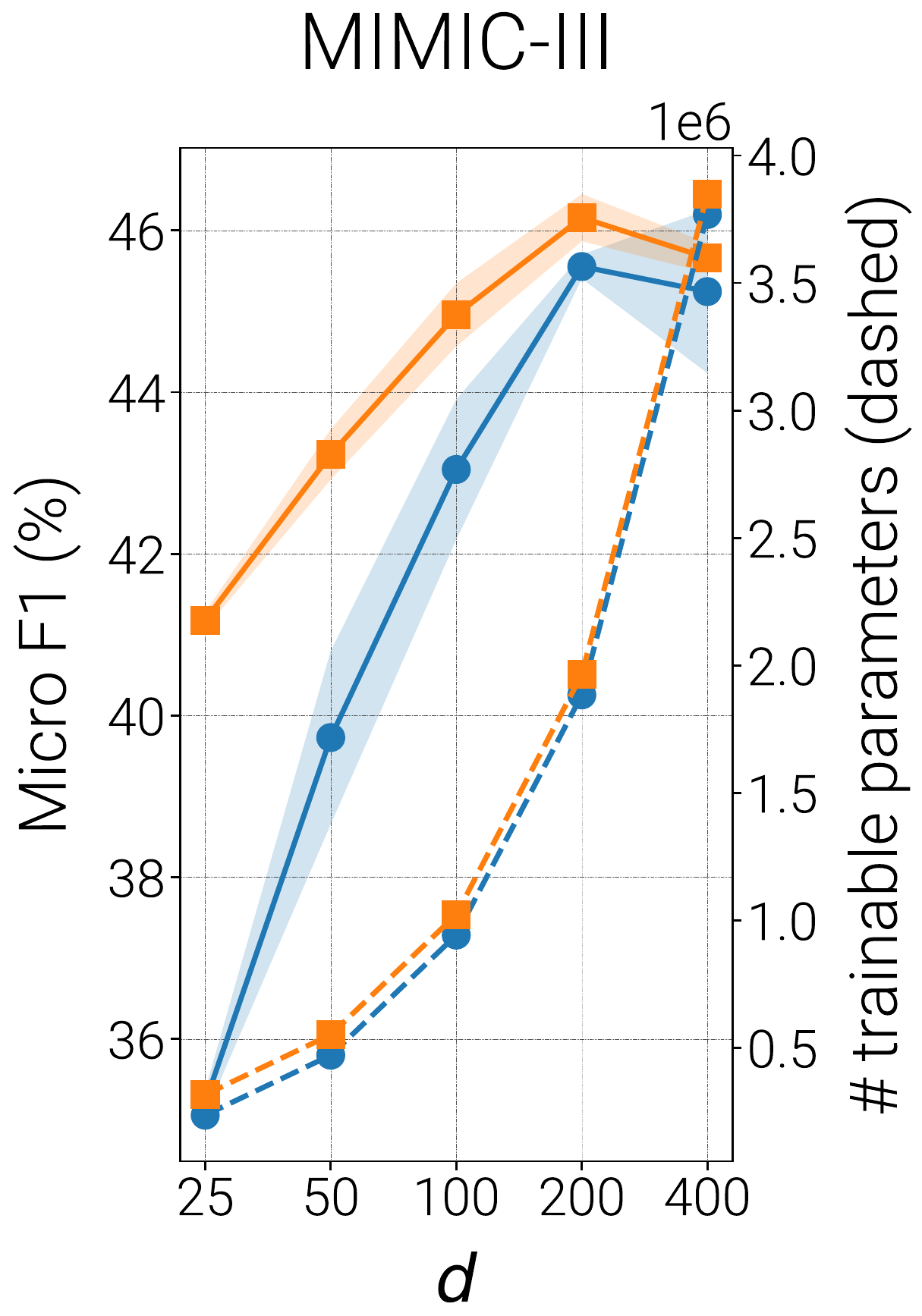}
\end{subfigure}
\hfill
\begin{subfigure}{.32\columnwidth}
\centering
\includegraphics[width=\linewidth]{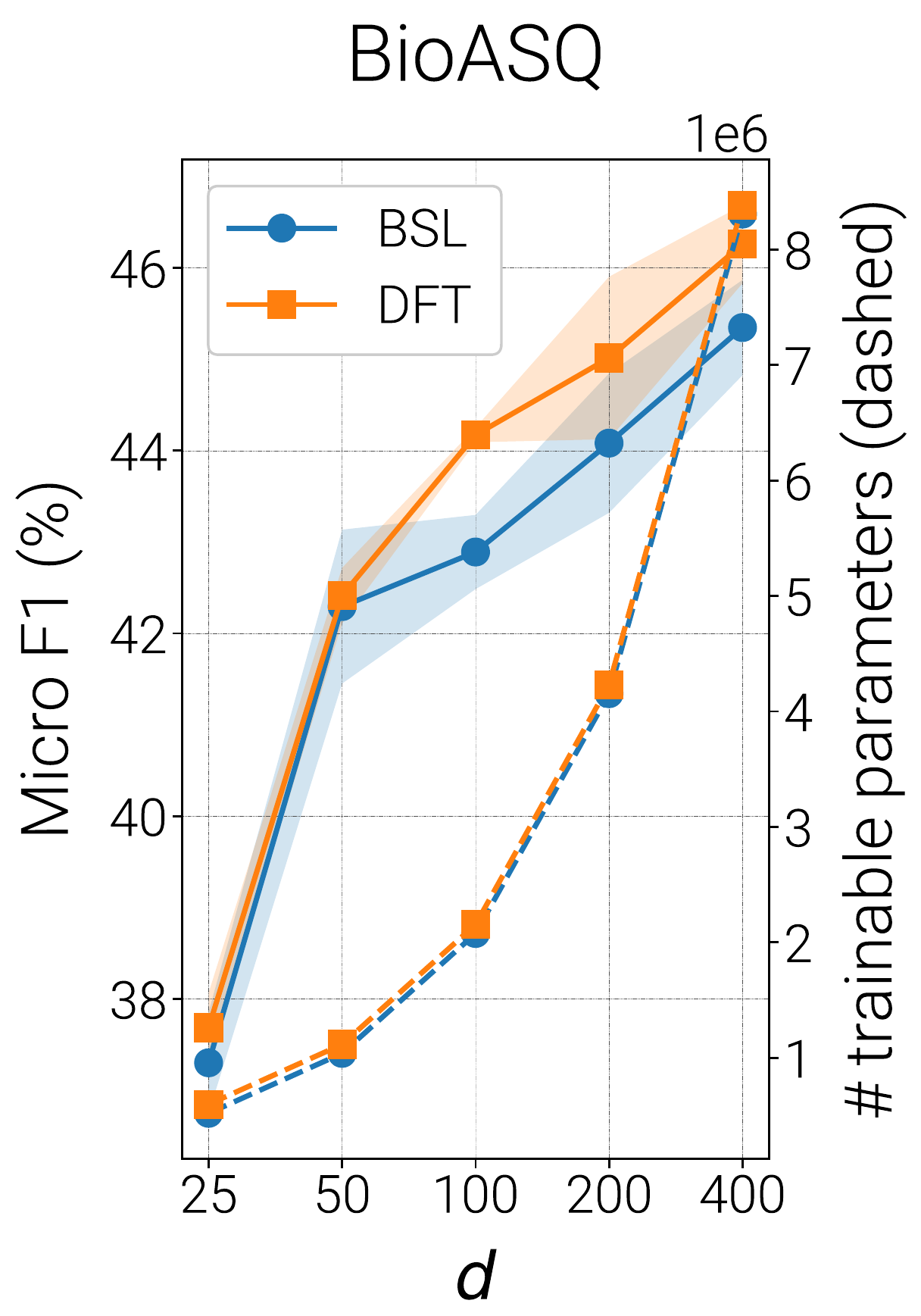}
\end{subfigure}
\hfill
\begin{subfigure}{.32\columnwidth}
\centering
\includegraphics[width=\linewidth]{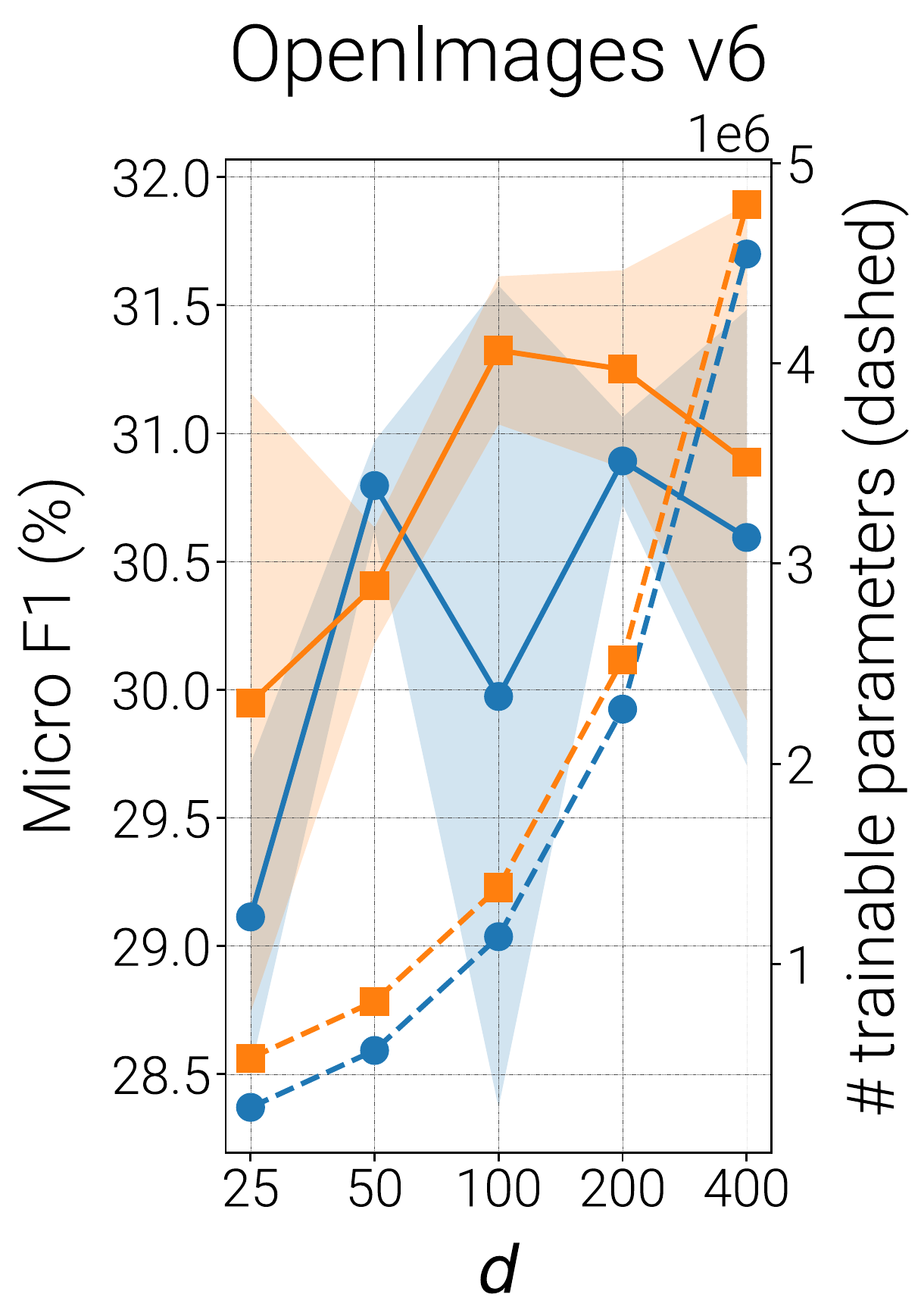}
\end{subfigure}
\caption{Test set Micro F1 across datasets.}
\end{figure}

\begin{figure}[h!]
\begin{subfigure}{.32\columnwidth}
\centering
\includegraphics[width=\linewidth]{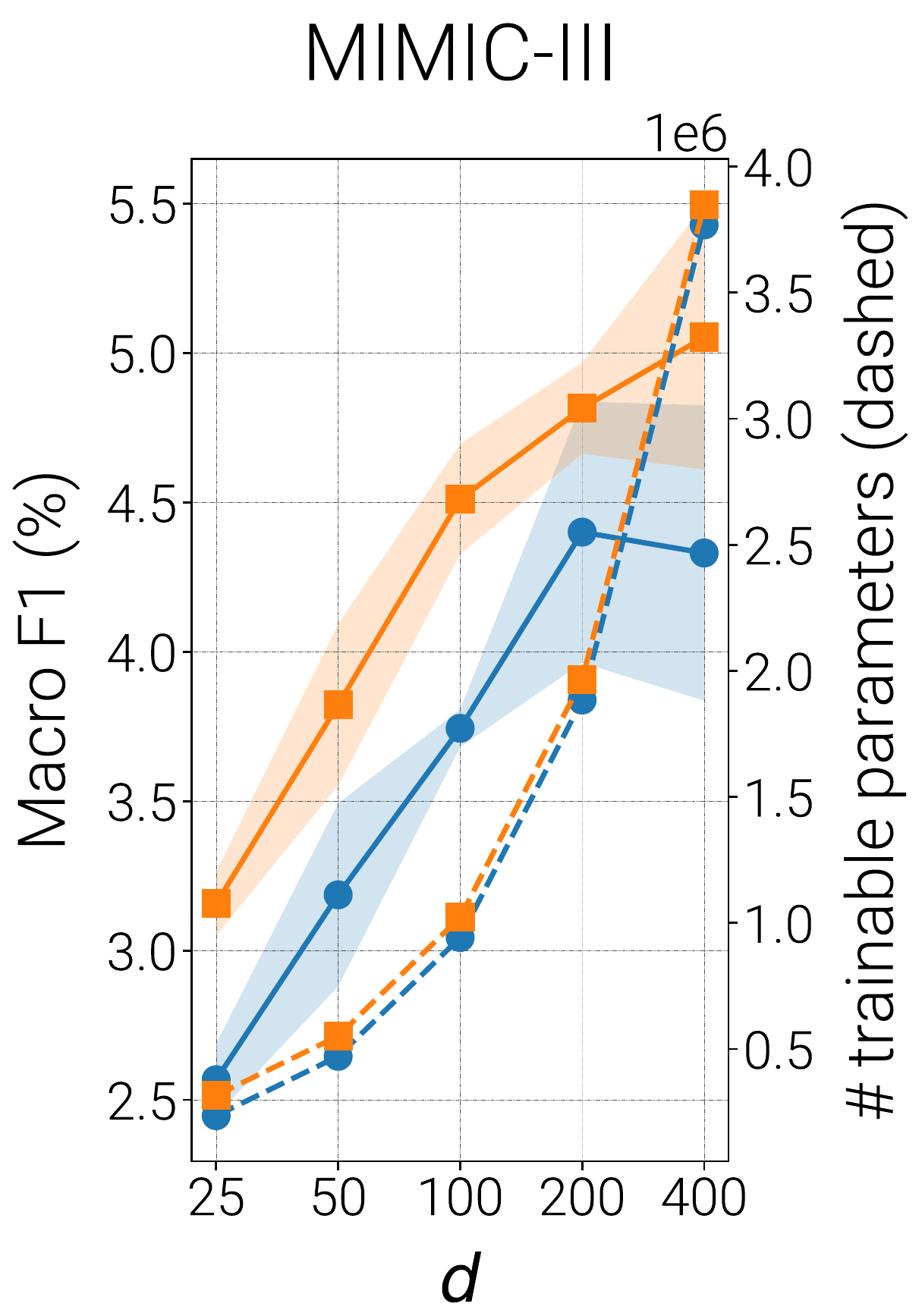}
\end{subfigure}
\hfill
\begin{subfigure}{.32\columnwidth}
\centering
\includegraphics[width=\linewidth]{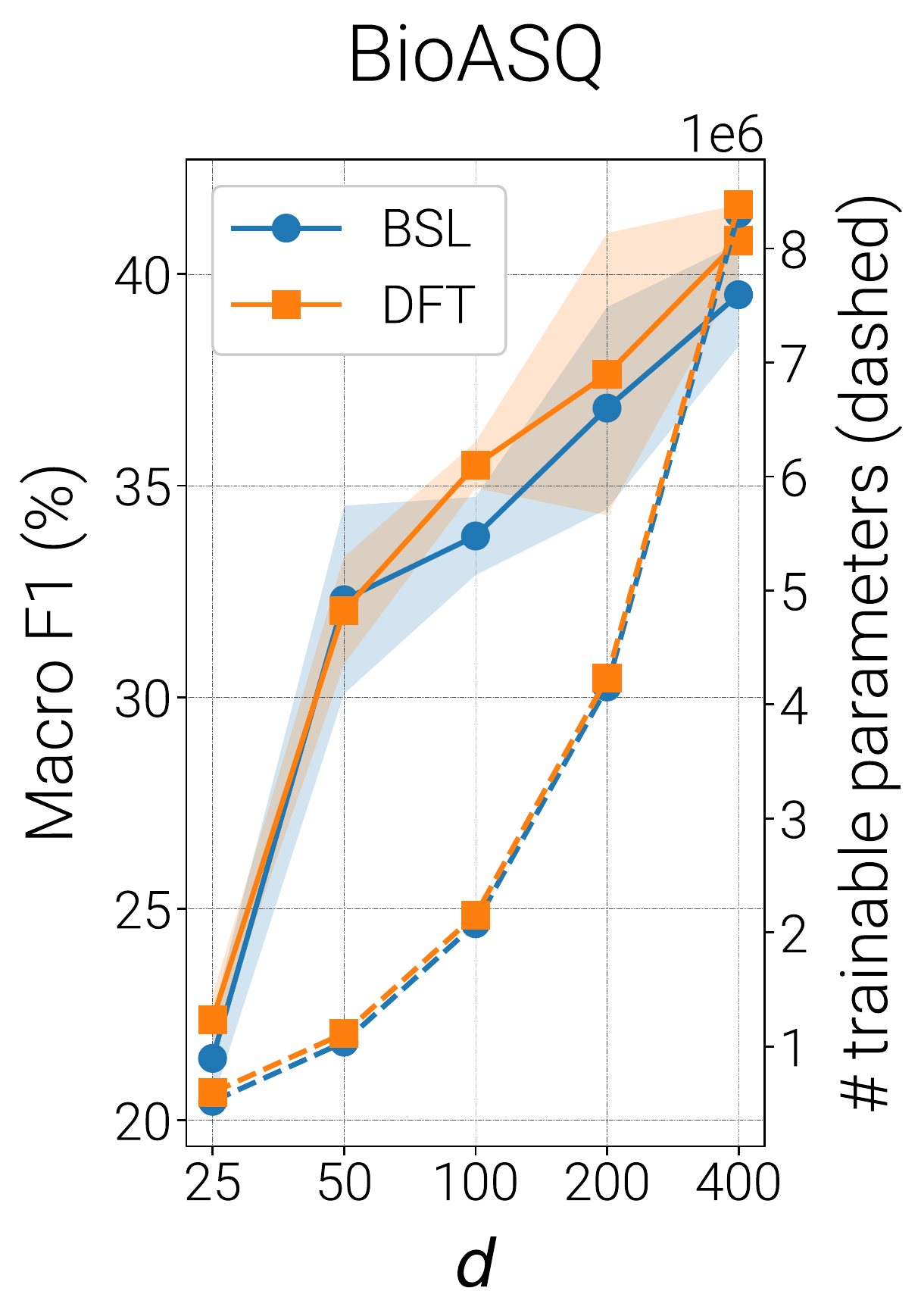}
\end{subfigure}
\hfill
\begin{subfigure}{.32\columnwidth}
\centering
\includegraphics[width=\linewidth]{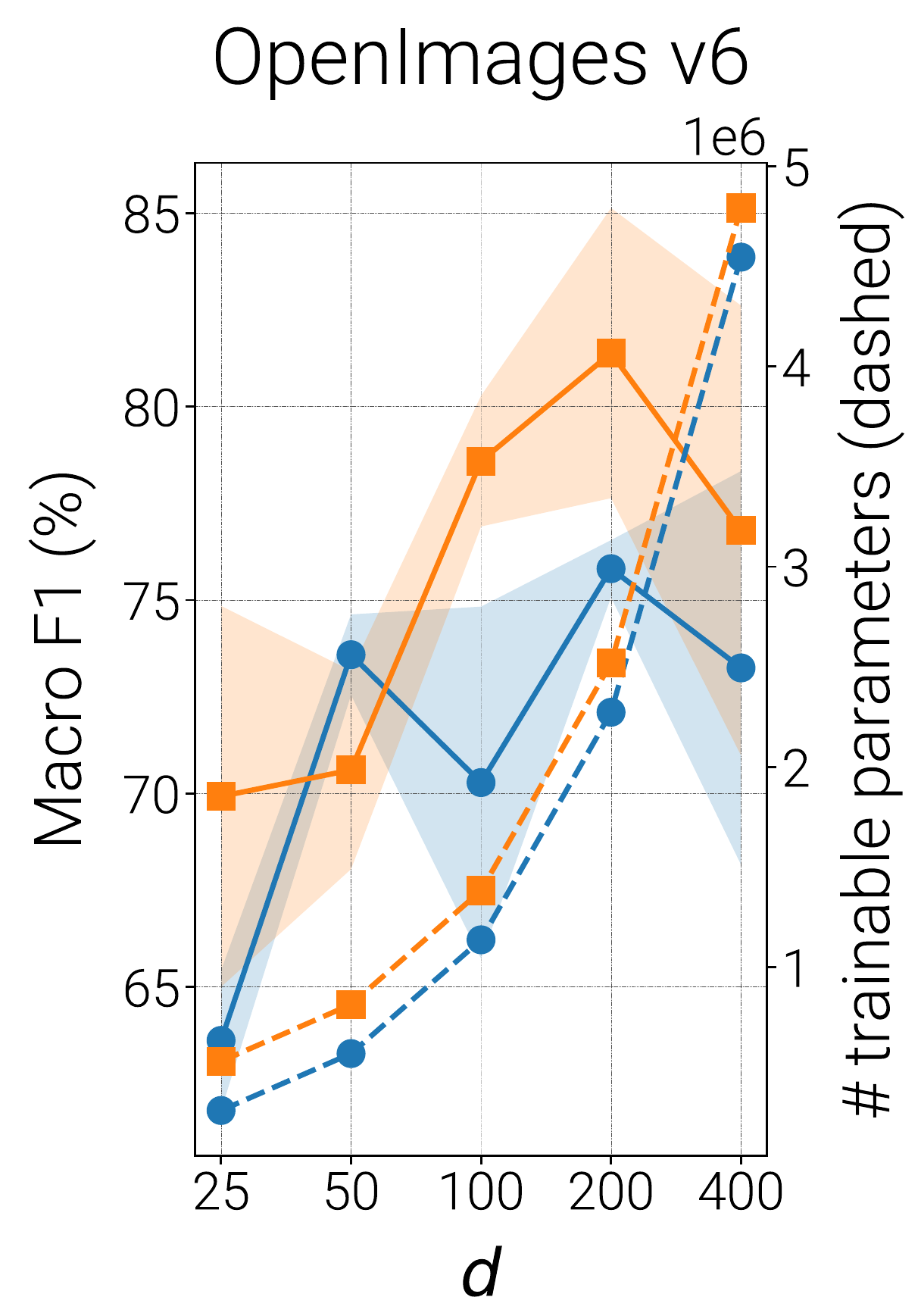}
\end{subfigure}
\caption{Test set Macro F1 across datasets.}
\end{figure}

\begin{figure}[h!]
\hfill
\begin{subfigure}{.33\columnwidth}
\includegraphics[width=\linewidth]{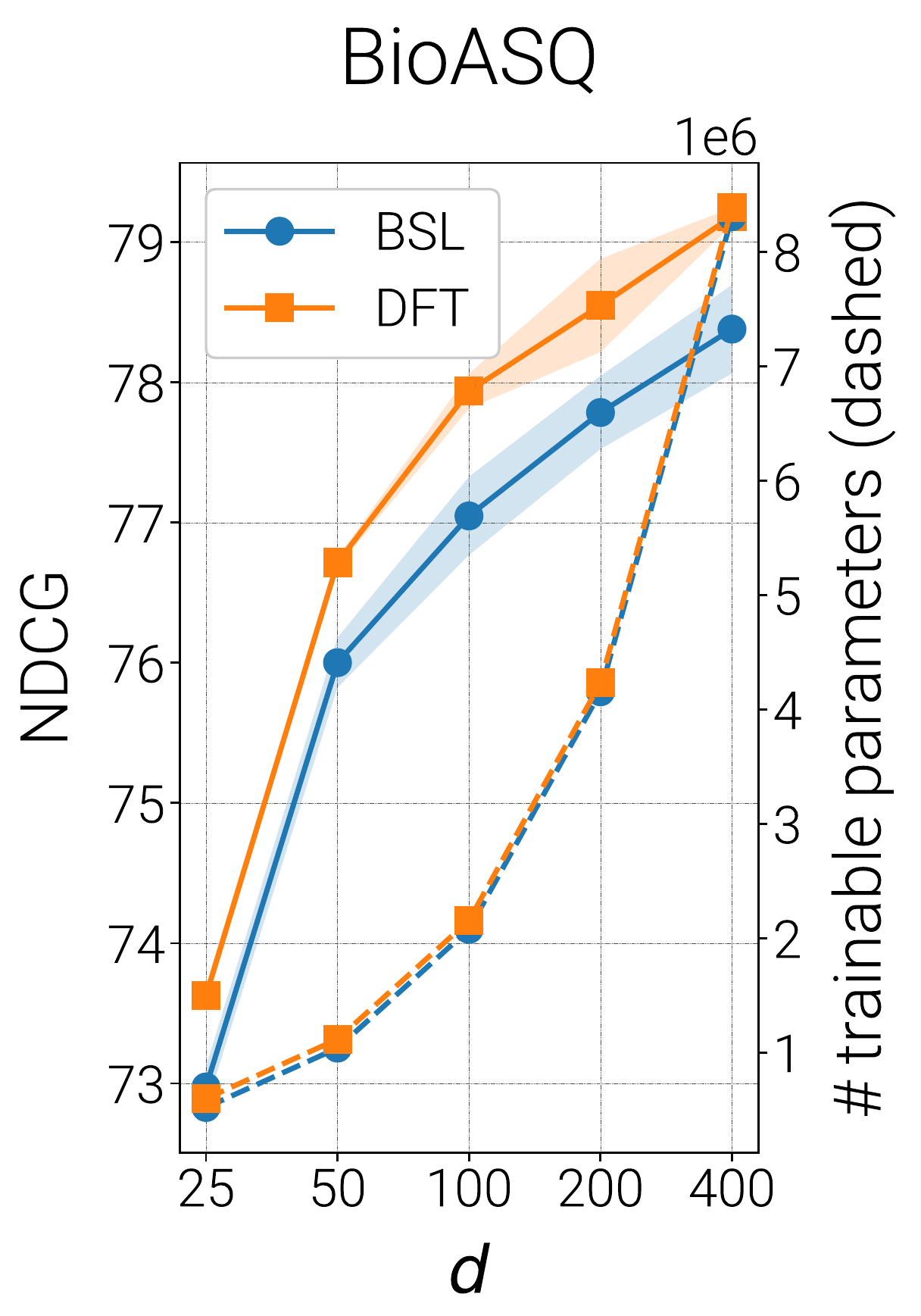}
\end{subfigure}
\begin{subfigure}{.33\columnwidth}
\includegraphics[width=\linewidth]{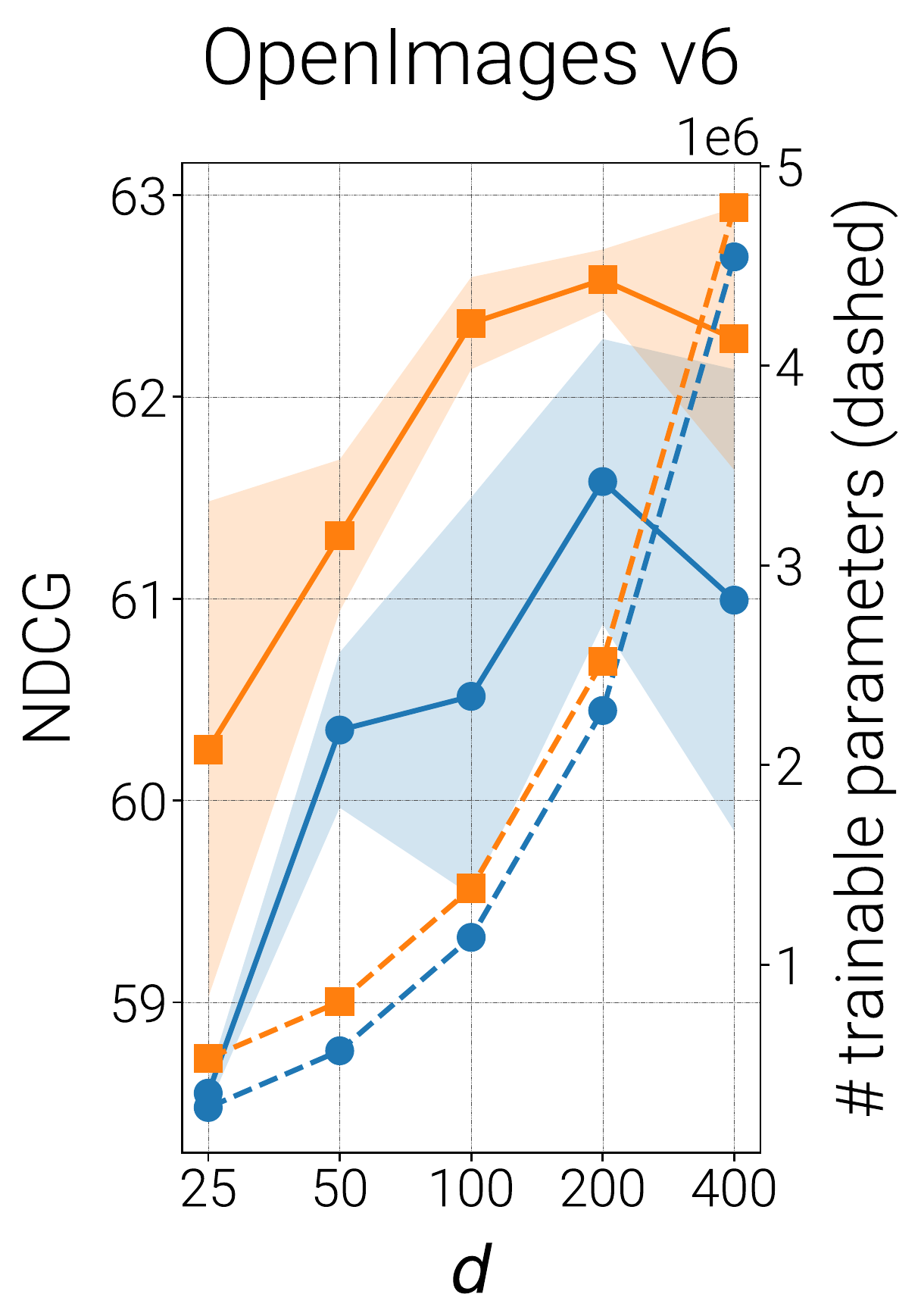}
\end{subfigure}
\caption{Test set nDCG across datasets.}
\end{figure}

\FloatBarrier

\section{Derivation of Linear Programme}
\label{app:derivlp}
In order to check whether a label assignment $\vv{y}$ is argmaxable, we need to check whether there exists an input $\vv{x}$ which can be assigned such a $\vv{y}$. For this to be possible, there must be an input $\vv{x}$ for which the dot product with the corresponding binary classifier $\vv{w}^{(i)}$ agrees in sign with $\vv{y}_i$;  i.e. ${\vv{w}^{(i)}}^\T \vv{x} > 0$ for $\vv{y}_i = +$ and ${\vv{w}^{(i)}}^\T \vv{x} < 0$ for $\vv{y}_i = -$. From this perspective, each $\vv{w}^{(i)}$ and the corresponding sign of the label $\vv{y}_i$ define a halfspace and we
are checking if the intersection of halfspaces exists (argmaxable $\vv{y}$) or not (unargmaxable $\vv{y}$).

\subsection{Halfspace Constraints}
Below we derive the constraints we want to encode for the Linear Programme. More precisely, if the intersection of halfspaces exists we also want to find the largest margin $\epsilon \snorm{\vv{w}^{(i)}}$ for which this is true (Chebyshev LP). In the case $\vv{y}_i = +$ we want the dot product to be positive even if we subtract the margin $\epsilon \snorm{\vv{w}^{(i)}}$. With the same motivation for $\vv{y}_i = -$, we get the constraints $LP(\vv{y}_i)$:

\begin{equation}
LP(\vv{y}_i) =
\begin{cases}
{\vv{w}^{(i)}}^\T \vv{x} - \epsilon \snorm{\vv{w}^{(i)}} \geq 0 & \text{for }\vv{y}_i = +  \\
{\vv{w}^{(i)}}^\T \vv{x} + \epsilon \snorm{\vv{w}^{(i)}} \leq 0 & \text{for }\vv{y}_i = - 
\end{cases}
\end{equation}

We can rewrite the $\vv{y}_i = +$ case by multiplying by $-1$:

\begin{align}
{\vv{w}^{(i)}}^\T \vv{x} - \epsilon \snorm{\vv{w}^{(i)}} \geq 0& \implies \\
-{\vv{w}^{(i)}}^\T \vv{x} + \epsilon \snorm{\vv{w}^{(i)}} \leq 0 
\end{align}

and succinctly combine both cases:
\begin{equation}
-\vv{y}_i {\vv{w}^{(i)}}^\T \vv{x} + \epsilon \snorm{\vv{w}^{(i)}} \leq 0 
\end{equation}

where we abuse notation and assume $\vv{y}_i$ takes values $+1$ and $-1$ correspondingly.

\subsection{Box Constraints}
In order for the Chebyshev center to be defined, we need to bound the magnitude of each dimension of $\vv{x}$. As such, we assume the activations $\vv{x}$ are independently bounded to have magnitude less than $10^4$, i.e. we have box constraints:

\begin{equation}
-10^4 \leq \vv{x}_{j} \leq 10^4, \quad 1 \leq j \leq d
\end{equation}

\subsection{LP Sensitivity}

In theory, the constraint for the margin of the Chebyshev LP is that $\epsilon$ must be positive. However, Gurobi has a sensitivity limit of $10^{-9}$, so we set $\text{eps}=10^{-8}$ and obtain:

\begin{equation}
\epsilon > \text{eps}
\end{equation}

\subsection{Summary}
We combine all the above to get the optimisation problem:

\begin{alignat}{2}
& \text{maximise} & \quad &  \epsilon  \\
& \text{subject to} &  & -\vv{y}_i {\vv{w}^{(i)}}^\T \vv{x} + \epsilon \Vert \vv{w}^{(i)} \Vert_2 \leq 0, \quad  1 \leq i \leq n, \nonumber \\
& & & -10^4 \leq \vv{x}_{j} \leq 10^4, \quad 1 \leq j \leq d, \quad
 \epsilon > \text{eps} \nonumber
\end{alignat}

\section{Sign Rank and Argmaxability}
\label{app:signrank}
In~\cref{sec:relatedwork}, we briefly discussed the sign rank of a matrix but did not have space to go into details. Herein we clarify the connection between the sign rank of a matrix and argmaxability, as we defined it in the paper.

Consider a sign matrix $\vv{S}=\{+1, -1\}^{N \times n}$ where we stack $N$ sign vectors (label assignments) in the rows of $\vv{S}$. The sign rank of $\vv{S}$ is the smallest rank a matrix $\vv{M} \in \R^{N \times n}$ can have such that we can still reconstruct $\vv{S}$ by applying the sign function element-wise to $\vv{M}$.
\begin{equation}
\signrank{\vv{S}} = \min\{\rank{\vv{M}}: \sign{\vv{M}} = \vv{S}\}
\end{equation}
In argmaxability terms, the sign rank, $r=\signrank{\vv{S}}$, is the smallest dimensionality the feature vectors can have such that there exists a linear classifier for which all label assignments in the rows of $\vv{S}$ are argmaxable. To see this, factorise $\vv{M}$ into $\vv{M}=\left(\vv{WX}\right)^\top$, where $\vv{W} \in \R^{n \times d},\, \vv{X} \in \R^{d \times N}$. Interpret $\vv{W}$ as the parametrisation of the BSL and $\vv{X}$ as the feature vectors for each input, stacked as columns of $\vv{X}$. If all entries of $\vv{S}$ are argmaxable, then the sign rank of $\vv{S}$ is at least $d$, and we have a constructive proof of this fact.

An interesting direction for future work is to explore the sign rank literature for such constructive low-rank factorisations that guarantee argmaxability for label assignments of interest. Experiments in~\citet{Neumann2016} and \citet{Chanpuriya2020} show that in practice many datasets have small sign rank, but for MLC we also care about generalisation and robustness ($\epsilon$-argmaxability).

\section{Illustration of Argmaxable Regions}
\label{app:3d}
In the figures on the next page, we provide an illustration of the structure of the label assignments for $\dft{6}{3}$. In~\cref{fig:hasse}, we illustrate the subset of label assignments that are argmaxable.  We represent each label assignment as a sign vector where $+$ denotes an active label and $-$ an inactive one. We color the node green if the corresponding label assignment is argmaxable. We can inspect the consequences of~\cref{thm:card} visually, the $1$-active label assignments are argmaxable. In~\cref{fig:balls}, we illustrate the geometric realisation of $\dft{6}{3}$, where each row of the matrix is a normal vector to a hyperplane. To simplify the plot, instead of plotting the hyperplanes, we plot the Chebyshev balls that correspond to each region. Lastly, in~\cref{fig:hypes} we show the hyperplanes to make it clearer how the balls in the regions are constructed. For an interactive visualisation see {\footnotesize \url{https://grv.unargmaxable.ai/static/files/alternating/index.html}}.
\begin{figure*}[!t]
    \centering
    \begin{subfigure}[t]{0.95\columnwidth}
    \centering
    \includegraphics[width=.95\linewidth]{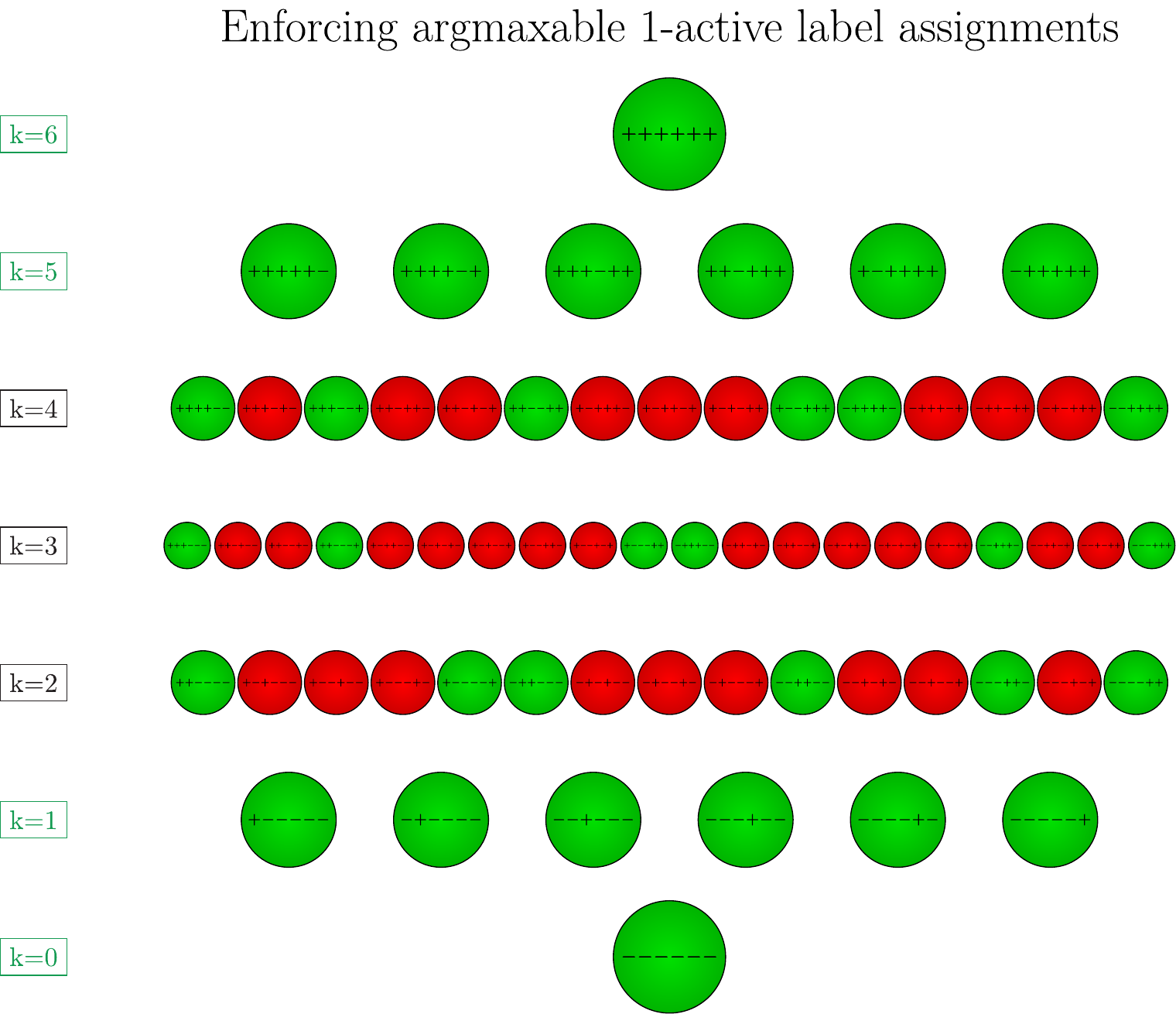}
    \subcaption{Visual check that for $\vv{W}=\dft{6}{3}$ all $2$-alternating, and hence all $1$-active labels are argmaxable (see subplot d). Each node corresponds to a label assignment, and it is green if it is argmaxable and red if not. We order the labels in levels in terms of $k$, the number of active labels. As can be seen, the levels for $k=0$ and $k=1$ have only green nodes, with unargmaxable label assignments first occurring for $k=2$.}
    \label{fig:hasse}
    \end{subfigure}
    \hfill
    \begin{subfigure}[t]{0.95\columnwidth}
    \centering
    \includegraphics[width=.9\linewidth]{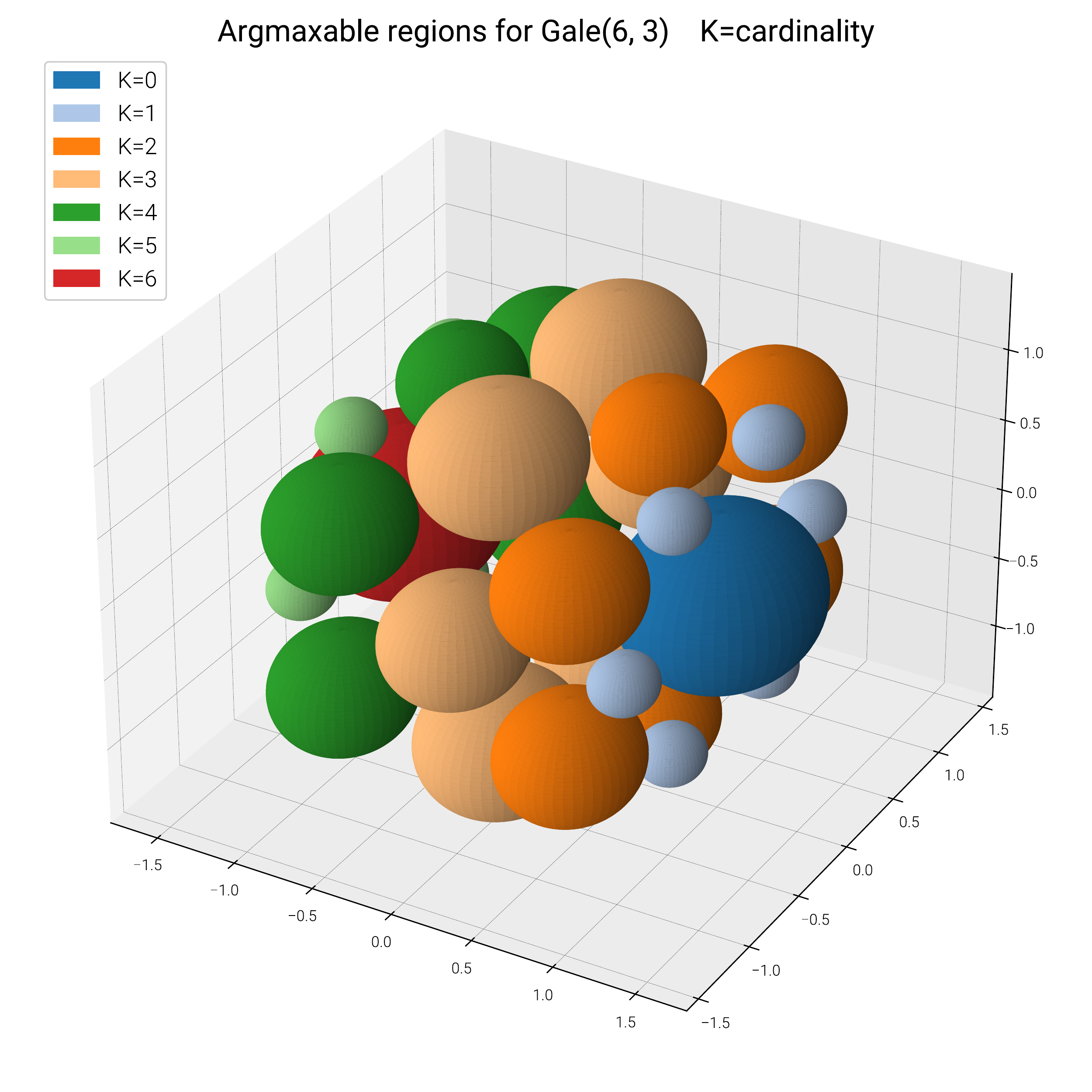}
    \subcaption{Geometric realisation of $\vv{W} = \dft{6}{3}$ in 3D space. The 6 hyperplanes defined by the rows of $\vv{W}$ tesselate 3D space into $32$ regions. For the illustration, we drop the hyperplanes and only plot the Chebyshev regions: each ball is the largest ball that fits in the corresponding region, as discovered by the Chebyshev LP. Each argmaxable label assignment from the plot on the left has a corresponding ball. The six 1-active label assignments from above are the light blue balls, arranged like petals.}
    \label{fig:balls}
    \end{subfigure}
    \\
    \vspace{1cm}
    \begin{subfigure}[t]{0.95\columnwidth}
    \centering
    \includegraphics[width=.95\linewidth]{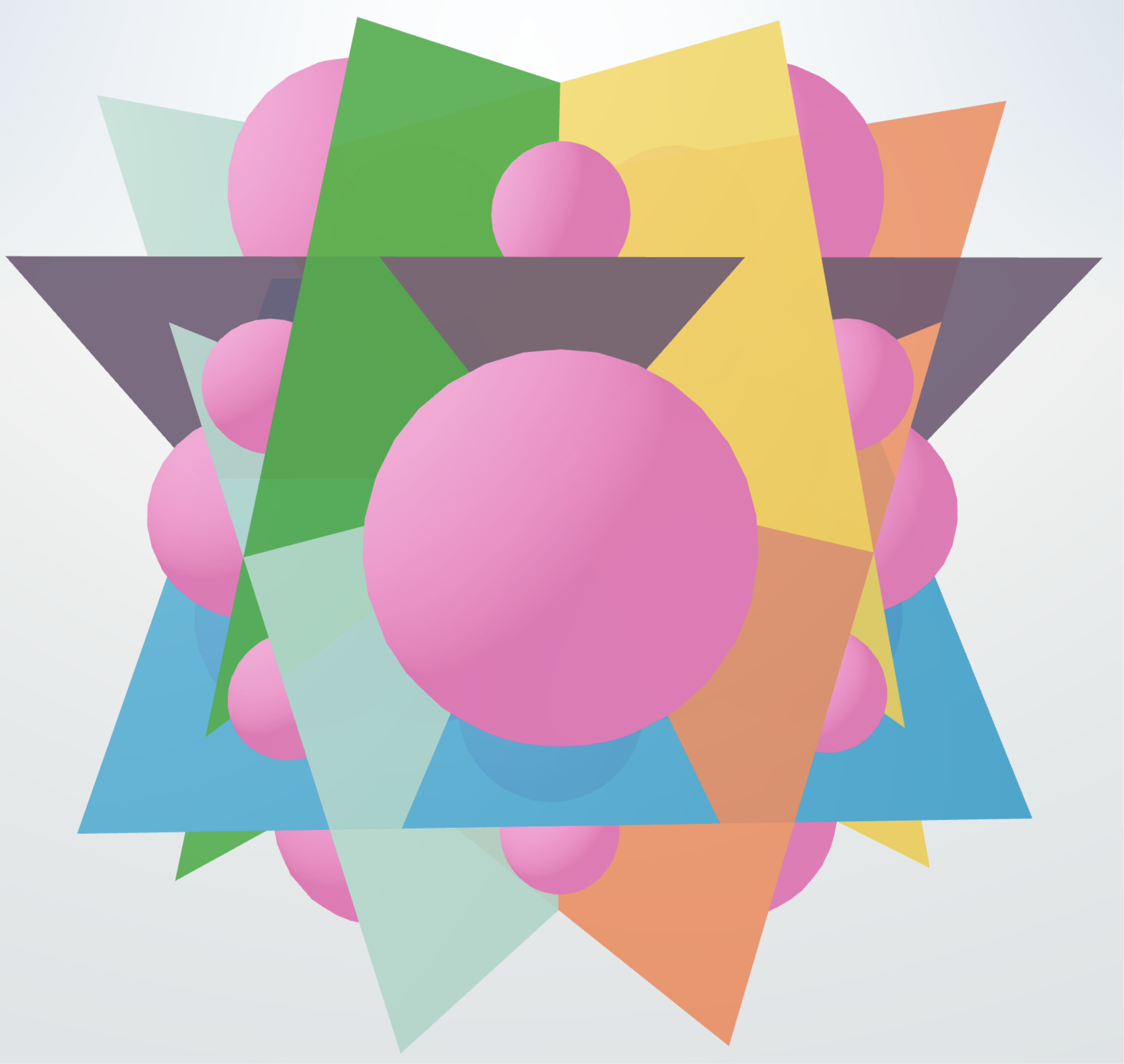}
    \subcaption{Same plot as above right but with hyperplanes drawn. The orientation is different such that the all $-$ region is in the center (blue ball for above right).}
    \label{fig:hypes}
    \end{subfigure}
    \hfill
    \begin{subfigure}[t]{0.95\columnwidth}
    \centering
    \includegraphics[width=.95\linewidth]{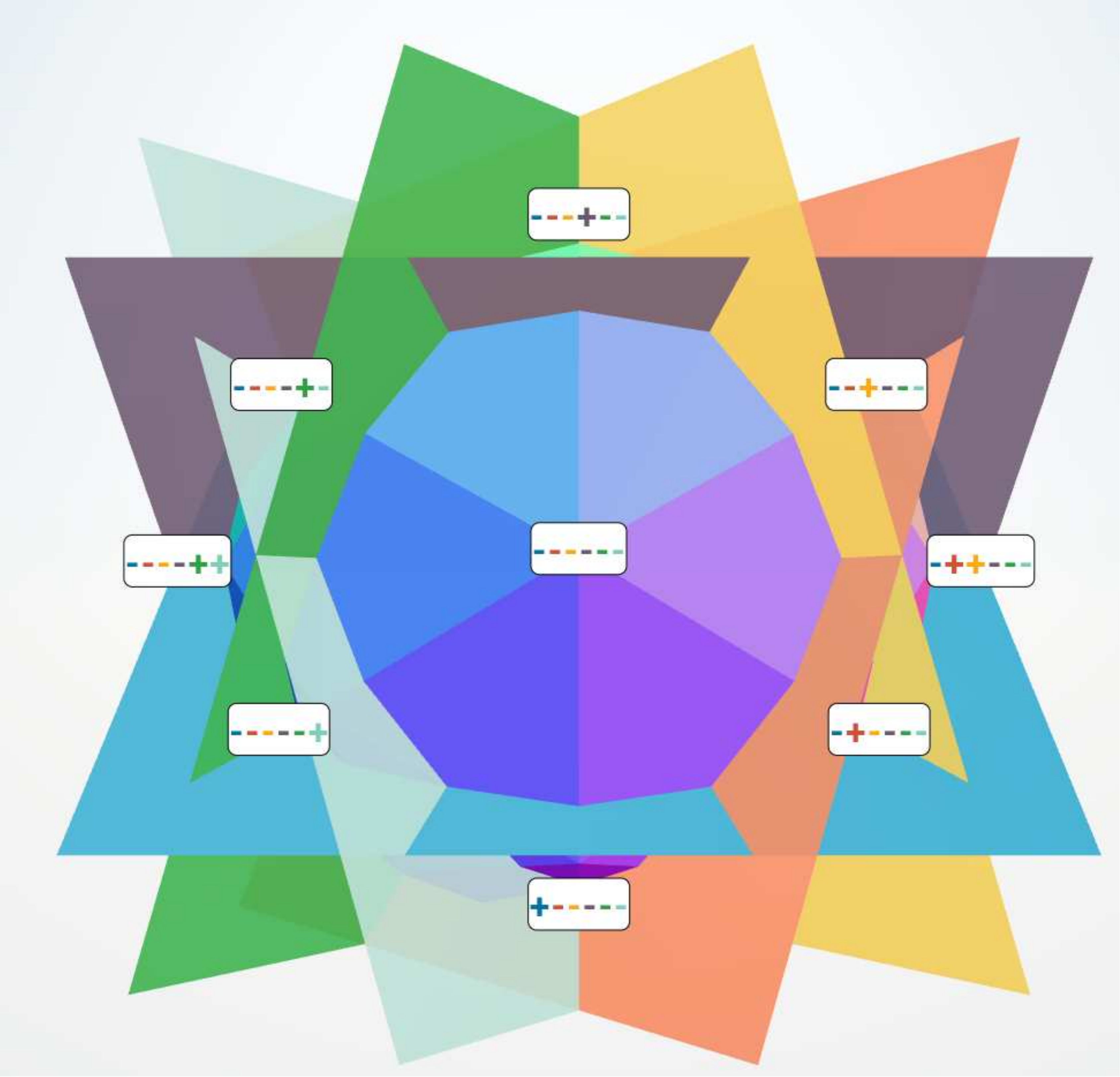}
    \subcaption{Same plot as on the left but showing sign vectors instead of balls. As can be seen, all 1-active label assignments are argmaxable, they surround the all $-$ vector.}
    \label{fig:hypeg}
    \end{subfigure}
\end{figure*}

\end{document}